\documentclass[twoside,11pt]{article}

\usepackage{ntheorem} 
\usepackage{jmlr2e}
\usepackage{hyperref}

\usepackage{microtype,marvosym} 
\usepackage{amsmath,amssymb,graphicx}
\usepackage{url}
\usepackage{subcaption} 

\usepackage{afterpage}
\usepackage{nameref}
\usepackage{adjustbox}

\usepackage{tikz}
\usepackage{pgfplots}
\pgfplotsset{compat=newest}
\pgfplotsset{plot coordinates/math parser=false}       

\pgfkeys{/pgfplots/mystyle/.style={
  semithick,
  tick style={major tick length=4pt,semithick,gray},
  xtick align = inside,
  ytick align = inside,
  xlabel near ticks,
  ylabel near ticks,
  try min ticks = 3,
  }}

\usepackage{natbib}
\newcommand{\textcite}{\citet}

\definecolor{lred}{RGB}{200,0,0}
\definecolor{dred}{RGB}{130,0,0} \definecolor{dblu}{RGB}{0,0,130}
\definecolor{dgre}{RGB}{0,130,0} \definecolor{dgra}{RGB}{50,50,50}
\definecolor{mgra}{RGB}{100,100,100}
\definecolor{lgra}{RGB}{220,220,220}
\definecolor{MPG}{RGB}{000,125,122}
\definecolor{ora}{HTML}{FF9933}

\definecolor{AMPurple}{HTML}{663366}
\definecolor{Burgundy}{HTML}{993333}
\definecolor{Coffee}{HTML}{7B6049}
\definecolor{ForestGreen}{HTML}{005826}
\definecolor{Lavender}{HTML}{6E6AB1}
\definecolor{PSLightBlue}{HTML}{7DA7D9}

\usepackage{booktabs}

\newcommand{\g}{\,|\,}

\newcommand{\Var}{\mathbb{V}}

\renewcommand{\Re}{\mathbb{R}}

\newcommand{\N}{\mathcal{N}}

\newcommand{\Trans}{^{\intercal}}

\renewcommand{\vec}{\boldsymbol} 
\newcommand{\mat}{\boldsymbol} 
\newcommand{\inv}[1]{{#1}^{-\!1}}

\renewcommand{\O}{\mathcal{O}} 
 
\newcommand{\GP}{\mathcal{GP}}
\newcommand{\Id}{\vec{I}}

\newcommand{\tr}[1]{\operatorname{tr}\left[#1\right]}

\usepackage{colonequals}
\newcommand{\ce}{\colonequals}

\usepackage{nicefrac}

\usetikzlibrary{arrows,shapes,plotmarks}

\tikzset{>=stealth'} 
\tikzstyle{graphnode} = 
   [circle,draw=black,minimum size=22pt,text centered,text
     width=22pt,inner sep=0pt] 
\tikzstyle{var}   =[graphnode,fill=white]
\tikzstyle{obs}   =[graphnode,fill=black,text=white]
\tikzstyle{fac}   =[rectangle,draw=black,fill=black!25,minimum size=5pt]
\tikzstyle{facprior} =[rectangle,draw=black,fill=black,text=white,minimum size=5pt]
\tikzstyle{edge}  =[draw=white,double=black,thick,-]
\tikzstyle{prior} =[rectangle, draw=black, fill=black, minimum size=
5pt, inner sep=0pt]
\tikzstyle{dirprior} = [circle, draw=black, fill=black, minimum
size=5pt, inner sep=0pt]

\DeclareSymbolFont{stmry}{U}{stmry}{m}{n}
\DeclareMathSymbol\leftarrowtriangle\mathrel{stmry}{"5E}
\DeclareMathSymbol\rightarrowtriangle\mathrel{stmry}{"5F}
\renewcommand{\gets}{\operatorname*{\leftarrowtriangle}}
\renewcommand{\to}{\operatorname*{\rightarrowtriangle}}

\newcounter{PHcomment}

\usepackage{algorithm}
\usepackage{algpseudocode}

\makeatletter
\newcommand{\superimpose}[2]{
  {\ooalign{$#1\@firstoftwo#2$\cr\hfil$#1\@secondoftwo#2$\hfil\cr}}}
\makeatother
\newcommand{\sk}{\mathpalette\superimpose{{\otimes}{\ominus}}}
\newcommand{\kr}{\otimes}
\newcommand{\vecm}[1]{\operatorname{vec}\left(#1\right)} 
\newcommand{\vecmtrans}[1]{\vecm{#1}\Trans} 
\newcommand{\tomat}[1]{\operatorname{mat}\left(#1\right)} 

\newcommand{\dx}[1]{\ \mathrm{d}#1} 

\newcommand{\pname}[1]{\emph{#1}} 

\newcommand{\eqcomment}[1]{\quad\text{#1}}

\newcommand{\latin}[1]{\emph{#1}}

\newcommand{\ie}{\latin{i.e.~}}
\newcommand{\eg}{\latin{e.g.~}}
\newcommand{\cf}{\latin{c.f.~}}

\newcommand{\todo}[1]{}
\newcommand{\pn}[1]{}
\newcommand{\qtp}[1]{}
\newcommand{\tempdisable}[1]{#1} 

\newcommand{\ksig}{\psi} 
\newcommand{\kw}{\ensuremath{w}} 
\newcommand{\kmu}{\ensuremath{{k}_0}} 
\newcommand{\km}{{k}_M} 
\newcommand{\ksigm}{\ksig_M} 
\newcommand{\kwm}{\ensuremath{w_M}} 
\renewcommand{\S}{\ensuremath{\vec S}} 
\newcommand{\T}{\ensuremath{\vec T_{\vec p}}} 
\newcommand{\Tm}[1]{\ensuremath{\tomat{\T#1}}} 
\newcommand{\Y}{\ensuremath{\vec Y}} 

\newcommand{\Km}{\vec K_M} 

\newcommand{\noise}{\ensuremath{\sigma^2}}

\newcommand{\ls}{\gamma} 

\newcommand{\kmcg}{KMCG}
\newcommand{\SoR}{SoR}
\newcommand{\smse}{SMSE}
\newcommand{\relerr}{\epsilon_f}
\newcommand{\relerrVar}{\epsilon_{var}}
\newcommand{\relerrNLZ}{\epsilon_{ev}}
\newcommand{\baseline}{baseline}
\newcommand{\baselines}{baselines}
\newcommand{\FITC}{FITC}
\newcommand{\VFE}{VFE}
\newcommand{\repetitions}{10} 

\newcommand{\fsize}{0.315} 
\newcommand{\defaultwidthfactor}{0.18} 
\newlength{\tableoffset}
\usepackage{epigraph}
\usepackage{amsfonts}
\usepackage{mdframed}
\renewcommand{\epsilon}{\varepsilon}
\newlength\figheight%
\newlength\figwidth%

\usepackage{thmtools}
\usepackage{thm-restate}

\makeatletter
\newcommand{\repeatable}[2]{%
  \label{#1}\global\@namedef{repeatable@#1}{#2}#2
}
\newcommand{\eqrepeat}[1]{%
  \@ifundefined{repeatable@#1}{NOT FOUND}{\@nameuse{repeatable@#1}}%
  \tag{\ref{#1}}}
\makeatother

\newcommand{\resultLegend}[1]{
{\smash{
\tikzexternaldisable
\definecolor{mycolor1}{rgb}{0.00000,0.47170,0.46040}%
\definecolor{mycolor2}{rgb}{0.49060,0.00000,0.00000}%
\definecolor{mycolor3}{rgb}{0.00000,0.00000,0.50900}%
\raisebox{-.65\figheight}{\ref{#1}}
\tikzexternalenable
}}
}
\newcommand{\legendFull}{\resultLegend{all}}
\newcommand{\legendWithoutCG}{\resultLegend{withoutcg}}
\newcommand{\legendAllCG}{\resultLegend{cgunstable}}

\newcommand{\resultFigureTable}[5][\\CG-steps $P$&CG-steps $P$&CG-steps $P$]{
 \renewcommand{\fsize}{0.3}
\setlength{\figheight}{0.27\textheight}
\setlength{\figwidth}{0.4\textwidth}
\setlength{\tabcolsep}{-4pt}
\hbox{\hspace{-1.5\tableoffset}
\begin{tabular}{ccc}
\adjustbox{valign=t}{\input{tikz/#2_ABALONE_folds_2_folds_visible_1_#3__title_#4__seed_12345.tikz}}&\adjustbox{valign=t}{\input{tikz/#2_PRECIPITATION_folds_2_folds_visible_1_#3__title_#4__seed_12345.tikz}}
&\adjustbox{valign=t}{\input{tikz/#2_PUMADYN_folds_2_folds_visible_1_#3__title_#4__seed_12345.tikz}}\\
\adjustbox{valign=t}{\input{tikz/#2_MPG_folds_2_folds_visible_1_#3__title_#4__seed_12345.tikz}}
& #5
&\adjustbox{valign=t}{\input{tikz/#2_POLETELECOMM_folds_2_folds_visible_1_#3__title_#4__seed_12345.tikz}}\\
\adjustbox{valign=t}{\input{tikz/#2_ELEVATORS_folds_2_folds_visible_1_#3__title_#4__seed_12345.tikz}}&\adjustbox{valign=t}{\input{tikz/#2_AILERONS_folds_2_folds_visible_1_#3__title_#4__seed_12345.tikz}}
&\adjustbox{valign=t}{\input{tikz/#2_TOY_folds_2_folds_visible_1_#3__title_#4__seed_12345.tikz}}
#1
\end{tabular}
}
}

\newcommand{\gridResultFigureTable}[2][CG-steps $P$]{
\let\saveddatasetCaption\datasetCaption
\let\saveddatasetSize\datasetSize
\let\saveddatasetDim\datasetDim
\renewcommand{\datasetCaption}[1]{}
\renewcommand{\datasetSize}[1]{}
\renewcommand{\datasetDim}[1]{}

\hbox{\hspace{-1.5\tableoffset}
\renewcommand{\fsize}{0.3}
\setlength{\figheight}{0.27\textheight}
\setlength{\figwidth}{0.27\textwidth}
\setlength{\tabcolsep}{-4pt}
\begin{tabular}{ccccc}
& \hphantom{$10^{-1}$}$10\times 10$& $1000\times 1000$ & $100\times 100\times 100$ & $30\times 30\times 30\times 30$\\
&\adjustbox{valign=t}{\input{tikz/#2_linspace_distorted_10_10_uniform_test_set_folds_2_folds_visible_1_covSEard__title_relative_error__seed_12345.tikz}}
&\adjustbox{valign=t}{\input{tikz/#2_linspace_distorted_1000_1000_uniform_test_set_folds_2_folds_visible_1_covSEard__title_relative_error__seed_12345.tikz}}
&\adjustbox{valign=t}{\input{tikz/#2_linspace_distorted_100_100_100_uniform_test_set_folds_2_folds_visible_1_covSEard__title_relative_error__seed_12345.tikz}}
&\adjustbox{valign=t}{\input{tikz/#2_linspace_distorted_30_30_30_30_uniform_test_set_folds_2_folds_visible_1_covSEard__title_relative_error__seed_12345.tikz}}
\\ 
&\adjustbox{valign=t}{\input{tikz/#2_linspace_distorted_10_10_uniform_test_set_folds_2_folds_visible_1_covSEard__title_relative_error_of_var__seed_12345.tikz}}
&\adjustbox{valign=t}{\input{tikz/#2_linspace_distorted_1000_1000_uniform_test_set_folds_2_folds_visible_1_covSEard__title_relative_error_of_var__seed_12345.tikz}}
&\adjustbox{valign=t}{\input{tikz/#2_linspace_distorted_100_100_100_uniform_test_set_folds_2_folds_visible_1_covSEard__title_relative_error_of_var__seed_12345.tikz}}
\makebox[0pt]{
\hspace{-\figwidth}
\legendFull
}
&\adjustbox{valign=t}{\input{tikz/#2_linspace_distorted_30_30_30_30_uniform_test_set_folds_2_folds_visible_1_covSEard__title_relative_error_of_var__seed_12345.tikz}}
\\ 
&\adjustbox{valign=t}{\input{tikz/#2_linspace_distorted_10_10_uniform_test_set_folds_2_folds_visible_1_covSEard__title_relative_nlZ_error__seed_12345.tikz}}
&\adjustbox{valign=t}{\input{tikz/#2_linspace_distorted_1000_1000_uniform_test_set_folds_2_folds_visible_1_covSEard__title_relative_nlZ_error__seed_12345.tikz}}
&\adjustbox{valign=t}{\input{tikz/#2_linspace_distorted_100_100_100_uniform_test_set_folds_2_folds_visible_1_covSEard__title_relative_nlZ_error__seed_12345.tikz}}
&\adjustbox{valign=t}{\input{tikz/#2_linspace_distorted_30_30_30_30_uniform_test_set_folds_2_folds_visible_1_covSEard__title_relative_nlZ_error__seed_12345.tikz}}
\\&\hphantom{$\epsilon_f 10^{-2}$}#1&#1&#1&#1
\end{tabular}
}
\let\datasetCaption\saveddatasetCaption
\let\datasetSize\saveddatasetSize
\let\datasetDim\saveddatasetDim
}

\newcommand{\datasetCaption}[1]{#1}
\newcommand{\datasetSize}[1]{\\{\tiny $N=#1$}}
\newcommand{\datasetDim}[1]{, {\tiny $D=#1$}}
\newcommand{\cgImpl}[1]{} 
\newcommand{\FOM}{CG}

\newcommand{\commonCaption}{
The shaded area visualizes minimum and maximum over all \baseline{} runs.
A cross denotes the end of a crashed run.
}

\usepackage[capitalize]{cleveref} 
\usepackage{autonum} 

\usetikzlibrary{plotmarks}
\usepgfplotslibrary{external} 
\tikzsetexternalprefix{tikz_out/}

\jmlrheading{\todo{volume}}{2019}{\todo{pages}}{\todo{submitted}}{\todo{published}}{Bartels19a}{Simon Bartels and Philipp Hennig}

\ShortHeadings{Conjugate Gradients for Kernel Machines}{Bartels and Hennig}
\firstpageno{1}

\begin{document}
	
\tikzexternaldisable
\smash{
%
%
\definecolor{mycolor1}{rgb}{0.00000,0.47170,0.46040}%
\definecolor{mycolor2}{rgb}{0.49060,0.00000,0.00000}%
\definecolor{mycolor3}{rgb}{0.00000,0.00000,0.50900}%
\begin{tikzpicture}

\begin{axis}[%
width=0.951\figwidth,
height=\figheight,
at={(0\figwidth,0\figheight)},
scale only axis,
xmin=0,
xmax=1,
ymin=0,
ymax=1,
hide axis,
axis x line*=bottom,
axis y line*=left,
legend style={at={(0.629,0.779)},anchor=south west,legend cell align=left,align=left,draw=white!15!black},
mystyle,
legend to name=all
]
\addplot [color=mycolor1,solid,line width=1.4pt]
  table[row sep=crcr]{%
-1	-1\\
};
\addlegendentry{\FOM{}};

\addplot [color=mycolor2,solid,line width=1.4pt]
  table[row sep=crcr]{%
-1	-1\\
};
\addlegendentry{\kmcg{}\cgImpl{\FOM{}}};

\addplot [color=mycolor3,dotted,line width=1.0pt]
  table[row sep=crcr]{%
-1	-1\\
};
\addlegendentry{VFE};

\addplot [color=mycolor3,dashed,line width=1.0pt]
  table[row sep=crcr]{%
-1	-1\\
};
\addlegendentry{FITC};

\end{axis}
\end{tikzpicture}%
%
%
\definecolor{mycolor1}{rgb}{0.00000,0.00000,0.50900}%
\definecolor{mycolor2}{rgb}{0.49060,0.00000,0.00000}%
\begin{tikzpicture}

\begin{axis}[%
width=0.951\figwidth,
height=\figheight,
at={(0\figwidth,0\figheight)},
scale only axis,
xmin=0,
xmax=1,
ymin=0,
ymax=1,
hide axis,
axis x line*=bottom,
axis y line*=left,
legend style={at={(0.629,0.81)},anchor=south west,legend cell align=left,align=left,draw=white!15!black},
mystyle,
legend to name=withoutcg
]
\addplot [color=mycolor1,dashed,line width=1.0pt]
  table[row sep=crcr]{%
-1	-1\\
};
\addlegendentry{FITC};

\addplot [color=mycolor1,dotted,line width=1.0pt]
  table[row sep=crcr]{%
-1	-1\\
};
\addlegendentry{VFE};

\addplot [color=mycolor2,solid,line width=1.4pt]
  table[row sep=crcr]{%
-1	-1\\
};
\addlegendentry{\kmcg{}\cgImpl{\FOM{}}};

\end{axis}
\end{tikzpicture}%
%
%
\definecolor{mycolor1}{rgb}{0.49060,0.00000,0.00000}%
\definecolor{mycolor2}{rgb}{0.00000,0.47170,0.46040}%
\begin{tikzpicture}

\begin{axis}[%
width=0.951\figwidth,
height=\figheight,
at={(0\figwidth,0\figheight)},
scale only axis,
xmin=0,
xmax=1,
ymin=0,
ymax=1,
hide axis,
axis x line*=bottom,
axis y line*=left,
legend style={at={(0.629,0.779)},anchor=south west,legend cell align=left,align=left,draw=white!15!black},
mystyle,
legend to name=cgunstable
]
\addplot [color=mycolor1,dashdotted,line width=1.0pt]
  table[row sep=crcr]{%
-1	-1\\
};
\addlegendentry{\kmcg{}\cgImpl{CG}};

\addplot [color=mycolor2,dashdotted,line width=1.0pt]
  table[row sep=crcr]{%
-1	-1\\
};
\addlegendentry{CG};

\addplot [color=mycolor2,solid,line width=1.4pt]
  table[row sep=crcr]{%
-1	-1\\
};
\addlegendentry{\FOM{}};

\addplot [color=mycolor1,solid,line width=1.4pt]
  table[row sep=crcr]{%
-1	-1\\
};
\addlegendentry{\kmcg{}\cgImpl{\FOM{}}};

\end{axis}
\end{tikzpicture}%
}
\tikzexternalenable

\title{Conjugate Gradients for Kernel Machines}

\author{\name Simon Bartels \email sbartels@tue.mpg.de \\
       \name Philipp Hennig \email ph@tue.mpg.de \\
       \addr
       Max Planck Institute for Intelligent Systems and University of T\"ubingen\\
       Maria-von-Linden-Str.~6, T\"ubingen, GERMANY}

\editor{Mohammad Emtiyaz Khan}

\maketitle

\pn{
\textbf{the story in short:}\\
We treat the problem as inference.
Our model shows how to construct a low-rank estimator for a symmetric bivariate function.
The typical approach is to project the infeasible operation (matrix inversion) into another space where it becomes feasible (low-rank space?).
The question is where to project to.
For some known approximation methods we show which projections they make.
The novelty is to use CG for projections.
As opposed to above methods, the projection is not fixed in advance but adapted on the fly.
This is what costs $\O(N^2)$.\\
\kmcg{} introduces next to the parameters $N$ and $P$ a third option $M$.
I need to elaborate on the distinction between $N$ and $M$.
}

\pn{
\textbf{How could I meet the reviewer concerns?}
\\\textbf{some problems}
\begin{itemize}
\item The impact of this paper is not very large so putting a lot of effort into this is not worth it.
\end{itemize}
}

\pn{
\textbf{How I see this work}
\begin{itemize}
	\item explorative
	\item goal is to present an idea how to improve on CG
	\item not much theory, as this is very hard and not worth it for this paper
	\item CG requires a low-level implementation which is not worth it for this paper, the real-time experiments are shown just for completeness in the appendix
\end{itemize}
}

\pn{
\begin{itemize}
\item missing stopping criterion in future work? How about I run \kmcg{} on the same problem and incorporate the noise in the likelihood? Also allows to interpret the noise-term for numerical stability.
\begin{itemize}
	\item I should make a table of all algorithm that I should try.
	\item explicit vs. implicit residual
	\item If I use the implicit residual it might be worth it to calculate the $\vec x$ explicit
	\item With and without noise on kernel matrix
\end{itemize}
\item Proofs
\item heading
\end{itemize}
}
\begin{abstract}
Regularized least-squares (kernel-ridge / Gaussian process) regression is a fundamental algorithm of statistics and machine learning. 
Because generic algorithms for the exact solution have cubic complexity in the number of datapoints, large datasets require to resort to approximations.
In this work, the computation of the least-squares prediction is itself treated as a probabilistic inference problem. 
We propose a structured Gaussian regression model on the kernel function that uses projections of the kernel matrix to obtain a low-rank approximation of the kernel and the matrix. 
A central result is an enhanced way to use the method of conjugate gradients for the specific setting of least-squares regression as encountered in machine learning.
Our method improves the approximation of the kernel ridge regressor / Gaussian process posterior mean over vanilla conjugate gradients and, allows computation of the posterior variance and the log marginal likelihood (evidence) without further overhead. 
\end{abstract}


\section{Introduction}

Regularized least-squares is one of the fundamental algorithms in  statistics and machine learning. 
Due to its importance it is known under a variety of names such as kernel ridge regression \citep{hoerl1970ridge}, spline regression (\eg{}\citet{wahba1990spline}), Kriging (\eg{}\citet{matheron1973intrinsic}) and Gaussian process (GP) regression (\eg{}\citet{RasmussenWilliams}).
The common principle is the estimation of a regression function from a reproducing kernel Hilbert space (RKHS) $f:\mathbb{X}\rightarrow\Re$ over some domain $\mathbb{X}$ that minimizes the regularized loss \cite[Eq.~(6.19)]{RasmussenWilliams}
\begin{align}
  \mathcal{L}(f) = \frac{1}{2}\|f\|^2 _k + \frac{1}{2\noise} \sum_{i=1}^N (y_i - f(x_i))^2,
\end{align}
where $(\vec x_i,y_i) \in \mathbb{X}\times\Re$, $i=1,\dots,N$ are observations, $\sigma^2 \in \Re^+$ is a regularization parameter, $k$ is the corresponding kernel and $\|\cdot\|_k$ is the RKHS norm of $f$.

The minimizer of this loss has a closed-form solution that coincides with the posterior mean of the Gaussian process $p(f\g \vec X, \vec y)=\GP(f; \bar{f}, \bar{c})$ under a zero-mean prior $p(f)=\GP(f;0,k)$ 
and likelihood $p(\vec{y}\g \vec{f}(\vec{X})) = \N(\vec{y};\vec{f}(\vec{X}),\noise\Id)$ \citep{kimeldorf1970correspondence,wahba1990spline,RasmussenWilliams}: \begin{align}
\repeatable{eq:mean}{
  \bar{f}(\vec x_*)&=\vec k_* \Trans(\vec K+\noise \vec I)^{-1}\vec y} ,
\\  \label{eq:var}
  \bar{c}(\vec x_*, \vec x_{**})&=k(\vec x_*, \vec x_{**})-\vec k_*\Trans(\vec K+\noise \vec I)^{-1}\vec k_{**}
\end{align}
where $\vec K_{ij}=k(x_i,x_j)$, and $\vec k_{*,i} = k(\vec x_*,\vec x_i)$. 

\begin{figure} 
 \setlength{\figwidth}{0.9\textwidth}
 \setlength{\figheight}{0.2\textwidth}
 \input{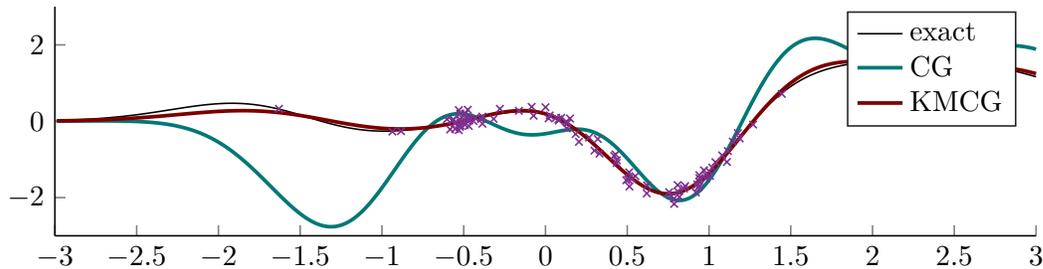}
 \caption{
 Our algorithm \kmcg{} in comparison to CG on a toy setup. 
The dataset consists of one hundred data-points where the targets are a draw from a zero-mean Gaussian process with squared exponential kernel (Eq.~\eqref{eq:sq_exp} with $\vec \Lambda=0.25$ and $\theta_f=2$).
The thin, black line is the posterior mean of that Gaussian process (Eq.~\ref{eq:mean}).
The light-green line is the mean prediction produced by conjugate gradients after $P=7$ steps and the dark-red line is the mean prediction of \kmcg{} (where the number of inducing inputs $M=N$).
}
 \label{fig:toy_example}
\end{figure}

For datasets up to about $N\sim 5 \cdot 10^4$ observations, the standard approach to solve Equations \eqref{eq:mean} and \eqref{eq:var} is to compute a Cholesky decomposition \citep{Cholesky} of $\vec K+\noise \vec I$ at a cubic cost $\mathcal{O}(N^3)$.
For larger datasets, a number of approximate algorithms have been proposed that yield an approximation $\hat{f}$ to $\bar{f}$ in linear time \citep{Zhu.1998, Csato.2002, Snelson.2007, Walder.2008, Rahimi.2009, titsias09:_variat_learn_induc_variab_spars_gauss_proces, LazaroGredilla.2010, Yan.2010, Le.2013, Solin.2014, wilson2015kissGP, hensman18variationalfourier}.
Comparative empirical studies like those of \citet{chalupka2013comparison} or \citet{quinonero2005unifying} indicate that some of these methods can provide good approximations in a reasonable amount of time, although there is no conclusive `best practice' among these choices.

Not included in the list above are iterative linear solvers, such as the method of conjugate gradients (CG) \citep{hestenes1952methods}.
These algorithms construct an approximate solution to systems of linear equations $\vec A\vec x=\vec b$ using repeated matrix-vector multiplications (MVMs). 
In general, each MVM with $\vec K$ has quadratic costs $\O(N^2)$ which is one reason why the machine learning community prefers the methods above.
Furthermore, a linear solver needs to run again for new test inputs when computing the posterior uncertainty (Eq.~\ref{eq:var}) and Gaussian process regression often requires the evaluation of the log marginal likelihood:
\begin{align}
\label{eq:llh}
\ln p(\vec y)&=-\frac{1}{2}\vec y\Trans (\vec K + \noise \vec I)^{-1}\vec y+\frac{1}{2}\ln|2\pi (\vec K+\noise\vec I)|^{-1}.
\end{align}
Conjugate gradients can be used to estimate $|\vec K|$ \citep{filippone2015ulisse}, yet also requiring several runs.

Below, we present a way of using CG specifically tailored to Equations \eqref{eq:mean} to \eqref{eq:llh} which we dub \pname{kernel machine conjugate gradients} (\kmcg{}).
Our approach follows the notion of probabilistic numerics (PN) \citep{HenOsbGirRSPA2015} which phrases approximation as inference.
A common idea of PN formulations is to replace a deterministic yet intractable operation by Bayesian inference where, by design, prior and likelihood admit analytic estimation of the intractable solution.
In our case, the `intractable' operation is the inversion of very large matrices (\ie{}of size $N\times N$ such that $N^3$ is intractable), and the design criterion for the prior is that the posterior mean over the matrix has to admit efficient inversion, which we achieve through the matrix inversion lemma. 
Instead of providing an approximation solely to the vector $(\vec K+\noise \vec I)^{-1}\vec y$, our approach uses the MVMs performed by CG to learn an approximation directly to the function $k$.

The following section proposes a model-template that can be used to learn low-rank approximations to kernel functions. 
The subsequent section shows how conjugate gradients can be applied into that template.
A discussion on how our approach relates to existing work is presented thereafter, in Section \ref{sec:related_work}.

\section{Model}
\label{sec:model}
To approximate Equations \eqref{eq:mean} to \eqref{eq:llh}, we will approximate the kernel and, to this end, present a probabilistic estimation rule for $k$.
The idea is to treat the kernel as unknown and to choose prior and likelihood such that the posterior mean $\km$ is efficient to evaluate and yields a kernel of finite rank.
Substituting for this finite-rank kernel in Equations \eqref{eq:mean} to \eqref{eq:llh} then allows to compute these expressions faster.
The following sections describe finite-rank kernel, our prior, possible likelihoods and resulting posteriors.
Figure \ref{fig:story} shows a schematic summary of this section.
\setlength{\figwidth}{\defaultwidthfactor\textwidth}
\setlength{\figheight}{0.8\figwidth}

\afterpage{
\thispagestyle{empty} 
\begin{figure}[hp!]
\caption{
Schematic summary of our proposed kernel approximation method. 
}
\label{fig:story}

\centering
\begin{mdframed}
\rotatebox[origin=c]{90}{
\textbf{ground-truth} in $\Re^2$
}
%
\begin{subfigure}{0.95\textwidth}
\centering
\setlength{\figwidth}{\defaultwidthfactor\textwidth}
\setlength{\figheight}{0.8\figwidth}
%
%
\begin{tikzpicture}

\begin{axis}[%
width=0.75\figwidth,
height=\figheight,
at={(0\figwidth,0\figheight)},
scale only axis,
point meta min=-2.0000,
point meta max=2.0000,
axis on top,
xmin=-3.0060,
xmax=3.0060,
xtick={-3,  0,  3},
xlabel={$x_1$},
ymin=-3.0060,
ymax=3.0060,
ytick={-3,  0,  3},
ylabel={$x_2$},
axis background/.style={fill=white},
legend style={legend cell align=left,align=left,draw=white!15!black},
mystyle, 
colorbar style={ ytick={-2, 0, 2}, yticklabels={-2.0, 0.0, 2.0}},
colormap={mymap}{[1pt] rgb(0pt)=(0.490852,0.000494315,0.000494315); rgb(1010pt)=(0.999496,0.999011,0.999011); rgb(1011pt)=(0.999506,0.999739,0.999733); rgb(2022pt)=(0,0.4717,0.4604)},
colorbar
]
\addplot [forget plot] graphics [xmin=-3.0060,xmax=3.0060,ymin=-3.0060,ymax=3.0060] {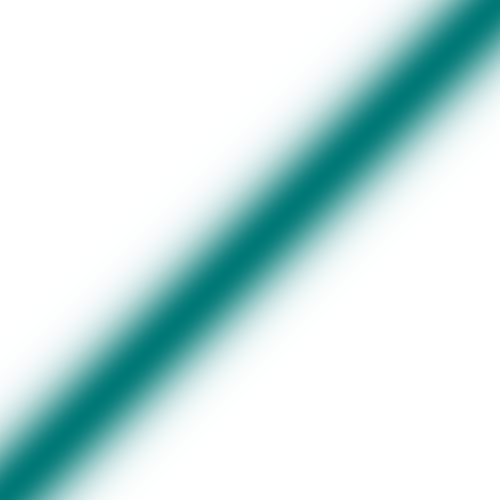};
\end{axis}
\end{tikzpicture}%
\input{tikz/fig_uncert_progression_ground_truth__data_grid__title_slice__step_0__seed_2.tikz}
\caption{The kernel $k$, here a squared exponential (Eq.~\ref{eq:sq_exp}) is assumed to be an unknown function.}
\label{fig:story_prior}
\end{subfigure}
\end{mdframed}
\begin{mdframed}
\rotatebox[origin=c]{90}{
\textbf{prior} in $\Re^2$
}
\begin{subfigure}{0.95\textwidth}
\centering
\begin{tabular}{cccc}
%
%
\begin{tikzpicture}

\begin{axis}[%
width=0.75\figwidth,
height=\figheight,
at={(0\figwidth,0\figheight)},
scale only axis,
point meta min=-2.0000,
point meta max=2.0000,
axis on top,
xmin=-3.0060,
xmax=3.0060,
xtick={-3,  0,  3},
ymin=-3.0060,
ymax=3.0060,
ytick={-3,  0,  3},
axis background/.style={fill=white},
legend style={legend cell align=left,align=left,draw=white!15!black},
mystyle
]
\addplot [forget plot] graphics [xmin=-3.0060,xmax=3.0060,ymin=-3.0060,ymax=3.0060] {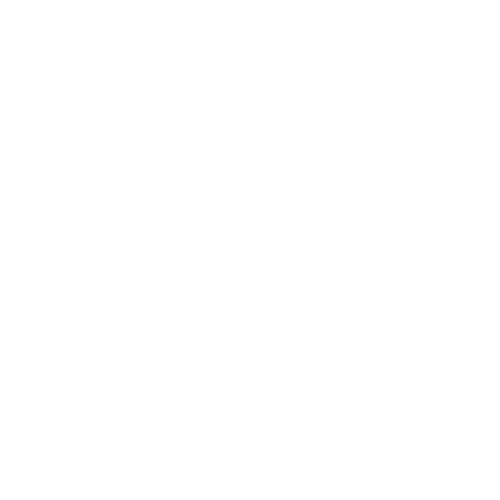};
\end{axis}
\end{tikzpicture}%
&
%
%
\begin{tikzpicture}

\begin{axis}[%
width=0.703\figwidth,
height=\figheight,
at={(0\figwidth,0\figheight)},
scale only axis,
point meta min=-2.0000,
point meta max=2.0000,
axis on top,
xmin=-3.0060,
xmax=3.0060,
xtick={-3,  0,  3},
ymin=-3.0060,
ymax=3.0060,
ytick={\empty},
axis background/.style={fill=white},
mystyle
]
\addplot [forget plot] graphics [xmin=-3.0060,xmax=3.0060,ymin=-3.0060,ymax=3.0060] {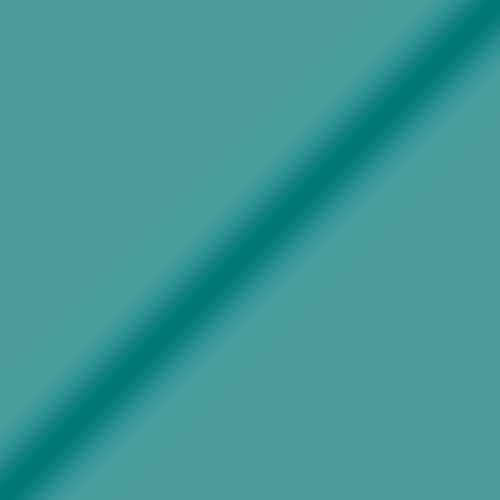};
\end{axis}
\end{tikzpicture}%
&
%
%
\begin{tikzpicture}

\begin{axis}[%
width=0.703\figwidth,
height=\figheight,
at={(0\figwidth,0\figheight)},
scale only axis,
point meta min=-2.0000,
point meta max=2.0000,
axis on top,
xmin=-3.0060,
xmax=3.0060,
xtick={-3,  0,  3},
ymin=-3.0060,
ymax=3.0060,
ytick={\empty},
axis background/.style={fill=white},
legend style={legend cell align=left,align=left,draw=white!15!black},
mystyle
]
\addplot [forget plot] graphics [xmin=-3.0060,xmax=3.0060,ymin=-3.0060,ymax=3.0060] {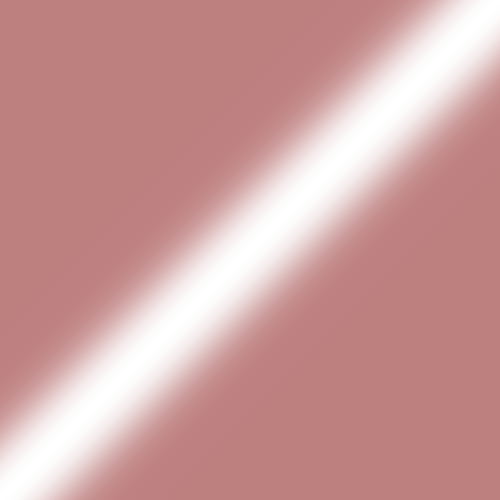};
\end{axis}
\end{tikzpicture}%
&
%
%
\begin{tikzpicture}

\begin{axis}[%
width=0.703\figwidth,
height=\figheight,
at={(0\figwidth,0\figheight)},
scale only axis,
point meta min=-2.0000,
point meta max=2.0000,
axis on top,
xmin=-3.0060,
xmax=3.0060,
xtick={-3,  0,  3},
ymin=-3.0060,
ymax=3.0060,
ytick={\empty},
axis background/.style={fill=white},
legend style={legend cell align=left,align=left,draw=white!15!black},
mystyle
]
\addplot [forget plot] graphics [xmin=-3.0060,xmax=3.0060,ymin=-3.0060,ymax=3.0060] {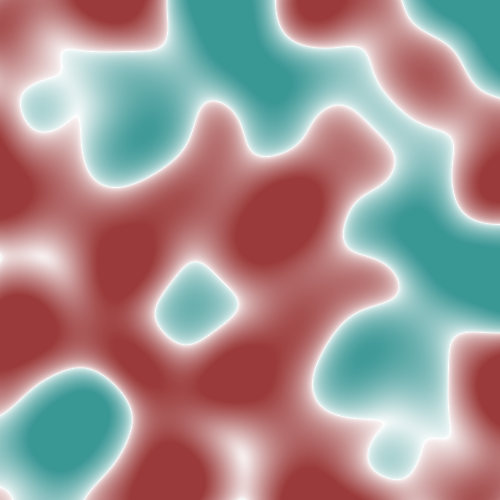};
\end{axis}
\end{tikzpicture}%
\\$\kmu$ & $\sqrt{\ksig}$ & $\frac{|k-\kmu|}{\sqrt{\ksig}}-1$ & sample\\~
\end{tabular}
\caption{Section \ref{sec:prior} 
describes a Kronecker-structured Gaussian process prior over the kernel.
Above pictures show from left-to-right: prior mean (zero), prior standard deviation, the absolute error divided by the standard deviation minus one and a sample from this prior.}
\end{subfigure}
\end{mdframed}
\begin{mdframed}
\rotatebox[origin=c]{90}{
{\textbf{likelihood} on $\mathbb{N}^2$} 
}
\begin{subfigure}{0.95\textwidth}
\centering
%
%
\definecolor{mycolor1}{rgb}{0.45187,0.71042,0.70423}%
\definecolor{mycolor2}{rgb}{0.98278,0.99090,0.99071}%
\begin{tikzpicture}

\begin{axis}[%
width=\figwidth,
height=0.356\figheight,
at={(0\figwidth,0\figheight)},
scale only axis,
point meta min=-2.0000,
point meta max=2.0000,
axis on top,
xmin=0.5000,
xmax=20.5000,
xtick={\empty},
y dir=reverse,
ymin=0.5000,
ymax=5.5000,
ytick={\empty},
axis background/.style={fill=white},
legend style={legend cell align=left,align=left,draw=white!15!black},
mystyle
]
	\fill [mycolor2] (axis cs:0.5000,0.5000) rectangle (axis cs:1.5000,1.5000);
	\fill [mycolor2] (axis cs:0.5000,1.5000) rectangle (axis cs:1.5000,2.5000);
	\fill [mycolor2] (axis cs:0.5000,2.5000) rectangle (axis cs:1.5000,3.5000);
	\fill [mycolor2] (axis cs:0.5000,3.5000) rectangle (axis cs:1.5000,4.5000);
	\fill [mycolor1] (axis cs:0.5000,4.5000) rectangle (axis cs:1.5000,5.5000);
	\fill [mycolor2] (axis cs:1.5000,0.5000) rectangle (axis cs:2.5000,1.5000);
	\fill [mycolor2] (axis cs:1.5000,1.5000) rectangle (axis cs:2.5000,2.5000);
	\fill [mycolor2] (axis cs:1.5000,2.5000) rectangle (axis cs:2.5000,3.5000);
	\fill [mycolor2] (axis cs:1.5000,3.5000) rectangle (axis cs:2.5000,4.5000);
	\fill [mycolor2] (axis cs:1.5000,4.5000) rectangle (axis cs:2.5000,5.5000);
	\fill [mycolor1] (axis cs:2.5000,0.5000) rectangle (axis cs:3.5000,1.5000);
	\fill [mycolor2] (axis cs:2.5000,1.5000) rectangle (axis cs:3.5000,2.5000);
	\fill [mycolor2] (axis cs:2.5000,2.5000) rectangle (axis cs:3.5000,3.5000);
	\fill [mycolor2] (axis cs:2.5000,3.5000) rectangle (axis cs:3.5000,4.5000);
	\fill [mycolor2] (axis cs:2.5000,4.5000) rectangle (axis cs:3.5000,5.5000);
	\fill [mycolor2] (axis cs:3.5000,0.5000) rectangle (axis cs:4.5000,1.5000);
	\fill [mycolor2] (axis cs:3.5000,1.5000) rectangle (axis cs:4.5000,2.5000);
	\fill [mycolor2] (axis cs:3.5000,2.5000) rectangle (axis cs:4.5000,3.5000);
	\fill [mycolor1] (axis cs:3.5000,3.5000) rectangle (axis cs:4.5000,4.5000);
	\fill [mycolor2] (axis cs:3.5000,4.5000) rectangle (axis cs:4.5000,5.5000);
	\fill [mycolor2] (axis cs:4.5000,0.5000) rectangle (axis cs:5.5000,1.5000);
	\fill [mycolor2] (axis cs:4.5000,1.5000) rectangle (axis cs:5.5000,2.5000);
	\fill [mycolor2] (axis cs:4.5000,2.5000) rectangle (axis cs:5.5000,3.5000);
	\fill [mycolor2] (axis cs:4.5000,3.5000) rectangle (axis cs:5.5000,4.5000);
	\fill [mycolor2] (axis cs:4.5000,4.5000) rectangle (axis cs:5.5000,5.5000);
	\fill [mycolor2] (axis cs:5.5000,0.5000) rectangle (axis cs:6.5000,1.5000);
	\fill [mycolor2] (axis cs:5.5000,1.5000) rectangle (axis cs:6.5000,2.5000);
	\fill [mycolor1] (axis cs:5.5000,2.5000) rectangle (axis cs:6.5000,3.5000);
	\fill [mycolor2] (axis cs:5.5000,3.5000) rectangle (axis cs:6.5000,4.5000);
	\fill [mycolor2] (axis cs:5.5000,4.5000) rectangle (axis cs:6.5000,5.5000);
	\fill [mycolor2] (axis cs:6.5000,0.5000) rectangle (axis cs:7.5000,1.5000);
	\fill [mycolor2] (axis cs:6.5000,1.5000) rectangle (axis cs:7.5000,2.5000);
	\fill [mycolor2] (axis cs:6.5000,2.5000) rectangle (axis cs:7.5000,3.5000);
	\fill [mycolor2] (axis cs:6.5000,3.5000) rectangle (axis cs:7.5000,4.5000);
	\fill [mycolor2] (axis cs:6.5000,4.5000) rectangle (axis cs:7.5000,5.5000);
	\fill [mycolor2] (axis cs:7.5000,0.5000) rectangle (axis cs:8.5000,1.5000);
	\fill [mycolor2] (axis cs:7.5000,1.5000) rectangle (axis cs:8.5000,2.5000);
	\fill [mycolor2] (axis cs:7.5000,2.5000) rectangle (axis cs:8.5000,3.5000);
	\fill [mycolor2] (axis cs:7.5000,3.5000) rectangle (axis cs:8.5000,4.5000);
	\fill [mycolor2] (axis cs:7.5000,4.5000) rectangle (axis cs:8.5000,5.5000);
	\fill [mycolor2] (axis cs:8.5000,0.5000) rectangle (axis cs:9.5000,1.5000);
	\fill [mycolor2] (axis cs:8.5000,1.5000) rectangle (axis cs:9.5000,2.5000);
	\fill [mycolor2] (axis cs:8.5000,2.5000) rectangle (axis cs:9.5000,3.5000);
	\fill [mycolor2] (axis cs:8.5000,3.5000) rectangle (axis cs:9.5000,4.5000);
	\fill [mycolor2] (axis cs:8.5000,4.5000) rectangle (axis cs:9.5000,5.5000);
	\fill [mycolor2] (axis cs:9.5000,0.5000) rectangle (axis cs:10.5000,1.5000);
	\fill [mycolor2] (axis cs:9.5000,1.5000) rectangle (axis cs:10.5000,2.5000);
	\fill [mycolor2] (axis cs:9.5000,2.5000) rectangle (axis cs:10.5000,3.5000);
	\fill [mycolor2] (axis cs:9.5000,3.5000) rectangle (axis cs:10.5000,4.5000);
	\fill [mycolor2] (axis cs:9.5000,4.5000) rectangle (axis cs:10.5000,5.5000);
	\fill [mycolor2] (axis cs:10.5000,0.5000) rectangle (axis cs:11.5000,1.5000);
	\fill [mycolor2] (axis cs:10.5000,1.5000) rectangle (axis cs:11.5000,2.5000);
	\fill [mycolor2] (axis cs:10.5000,2.5000) rectangle (axis cs:11.5000,3.5000);
	\fill [mycolor2] (axis cs:10.5000,3.5000) rectangle (axis cs:11.5000,4.5000);
	\fill [mycolor2] (axis cs:10.5000,4.5000) rectangle (axis cs:11.5000,5.5000);
	\fill [mycolor2] (axis cs:11.5000,0.5000) rectangle (axis cs:12.5000,1.5000);
	\fill [mycolor2] (axis cs:11.5000,1.5000) rectangle (axis cs:12.5000,2.5000);
	\fill [mycolor2] (axis cs:11.5000,2.5000) rectangle (axis cs:12.5000,3.5000);
	\fill [mycolor2] (axis cs:11.5000,3.5000) rectangle (axis cs:12.5000,4.5000);
	\fill [mycolor2] (axis cs:11.5000,4.5000) rectangle (axis cs:12.5000,5.5000);
	\fill [mycolor2] (axis cs:12.5000,0.5000) rectangle (axis cs:13.5000,1.5000);
	\fill [mycolor2] (axis cs:12.5000,1.5000) rectangle (axis cs:13.5000,2.5000);
	\fill [mycolor2] (axis cs:12.5000,2.5000) rectangle (axis cs:13.5000,3.5000);
	\fill [mycolor2] (axis cs:12.5000,3.5000) rectangle (axis cs:13.5000,4.5000);
	\fill [mycolor2] (axis cs:12.5000,4.5000) rectangle (axis cs:13.5000,5.5000);
	\fill [mycolor2] (axis cs:13.5000,0.5000) rectangle (axis cs:14.5000,1.5000);
	\fill [mycolor2] (axis cs:13.5000,1.5000) rectangle (axis cs:14.5000,2.5000);
	\fill [mycolor2] (axis cs:13.5000,2.5000) rectangle (axis cs:14.5000,3.5000);
	\fill [mycolor2] (axis cs:13.5000,3.5000) rectangle (axis cs:14.5000,4.5000);
	\fill [mycolor2] (axis cs:13.5000,4.5000) rectangle (axis cs:14.5000,5.5000);
	\fill [mycolor2] (axis cs:14.5000,0.5000) rectangle (axis cs:15.5000,1.5000);
	\fill [mycolor2] (axis cs:14.5000,1.5000) rectangle (axis cs:15.5000,2.5000);
	\fill [mycolor2] (axis cs:14.5000,2.5000) rectangle (axis cs:15.5000,3.5000);
	\fill [mycolor2] (axis cs:14.5000,3.5000) rectangle (axis cs:15.5000,4.5000);
	\fill [mycolor2] (axis cs:14.5000,4.5000) rectangle (axis cs:15.5000,5.5000);
	\fill [mycolor2] (axis cs:15.5000,0.5000) rectangle (axis cs:16.5000,1.5000);
	\fill [mycolor2] (axis cs:15.5000,1.5000) rectangle (axis cs:16.5000,2.5000);
	\fill [mycolor2] (axis cs:15.5000,2.5000) rectangle (axis cs:16.5000,3.5000);
	\fill [mycolor2] (axis cs:15.5000,3.5000) rectangle (axis cs:16.5000,4.5000);
	\fill [mycolor2] (axis cs:15.5000,4.5000) rectangle (axis cs:16.5000,5.5000);
	\fill [mycolor2] (axis cs:16.5000,0.5000) rectangle (axis cs:17.5000,1.5000);
	\fill [mycolor2] (axis cs:16.5000,1.5000) rectangle (axis cs:17.5000,2.5000);
	\fill [mycolor2] (axis cs:16.5000,2.5000) rectangle (axis cs:17.5000,3.5000);
	\fill [mycolor2] (axis cs:16.5000,3.5000) rectangle (axis cs:17.5000,4.5000);
	\fill [mycolor2] (axis cs:16.5000,4.5000) rectangle (axis cs:17.5000,5.5000);
	\fill [mycolor2] (axis cs:17.5000,0.5000) rectangle (axis cs:18.5000,1.5000);
	\fill [mycolor2] (axis cs:17.5000,1.5000) rectangle (axis cs:18.5000,2.5000);
	\fill [mycolor2] (axis cs:17.5000,2.5000) rectangle (axis cs:18.5000,3.5000);
	\fill [mycolor2] (axis cs:17.5000,3.5000) rectangle (axis cs:18.5000,4.5000);
	\fill [mycolor2] (axis cs:17.5000,4.5000) rectangle (axis cs:18.5000,5.5000);
	\fill [mycolor2] (axis cs:18.5000,0.5000) rectangle (axis cs:19.5000,1.5000);
	\fill [mycolor1] (axis cs:18.5000,1.5000) rectangle (axis cs:19.5000,2.5000);
	\fill [mycolor2] (axis cs:18.5000,2.5000) rectangle (axis cs:19.5000,3.5000);
	\fill [mycolor2] (axis cs:18.5000,3.5000) rectangle (axis cs:19.5000,4.5000);
	\fill [mycolor2] (axis cs:18.5000,4.5000) rectangle (axis cs:19.5000,5.5000);
	\fill [mycolor2] (axis cs:19.5000,0.5000) rectangle (axis cs:20.5000,1.5000);
	\fill [mycolor2] (axis cs:19.5000,1.5000) rectangle (axis cs:20.5000,2.5000);
	\fill [mycolor2] (axis cs:19.5000,2.5000) rectangle (axis cs:20.5000,3.5000);
	\fill [mycolor2] (axis cs:19.5000,3.5000) rectangle (axis cs:20.5000,4.5000);
	\fill [mycolor2] (axis cs:19.5000,4.5000) rectangle (axis cs:20.5000,5.5000);
\end{axis}
\end{tikzpicture}%
\input{tikz/ground_truth_dtc__data_grid__title_true_kernel_matrix__step_0__seed_2.tikz}
%
%
\definecolor{mycolor1}{rgb}{0.98278,0.99090,0.99071}%
\definecolor{mycolor2}{rgb}{0.45187,0.71042,0.70423}%
\begin{tikzpicture}

\begin{axis}[%
width=0.176\figwidth,
height=\figheight,
at={(0\figwidth,0\figheight)},
scale only axis,
point meta min=-2.0000,
point meta max=2.0000,
axis on top,
xmin=0.5000,
xmax=5.5000,
xtick={\empty},
y dir=reverse,
ymin=0.5000,
ymax=20.5000,
ytick={\empty},
axis background/.style={fill=white},
legend style={legend cell align=left,align=left,draw=white!15!black},
mystyle
]
	\fill [mycolor1] (axis cs:0.5000,0.5000) rectangle (axis cs:1.5000,1.5000);
	\fill [mycolor1] (axis cs:0.5000,1.5000) rectangle (axis cs:1.5000,2.5000);
	\fill [mycolor2] (axis cs:0.5000,2.5000) rectangle (axis cs:1.5000,3.5000);
	\fill [mycolor1] (axis cs:0.5000,3.5000) rectangle (axis cs:1.5000,4.5000);
	\fill [mycolor1] (axis cs:0.5000,4.5000) rectangle (axis cs:1.5000,5.5000);
	\fill [mycolor1] (axis cs:0.5000,5.5000) rectangle (axis cs:1.5000,6.5000);
	\fill [mycolor1] (axis cs:0.5000,6.5000) rectangle (axis cs:1.5000,7.5000);
	\fill [mycolor1] (axis cs:0.5000,7.5000) rectangle (axis cs:1.5000,8.5000);
	\fill [mycolor1] (axis cs:0.5000,8.5000) rectangle (axis cs:1.5000,9.5000);
	\fill [mycolor1] (axis cs:0.5000,9.5000) rectangle (axis cs:1.5000,10.5000);
	\fill [mycolor1] (axis cs:0.5000,10.5000) rectangle (axis cs:1.5000,11.5000);
	\fill [mycolor1] (axis cs:0.5000,11.5000) rectangle (axis cs:1.5000,12.5000);
	\fill [mycolor1] (axis cs:0.5000,12.5000) rectangle (axis cs:1.5000,13.5000);
	\fill [mycolor1] (axis cs:0.5000,13.5000) rectangle (axis cs:1.5000,14.5000);
	\fill [mycolor1] (axis cs:0.5000,14.5000) rectangle (axis cs:1.5000,15.5000);
	\fill [mycolor1] (axis cs:0.5000,15.5000) rectangle (axis cs:1.5000,16.5000);
	\fill [mycolor1] (axis cs:0.5000,16.5000) rectangle (axis cs:1.5000,17.5000);
	\fill [mycolor1] (axis cs:0.5000,17.5000) rectangle (axis cs:1.5000,18.5000);
	\fill [mycolor1] (axis cs:0.5000,18.5000) rectangle (axis cs:1.5000,19.5000);
	\fill [mycolor1] (axis cs:0.5000,19.5000) rectangle (axis cs:1.5000,20.5000);
	\fill [mycolor1] (axis cs:1.5000,0.5000) rectangle (axis cs:2.5000,1.5000);
	\fill [mycolor1] (axis cs:1.5000,1.5000) rectangle (axis cs:2.5000,2.5000);
	\fill [mycolor1] (axis cs:1.5000,2.5000) rectangle (axis cs:2.5000,3.5000);
	\fill [mycolor1] (axis cs:1.5000,3.5000) rectangle (axis cs:2.5000,4.5000);
	\fill [mycolor1] (axis cs:1.5000,4.5000) rectangle (axis cs:2.5000,5.5000);
	\fill [mycolor1] (axis cs:1.5000,5.5000) rectangle (axis cs:2.5000,6.5000);
	\fill [mycolor1] (axis cs:1.5000,6.5000) rectangle (axis cs:2.5000,7.5000);
	\fill [mycolor1] (axis cs:1.5000,7.5000) rectangle (axis cs:2.5000,8.5000);
	\fill [mycolor1] (axis cs:1.5000,8.5000) rectangle (axis cs:2.5000,9.5000);
	\fill [mycolor1] (axis cs:1.5000,9.5000) rectangle (axis cs:2.5000,10.5000);
	\fill [mycolor1] (axis cs:1.5000,10.5000) rectangle (axis cs:2.5000,11.5000);
	\fill [mycolor1] (axis cs:1.5000,11.5000) rectangle (axis cs:2.5000,12.5000);
	\fill [mycolor1] (axis cs:1.5000,12.5000) rectangle (axis cs:2.5000,13.5000);
	\fill [mycolor1] (axis cs:1.5000,13.5000) rectangle (axis cs:2.5000,14.5000);
	\fill [mycolor1] (axis cs:1.5000,14.5000) rectangle (axis cs:2.5000,15.5000);
	\fill [mycolor1] (axis cs:1.5000,15.5000) rectangle (axis cs:2.5000,16.5000);
	\fill [mycolor1] (axis cs:1.5000,16.5000) rectangle (axis cs:2.5000,17.5000);
	\fill [mycolor1] (axis cs:1.5000,17.5000) rectangle (axis cs:2.5000,18.5000);
	\fill [mycolor2] (axis cs:1.5000,18.5000) rectangle (axis cs:2.5000,19.5000);
	\fill [mycolor1] (axis cs:1.5000,19.5000) rectangle (axis cs:2.5000,20.5000);
	\fill [mycolor1] (axis cs:2.5000,0.5000) rectangle (axis cs:3.5000,1.5000);
	\fill [mycolor1] (axis cs:2.5000,1.5000) rectangle (axis cs:3.5000,2.5000);
	\fill [mycolor1] (axis cs:2.5000,2.5000) rectangle (axis cs:3.5000,3.5000);
	\fill [mycolor1] (axis cs:2.5000,3.5000) rectangle (axis cs:3.5000,4.5000);
	\fill [mycolor1] (axis cs:2.5000,4.5000) rectangle (axis cs:3.5000,5.5000);
	\fill [mycolor2] (axis cs:2.5000,5.5000) rectangle (axis cs:3.5000,6.5000);
	\fill [mycolor1] (axis cs:2.5000,6.5000) rectangle (axis cs:3.5000,7.5000);
	\fill [mycolor1] (axis cs:2.5000,7.5000) rectangle (axis cs:3.5000,8.5000);
	\fill [mycolor1] (axis cs:2.5000,8.5000) rectangle (axis cs:3.5000,9.5000);
	\fill [mycolor1] (axis cs:2.5000,9.5000) rectangle (axis cs:3.5000,10.5000);
	\fill [mycolor1] (axis cs:2.5000,10.5000) rectangle (axis cs:3.5000,11.5000);
	\fill [mycolor1] (axis cs:2.5000,11.5000) rectangle (axis cs:3.5000,12.5000);
	\fill [mycolor1] (axis cs:2.5000,12.5000) rectangle (axis cs:3.5000,13.5000);
	\fill [mycolor1] (axis cs:2.5000,13.5000) rectangle (axis cs:3.5000,14.5000);
	\fill [mycolor1] (axis cs:2.5000,14.5000) rectangle (axis cs:3.5000,15.5000);
	\fill [mycolor1] (axis cs:2.5000,15.5000) rectangle (axis cs:3.5000,16.5000);
	\fill [mycolor1] (axis cs:2.5000,16.5000) rectangle (axis cs:3.5000,17.5000);
	\fill [mycolor1] (axis cs:2.5000,17.5000) rectangle (axis cs:3.5000,18.5000);
	\fill [mycolor1] (axis cs:2.5000,18.5000) rectangle (axis cs:3.5000,19.5000);
	\fill [mycolor1] (axis cs:2.5000,19.5000) rectangle (axis cs:3.5000,20.5000);
	\fill [mycolor1] (axis cs:3.5000,0.5000) rectangle (axis cs:4.5000,1.5000);
	\fill [mycolor1] (axis cs:3.5000,1.5000) rectangle (axis cs:4.5000,2.5000);
	\fill [mycolor1] (axis cs:3.5000,2.5000) rectangle (axis cs:4.5000,3.5000);
	\fill [mycolor2] (axis cs:3.5000,3.5000) rectangle (axis cs:4.5000,4.5000);
	\fill [mycolor1] (axis cs:3.5000,4.5000) rectangle (axis cs:4.5000,5.5000);
	\fill [mycolor1] (axis cs:3.5000,5.5000) rectangle (axis cs:4.5000,6.5000);
	\fill [mycolor1] (axis cs:3.5000,6.5000) rectangle (axis cs:4.5000,7.5000);
	\fill [mycolor1] (axis cs:3.5000,7.5000) rectangle (axis cs:4.5000,8.5000);
	\fill [mycolor1] (axis cs:3.5000,8.5000) rectangle (axis cs:4.5000,9.5000);
	\fill [mycolor1] (axis cs:3.5000,9.5000) rectangle (axis cs:4.5000,10.5000);
	\fill [mycolor1] (axis cs:3.5000,10.5000) rectangle (axis cs:4.5000,11.5000);
	\fill [mycolor1] (axis cs:3.5000,11.5000) rectangle (axis cs:4.5000,12.5000);
	\fill [mycolor1] (axis cs:3.5000,12.5000) rectangle (axis cs:4.5000,13.5000);
	\fill [mycolor1] (axis cs:3.5000,13.5000) rectangle (axis cs:4.5000,14.5000);
	\fill [mycolor1] (axis cs:3.5000,14.5000) rectangle (axis cs:4.5000,15.5000);
	\fill [mycolor1] (axis cs:3.5000,15.5000) rectangle (axis cs:4.5000,16.5000);
	\fill [mycolor1] (axis cs:3.5000,16.5000) rectangle (axis cs:4.5000,17.5000);
	\fill [mycolor1] (axis cs:3.5000,17.5000) rectangle (axis cs:4.5000,18.5000);
	\fill [mycolor1] (axis cs:3.5000,18.5000) rectangle (axis cs:4.5000,19.5000);
	\fill [mycolor1] (axis cs:3.5000,19.5000) rectangle (axis cs:4.5000,20.5000);
	\fill [mycolor2] (axis cs:4.5000,0.5000) rectangle (axis cs:5.5000,1.5000);
	\fill [mycolor1] (axis cs:4.5000,1.5000) rectangle (axis cs:5.5000,2.5000);
	\fill [mycolor1] (axis cs:4.5000,2.5000) rectangle (axis cs:5.5000,3.5000);
	\fill [mycolor1] (axis cs:4.5000,3.5000) rectangle (axis cs:5.5000,4.5000);
	\fill [mycolor1] (axis cs:4.5000,4.5000) rectangle (axis cs:5.5000,5.5000);
	\fill [mycolor1] (axis cs:4.5000,5.5000) rectangle (axis cs:5.5000,6.5000);
	\fill [mycolor1] (axis cs:4.5000,6.5000) rectangle (axis cs:5.5000,7.5000);
	\fill [mycolor1] (axis cs:4.5000,7.5000) rectangle (axis cs:5.5000,8.5000);
	\fill [mycolor1] (axis cs:4.5000,8.5000) rectangle (axis cs:5.5000,9.5000);
	\fill [mycolor1] (axis cs:4.5000,9.5000) rectangle (axis cs:5.5000,10.5000);
	\fill [mycolor1] (axis cs:4.5000,10.5000) rectangle (axis cs:5.5000,11.5000);
	\fill [mycolor1] (axis cs:4.5000,11.5000) rectangle (axis cs:5.5000,12.5000);
	\fill [mycolor1] (axis cs:4.5000,12.5000) rectangle (axis cs:5.5000,13.5000);
	\fill [mycolor1] (axis cs:4.5000,13.5000) rectangle (axis cs:5.5000,14.5000);
	\fill [mycolor1] (axis cs:4.5000,14.5000) rectangle (axis cs:5.5000,15.5000);
	\fill [mycolor1] (axis cs:4.5000,15.5000) rectangle (axis cs:5.5000,16.5000);
	\fill [mycolor1] (axis cs:4.5000,16.5000) rectangle (axis cs:5.5000,17.5000);
	\fill [mycolor1] (axis cs:4.5000,17.5000) rectangle (axis cs:5.5000,18.5000);
	\fill [mycolor1] (axis cs:4.5000,18.5000) rectangle (axis cs:5.5000,19.5000);
	\fill [mycolor1] (axis cs:4.5000,19.5000) rectangle (axis cs:5.5000,20.5000);
\end{axis}
\end{tikzpicture}%
\raisebox{0.45\figheight}{\Huge $\vec =$}
\setlength{\figwidth}{0.4\figwidth}
\setlength{\figheight}{0.8\figwidth}
\raisebox{0.9\figheight}{
%
%
\definecolor{mycolor1}{rgb}{0.25866,0.60835,0.59997}%
\definecolor{mycolor2}{rgb}{0.98278,0.99090,0.99071}%
\definecolor{mycolor3}{rgb}{0.41269,0.68972,0.68309}%
\definecolor{mycolor4}{rgb}{0.29854,0.62942,0.62149}%
\definecolor{mycolor5}{rgb}{0.22540,0.59078,0.58203}%
\definecolor{mycolor6}{rgb}{0.23387,0.59526,0.58660}%
\definecolor{mycolor7}{rgb}{0.93791,0.96720,0.96649}%
\definecolor{mycolor8}{rgb}{0.88193,0.93763,0.93629}%
\definecolor{mycolor9}{rgb}{0.95444,0.97593,0.97541}%
\begin{tikzpicture}

\begin{axis}[%
width=0.703\figwidth,
height=\figheight,
at={(0\figwidth,0\figheight)},
scale only axis,
point meta min=-2.0000,
point meta max=2.0000,
axis on top,
xmin=0.5000,
xmax=5.5000,
xtick={\empty},
y dir=reverse,
ymin=0.5000,
ymax=5.5000,
ytick={\empty},
axis background/.style={fill=white},
legend style={legend cell align=left,align=left,draw=white!15!black},
mystyle
]
	\fill [mycolor5] (axis cs:0.5000,0.5000) rectangle (axis cs:1.5000,1.5000);
	\fill [mycolor9] (axis cs:0.5000,1.5000) rectangle (axis cs:1.5000,2.5000);
	\fill [mycolor4] (axis cs:0.5000,2.5000) rectangle (axis cs:1.5000,3.5000);
	\fill [mycolor6] (axis cs:0.5000,3.5000) rectangle (axis cs:1.5000,4.5000);
	\fill [mycolor1] (axis cs:0.5000,4.5000) rectangle (axis cs:1.5000,5.5000);
	\fill [mycolor9] (axis cs:1.5000,0.5000) rectangle (axis cs:2.5000,1.5000);
	\fill [mycolor5] (axis cs:1.5000,1.5000) rectangle (axis cs:2.5000,2.5000);
	\fill [mycolor8] (axis cs:1.5000,2.5000) rectangle (axis cs:2.5000,3.5000);
	\fill [mycolor7] (axis cs:1.5000,3.5000) rectangle (axis cs:2.5000,4.5000);
	\fill [mycolor2] (axis cs:1.5000,4.5000) rectangle (axis cs:2.5000,5.5000);
	\fill [mycolor4] (axis cs:2.5000,0.5000) rectangle (axis cs:3.5000,1.5000);
	\fill [mycolor8] (axis cs:2.5000,1.5000) rectangle (axis cs:3.5000,2.5000);
	\fill [mycolor5] (axis cs:2.5000,2.5000) rectangle (axis cs:3.5000,3.5000);
	\fill [mycolor1] (axis cs:2.5000,3.5000) rectangle (axis cs:3.5000,4.5000);
	\fill [mycolor3] (axis cs:2.5000,4.5000) rectangle (axis cs:3.5000,5.5000);
	\fill [mycolor6] (axis cs:3.5000,0.5000) rectangle (axis cs:4.5000,1.5000);
	\fill [mycolor7] (axis cs:3.5000,1.5000) rectangle (axis cs:4.5000,2.5000);
	\fill [mycolor1] (axis cs:3.5000,2.5000) rectangle (axis cs:4.5000,3.5000);
	\fill [mycolor5] (axis cs:3.5000,3.5000) rectangle (axis cs:4.5000,4.5000);
	\fill [mycolor4] (axis cs:3.5000,4.5000) rectangle (axis cs:4.5000,5.5000);
	\fill [mycolor1] (axis cs:4.5000,0.5000) rectangle (axis cs:5.5000,1.5000);
	\fill [mycolor2] (axis cs:4.5000,1.5000) rectangle (axis cs:5.5000,2.5000);
	\fill [mycolor3] (axis cs:4.5000,2.5000) rectangle (axis cs:5.5000,3.5000);
	\fill [mycolor4] (axis cs:4.5000,3.5000) rectangle (axis cs:5.5000,4.5000);
	\fill [mycolor5] (axis cs:4.5000,4.5000) rectangle (axis cs:5.5000,5.5000);
\end{axis}
\end{tikzpicture}%
}
\caption{Observations of $k$ stem from matrix-vector multiplications with the kernel matrix $\vec K$ (Section \ref{sec:likelihood}), sketched using random columns of the identity matrix.}
\end{subfigure}
\end{mdframed}
\begin{mdframed}
\rotatebox[origin=c]{90}{\textbf{posterior} in $\Re^2$}
\begin{subfigure}{0.95\textwidth}
\centering
\setlength{\figwidth}{\defaultwidthfactor\textwidth}
\setlength{\figheight}{0.8\figwidth}
\begin{tabular}{cccc}
%
%
\begin{tikzpicture}

\begin{axis}[%
width=0.75\figwidth,
height=\figheight,
at={(0\figwidth,0\figheight)},
scale only axis,
point meta min=-2.0000,
point meta max=2.0000,
axis on top,
xmin=-3.0060,
xmax=3.0060,
xtick={-3,  0,  3},
ymin=-3.0060,
ymax=3.0060,
ytick={-3,  0,  3},
axis background/.style={fill=white},
legend style={legend cell align=left,align=left,draw=white!15!black},
mystyle
]
\addplot [forget plot] graphics [xmin=-3.0060,xmax=3.0060,ymin=-3.0060,ymax=3.0060] {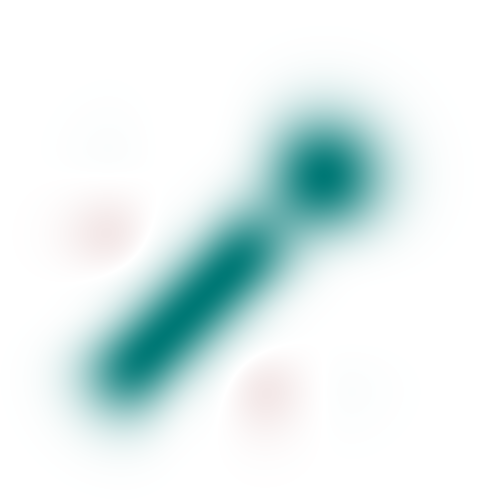};
\end{axis}
\end{tikzpicture}%
&
%
%
\begin{tikzpicture}

\begin{axis}[%
width=0.703\figwidth,
height=\figheight,
at={(0\figwidth,0\figheight)},
scale only axis,
point meta min=-2.0000,
point meta max=2.0000,
axis on top,
xmin=-3.0060,
xmax=3.0060,
xtick={-3,  0,  3},
ymin=-3.0060,
ymax=3.0060,
ytick={\empty},
axis background/.style={fill=white},
legend style={legend cell align=left,align=left,draw=white!15!black},
mystyle
]
\addplot [forget plot] graphics [xmin=-3.0060,xmax=3.0060,ymin=-3.0060,ymax=3.0060] {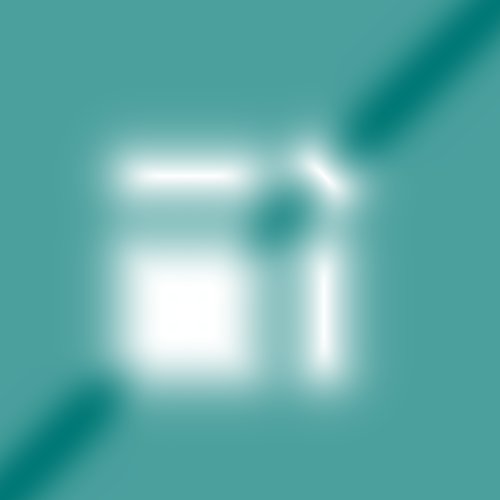};
\end{axis}
\end{tikzpicture}%
&
%
%
\begin{tikzpicture}

\begin{axis}[%
width=0.703\figwidth,
height=\figheight,
at={(0\figwidth,0\figheight)},
scale only axis,
point meta min=-2.0000,
point meta max=2.0000,
axis on top,
xmin=-3.0060,
xmax=3.0060,
xtick={-3,  0,  3},
ymin=-3.0060,
ymax=3.0060,
ytick={\empty},
axis background/.style={fill=white},
legend style={legend cell align=left,align=left,draw=white!15!black},
mystyle
]
\addplot [forget plot] graphics [xmin=-3.0060,xmax=3.0060,ymin=-3.0060,ymax=3.0060] {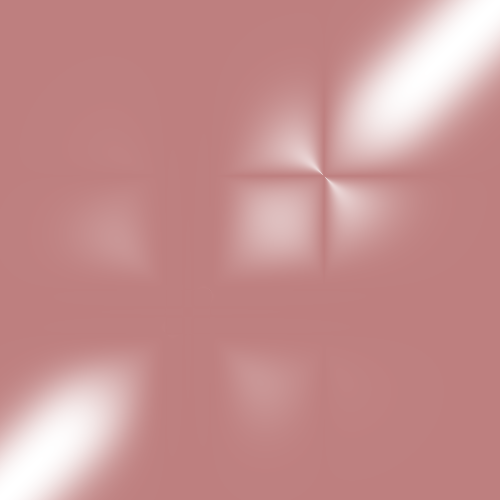};
\end{axis}
\end{tikzpicture}%
&
%
%
\begin{tikzpicture}

\begin{axis}[%
width=0.703\figwidth,
height=\figheight,
at={(0\figwidth,0\figheight)},
scale only axis,
point meta min=-2.0000,
point meta max=2.0000,
axis on top,
xmin=-3.0060,
xmax=3.0060,
xtick={-3,  0,  3},
ymin=-3.0060,
ymax=3.0060,
ytick={\empty},
axis background/.style={fill=white},
legend style={legend cell align=left,align=left,draw=white!15!black},
mystyle
]
\addplot [forget plot] graphics [xmin=-3.0060,xmax=3.0060,ymin=-3.0060,ymax=3.0060] {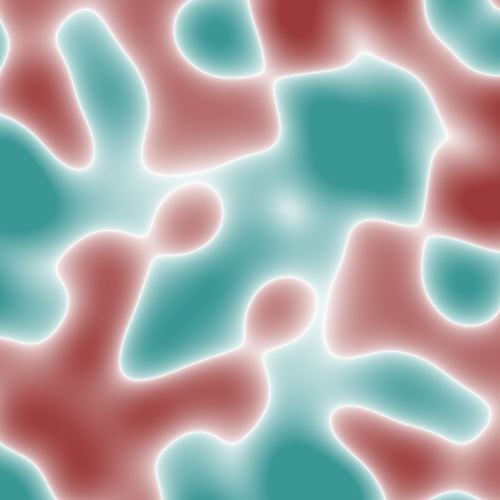};
\end{axis}
\end{tikzpicture}%
\\$\km$ & $\sqrt{\ksigm}$ & $\frac{|k-\km|}{\sqrt{\ksigm}}-1$ & sample
\end{tabular}
\caption{The posterior is again Gaussian (Section \ref{sec:posterior}) and similar to Figure \ref{fig:story_prior} the pictures show from left-to-right: mean, standard deviation, relative error and a sample.
By design, the posterior mean $\km$ is an approximation of finite rank which allows to efficiently solve the original least-squares problem (Section \ref{sec:degenerate_kernel_approximation}).}
\end{subfigure}
\end{mdframed}
\vfill
\end{figure}
\clearpage
}

\subsection{{Finite-rank Kernel}}
\label{sec:degenerate_kernel_approximation}
%
%
%
An $M$-rank approximation to a kernel is a factorization of the form
\begin{align}
\label{eq:deg_kern}
k(\vec x, \vec z)\approx \vec \phi(\vec x)^* \vec \Sigma^{-1} \vec \phi(\vec z)
\end{align}
where $\vec \phi(\vec x):\mathbb{X}\rightarrow \mathbb{C}^M$, $\vec \phi^*$ denotes the conjugate transpose, and $\vec{\Sigma}$ is an $M\times M$ Hermitian and positive definite matrix. 
Given such an expansion one can use the matrix-inversion, and matrix-determinant lemmata to approximate Equations \eqref{eq:mean} to \eqref{eq:llh} with the expressions below
\begin{align}
\label{eq:mean_2}
\overline{f}(\vec x_*)&\approx \vec \phi(\vec x_*)^* \left(\vec \Phi \vec \Phi^* + \noise \vec \Sigma \right)^{-1}\vec \Phi \vec y
\\\overline{c}(\vec x_*, \vec z_{*})&\approx \noise \vec \phi(\vec x_*)^* \left(\vec \Phi \vec \Phi^* + \noise \vec \Sigma \right)^{-1}\vec \phi(\vec z_{*}) \label{eq:var_2}
\\\ln p(\vec y)&\approx -\frac{1}{2}\vec y\Trans\vec\Phi^*\left(\vec \Phi \vec \Phi^* + \noise \vec \Sigma\right)^{-1}\vec \Phi \vec y-\frac{1}{2}\ln \left|\left(\frac{1}{\noise}\vec \Phi \vec \Phi^*+\vec \Sigma \right)\right| -\frac{N}{2}\ln(2\pi \noise) 
\label{eq:llh_2}
\end{align}
where $\vec \phi(\vec x_*)_j=\phi_j(\vec x_*)$ and $\vec \Phi_{ij}=\phi_i(\vec X_j)$.
Typically $M\ll N$ and therefore the computational costs to evaluate Equations \eqref{eq:mean_2} to \eqref{eq:llh_2} reduce from $\O(N^3)$ to $\O(NM^2)$, \ie{}linear in $N$.
The dominant factor is the matrix-matrix product $\vec \Phi\vec \Phi^*$. 

An example for a finite-rank kernel that will become important later, is the \pname{Subset of Regressors} (\SoR{}) approximation \citep{quinonero2005unifying}
\begin{align}
k_{SoR}(\vec x, \vec z)&=k(\vec x, \vec X_U)k(\vec X_U, \vec X_U)^{-1}k(\vec X_U, \vec z) \label{eq:k_sor}
\end{align}
where $\vec X_U$ is a set of $M$ so called inducing inputs.
The method proposed in this work (\kmcg{}) is related to \SoR{}.
Readers familiar with \SoR{} will be aware of the associated flaws, and methods to remedy them \citep{quinonero2005unifying, titsias09:_variat_learn_induc_variab_spars_gauss_proces}.
%
\label{sec:hybrid_extensions}
For stationary kernels and tests points far away from the data, the predictive uncertainty (Eq.~\ref{eq:var_2}) goes to zero.
The \pname{Deterministic Training Conditional (DTC)} approximation alleviates this issue by using the exact kernel for the prior uncertainty over the test inputs 
 \citep{quinonero2005unifying}.
In effect this is a substitution of \cref{eq:var_2} for \cref{eq:var_DTC} below.
\begin{align}
\overline{c}(\vec x_*, \vec z_{*})\approx& k(\vec x_*, \vec x_{**})-\vec \phi(\vec x_*)^* \left(\vec \Phi \vec \Phi^* + \noise \vec I\right)^{-1}\vec \phi(\vec z_{*})\label{eq:var_DTC}
\end{align}
We will apply the same substitution for our method \kmcg{}.

\subsection{Prior}
\label{sec:prior}
Consider a Gaussian process prior over bivariate functions
\begin{align}
\repeatable{eq:prior}{k&\sim\GP(\kmu, \ls \ksig)}
\end{align}
where $\ksig: \mathbb{X}^2\times \mathbb{X}^2\rightarrow \Re$ is a covariance function over kernel and $\ls\in\Re^+$ is a scaling parameter.
Since the posterior mean is meant to be a substitution for the exact kernel, this is an exchange of one least-squares problem for another.
Without further assumptions, calculating the posterior over $k$ is more expensive than computing the equations of interest (Equations \ref{eq:mean} to \ref{eq:llh}).
Efficient inference is rendered possible by imposing the following structure on $\ksig$
\begin{align}
\repeatable{eq:w_requirement}{
\ksig(k(\vec a, \vec b), k(\vec c, \vec d))&\ce\frac{1}{2}\kw(\vec a, \vec c)\kw(\vec b, \vec d)+\frac{1}{2}\kw(\vec a, \vec d)\kw(\vec b, \vec c)
}
\end{align}
for $\vec a, \vec b, \vec c, \vec d \in \mathbb{X}$ 
and where $\kw$ is a covariance function on the domain $\mathbb{X}$.
Consider the first addend.
It states that the similarity between $k(\vec a, \vec b)$ and $k(\vec c, \vec d)$ depends on the similarity of $\vec a$ and $\vec c$, and $\vec b$ and $\vec d$--a natural assumption for kernel matrices.
The second addend is a symmetrization of the first.
Observe that each addend is a product kernel of two pairs of inputs and recall that a product kernel produces Kronecker product matrices. 
The sum of the two products leads to covariance matrices that have a \emph{symmetric} Kronecker product form, \ie{}$\forall \vec A, \vec B\in\Re^{N\times N}: \ksig(\vec A, \vec B)=\vec A\sk\vec B\in\Re^{N^2\times N^2}$ (see Appendix \ref{app:kronecker}).
This will allow a sufficiently efficient evaluation of the posterior. 
\cref{fig:story} visualizes the variance and shows samples from this prior for the toy setup from \cref{fig:toy_example}.

This choice of prior offers a trade-off between efficient tractable inference and the desire to encode as much prior structural information about the kernel as possible.
One desirable property to encode is symmetry, and indeed, matrix-valued functions sampled from this prior distribution are symmetric (\cf{}\cref{fig:story} for examples, \cref{app:sampling} for formal proof).
Kernel functions are also positive definite. 
Unfortunately, since the positive definite cone is not a linear sub-space of the vector-space of real matrices, this property can not be encoded in a Gaussian prior, in closed form.\footnote{\eg{}\cite{Hennig_LinearSolvers2015} discusses this problem and possible solutions.}
However, it \emph{is} possible to guarantee positive-definiteness of the posterior mean point estimate through the specific choice of prior parameters $\kmu=0$ (proof in \cref{thm:spd_mean}, p.~\pageref{thm:spd_mean}).
For this reason, we adopt this choice for the remainder.
There are other properties of certain kernels that would be desirable to encode, but which are not feasible within the chosen framework without sacrificing fast computability.
For example, stationarity of the kernel can not be represented by a prior with Kronecker structure in the covariance since $\vec a$ and $\vec b$ (and symmetrically $\vec c$ and $\vec d$) do not appear together as arguments to $\kw$.

The question remains how to choose $\kw$.
Recall that $\kw$ should reflect the similarity between $k(\vec a, \vec b)$ and $k(\vec c, \vec d)$ which depends on the similarity of $\vec a$ and $\vec c$, and $\vec b$ and $\vec d$.
To measure the relationship between inputs is exactly the purpose of the kernel $k$ and we therefore set $$\kw\ce k$$ for the remainder.
Even if $k$ fails to capture similarity between inputs, as choice for $\kw$ it still captures the similarity between the kernel values.
Furthermore, samples from the approximate kernel will be a function of $\kw$ and lastly, this choice is convenient computationally as expressions simplify.

\subsection{Likelihood}
\label{sec:likelihood}
Having specified a prior over $k$, we will now be concerned with how to obtain observations.
Iterative solvers like conjugate gradients proceed by collecting a sequence of \emph{linear} projections of the (kernel) matrix to be inverted, in the form of matrix-vector products. 
In fact, this general structure also describes the setting of non-adaptive approaches like inducing point methods, which can be interpreted as collecting multiplications of the kernel matrix with a set of \emph{pre-specified} and \emph{sparse} vectors (namely the unit selection vectors $e_{\vec x_{u_i}}$). 
We can use these matrix-vector products for learning a low-rank version of the kernel by introducing the linear operator 
\begin{align}
\repeatable{eq:observation_operator}{
\T: & (\mathbb{X}\times \mathbb{X})^\Re \rightarrow \Re^{P^2}, k \mapsto  \vecm{\left[
\iint\! k(\vec x, \vec z)p_i(\vec x)p_j(\vec z)\dx{\vec x}\dx{\vec z}
\right]_{ij}}
}
\end{align}
where $i,j=1...P$, $\vec p=[p_1, ..., p_P]$ are densities or distributions and $\vecm{\vec A}$ is a column vector created by stacking the rows of $\vec A$.
\begin{example}[Matrix-vector multiplication]
\label{ex:dtc}
Define $\T$ with 
\begin{align}
\repeatable{eq:lk_dtc}{p_i(\vec x)&\ce\sum_{j=1}^M s_{ij} \delta(\vec x - \vec x_{u_j})}. 
\end{align}
Then the evaluation of $\T k$ reduces to a matrix vector product, that is $\Tm{k}=\vec S\Trans k(\vec X_U, \vec X_U)\vec S$ where $\vec S_{ij}=s_{ij}$, $\vec X_U=[\vec x_{u_1}, ..., \vec x_{u_M}]$ and $\tomat{~}$ transforms a $P^2$ vector into a $P\times P$ matrix, s.t.~$\tomat{\vecm{\vec A}}=\vec A$.
\end{example}
The $\vec x_{u_j}$ can be datapoints or arbitrary elements of the domain $\mathbb{X}$.
The choice $\S_{ij}\ce\delta_{ij}$ leads to the \pname{Subset of Regressors} approximation (\cref{thm:dtc}, p.~\pageref{thm:dtc}).
\begin{example}[Integrals with Eigenfunctions]
\label{ex:rks}
Let $\phi_i$ $i=1, ..., P$ be orthogonal Eigenfunctions of $k$ with respect to a density $\nu$ on $\mathbb{X}$, \ie{}
\begin{align}
\int k(\vec x, \vec z)\phi_i(\vec z)\nu(\vec z)\dx{\vec z}&=\lambda_i\phi_i(\vec x)
\\\int \phi_i(\vec z)\phi_j(\vec z)\nu(\vec z)\dx{\vec z}&=\delta_{ij}
\end{align}
where $\lambda_i\in \Re$ and $\delta_{ij}$ is the Kronecker indicator function (compare \citet[p. 96]{RasmussenWilliams}). 
Then for
\begin{align}
p_i(\vec x)&\ce\phi_i(\vec x)\nu(\vec x) \label{eq:lk_rks}
\end{align}
the observations $[\Tm k]_{ij}=\delta_{ij} \lambda_i$ are spectral values of the kernel.
\end{example}
In essence, this example shows another possibility to express prior knowledge over the kernel.
This likelihood leads to the \emph{Projected Bayes Regressor} \citep{Trecate.1999} (\cref{thm:pbr}, p.~\pageref{thm:pbr}), which is a historical, deterministic precursor to the more widely known random Fourier feature expansion of \cite{Rahimi.2008}.
\subsection{Posterior and Subsumed Approximation Methods}
\label{sec:posterior}
The observation operator $\T$ is a linear projection, and hence transforms the Gaussian prior into an also Gaussian posterior.
Given the prior (Eq.~\ref{eq:prior}) and any likelihood of the previous section, the posterior is Gaussian with: 
\begin{align}
p(k\mid \Y, \T)&=\N(\km, \kwm) 
\\\km&=\kmu+(\T \ksig)\Trans(\T(\T\ksig)\Trans)^{-1}(\vecm{\Y}-\T\kmu) \label{eq:km}
\\\ksigm&=\ksig-(\T\ksig)\Trans(\T(\T\ksig)\Trans)^{-1}\T\ksig  \label{eq:ksigm}
\end{align}
The concrete posterior depends on the choice of $\T$.
The following propositions presents approximation methods that have a view as GP inference with low-rank kernel and how they arise in our framework.

\begin{restatable}[Subset of Regressors]{proposition}{DTCprop}
\label{thm:dtc}
Consider the prior of \cref{eq:prior} with $\kmu\ce 0$ and $w\ce k$ and the likelihood defined in \cref{ex:dtc} with $s_{ij}=\delta_{ij}$. 
Then the posterior mean $\km$ is equivalent to that of \SoR{}:
$$\km(\vec x, \vec z)=k_{\SoR{}}=k(\vec x, \vec X_U)k(\vec X_U, \vec X_U)^{-1}k(\vec X_U, \vec z)$$
where $\vec X_U$ are inducing inputs, not necessarily part of $\vec X$.
\end{restatable}
The proof is part of \cref{app:dtc}.
An example of this posterior distribution is shown in Figure \ref{fig:story}.
The related method, \pname{Fully Independent Conditional} (FIC) has a very similar kernel, $k_{FIC}=k_{\SoR{}}(\vec x, \vec z)+\delta(\vec x - \vec z)(k(\vec x, \vec x)-k_{\SoR{}}(\vec x, \vec x))$.
The structure indicates that this kernel should fit as well into our framework.
One option that comes to mind, is to model the diagonal elements as certain, using a prior with mean $\kmu(k(\vec a, \vec b))\ce\delta(\vec a - \vec b)k(\vec a, \vec b)$ and covariance function $\ksig'(\vec a, \vec b; \vec c, \vec d)\ce (1-\delta(\vec a - \vec b))\ksig(\vec a, \vec b; \vec c, \vec d)(1-\delta(\vec c - \vec d))$. 
The posterior mean, however, is in general not the \pname{FIC} kernel, as for off-diagonal elements, the prediction differs to due to the certainty over the diagonal elements. 
Furthermore, the modification of the covariance function annuls the convenient algebraic properties of the associated covariance matrices and hence, this prior is dismissed as potential \pname{FIC} competitor.
Another strategy could be to add the diagonal elements as observations.
However, this is not possible with the operator as defined in Eq.~\eqref{eq:observation_operator} as it requires the mapping to a finite-dimensional vector.
Also restricting the observation to test- and training points does not lead \pname{FIC}.
It remains an open question whether \pname{FIC} fits into our proposed kernel approximation scheme.
Empirically, we found that replacing heuristically the approximate diagonal with the exact leads to worse performance.

\setlength{\figwidth}{\defaultwidthfactor\textwidth}
\setlength{\figheight}{0.8\figwidth}
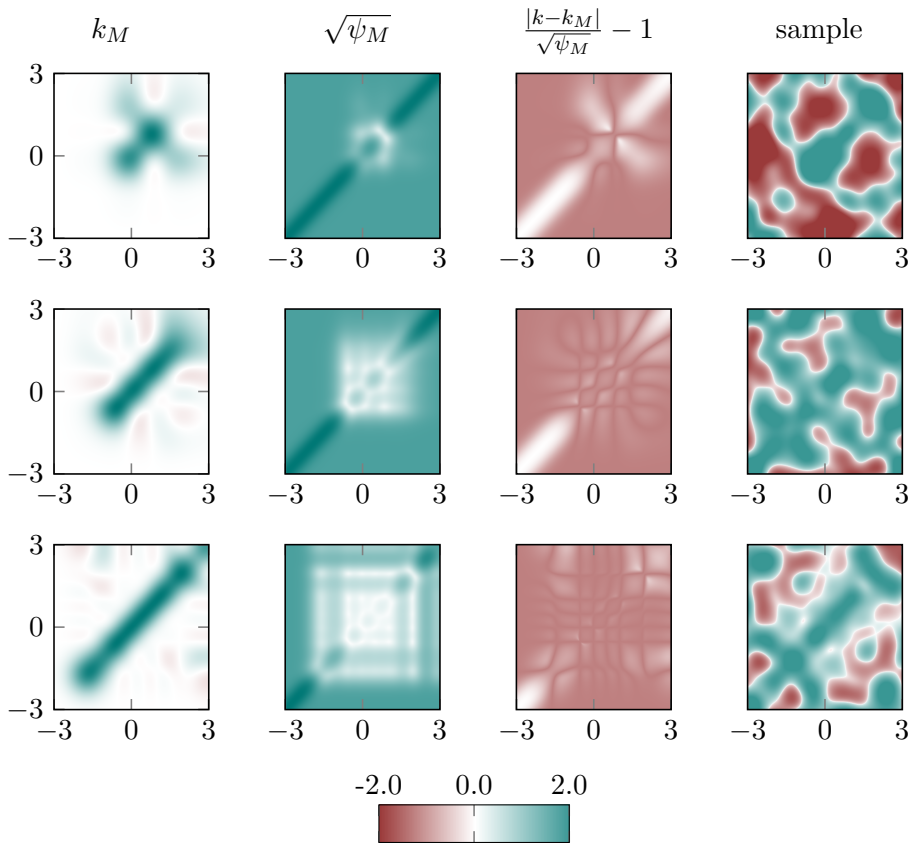
\begin{figure}
\tempdisable{
\centering
\begin{tabular}{cccc}
$\km$ & $\sqrt{\ksigm}$ & $\frac{|k-\km|}{\sqrt{\ksigm}}-1$ & sample \\ 
%
%
\begin{tikzpicture}

\begin{axis}[%
width=0.75\figwidth,
height=\figheight,
at={(0\figwidth,0\figheight)},
scale only axis,
point meta min=-2.0000,
point meta max=2.0000,
axis on top,
xmin=-3.0060,
xmax=3.0060,
xtick={-3,  0,  3},
ymin=-3.0060,
ymax=3.0060,
ytick={-3,  0,  3},
axis background/.style={fill=white},
mystyle
]
\addplot [forget plot] graphics [xmin=-3.0060,xmax=3.0060,ymin=-3.0060,ymax=3.0060] {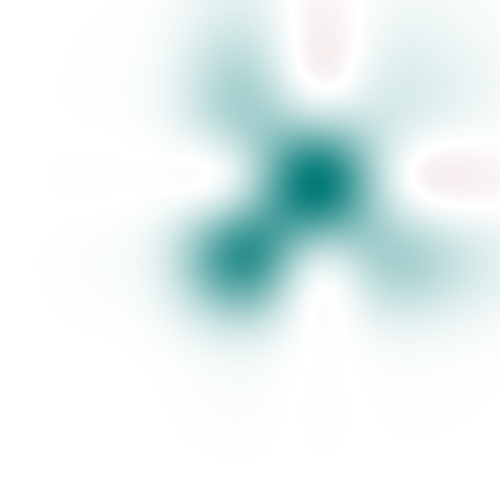};
\end{axis}
\end{tikzpicture}%
 &
%
%
\begin{tikzpicture}

\begin{axis}[%
width=0.75\figwidth,
height=\figheight,
at={(0\figwidth,0\figheight)},
scale only axis,
point meta min=-2.0000,
point meta max=2.0000,
axis on top,
xmin=-3.0060,
xmax=3.0060,
xtick={-3,  0,  3},
ymin=-3.0060,
ymax=3.0060,
ytick={\empty},
axis background/.style={fill=white},
mystyle
]
\addplot [forget plot] graphics [xmin=-3.0060,xmax=3.0060,ymin=-3.0060,ymax=3.0060] {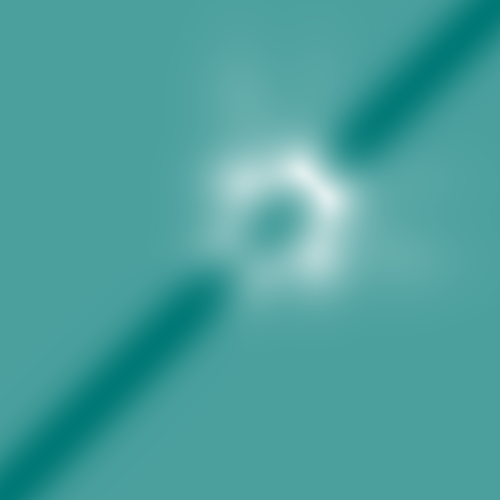};
\end{axis}
\end{tikzpicture}%
 &
%
%
\begin{tikzpicture}

\begin{axis}[%
width=0.75\figwidth,
height=\figheight,
at={(0\figwidth,0\figheight)},
scale only axis,
point meta min=-2.0000,
point meta max=2.0000,
axis on top,
xmin=-3.0060,
xmax=3.0060,
xtick={-3,  0,  3},
ymin=-3.0060,
ymax=3.0060,
ytick={\empty},
axis background/.style={fill=white},
mystyle
]
\addplot [forget plot] graphics [xmin=-3.0060,xmax=3.0060,ymin=-3.0060,ymax=3.0060] {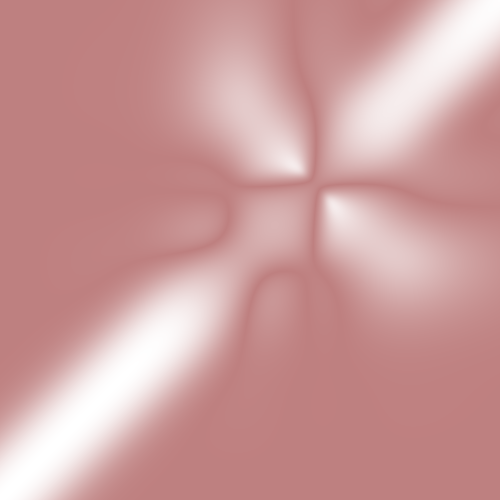};
\end{axis}
\end{tikzpicture}%
 &
%
%
\begin{tikzpicture}

\begin{axis}[%
width=0.75\figwidth,
height=\figheight,
at={(0\figwidth,0\figheight)},
scale only axis,
point meta min=-2.0000,
point meta max=2.0000,
axis on top,
xmin=-3.0060,
xmax=3.0060,
xtick={-3,  0,  3},
ymin=-3.0060,
ymax=3.0060,
ytick={\empty},
axis background/.style={fill=white},
mystyle
]
\addplot [forget plot] graphics [xmin=-3.0060,xmax=3.0060,ymin=-3.0060,ymax=3.0060] {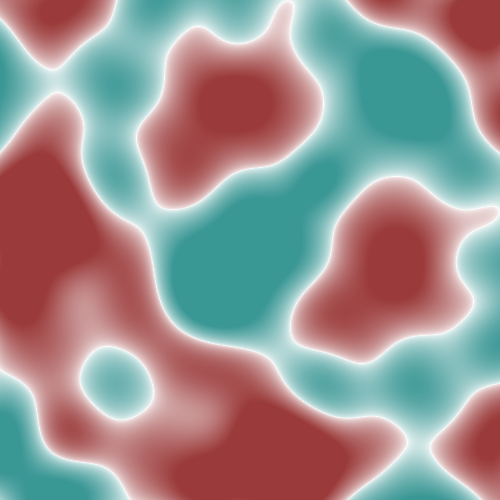};
\end{axis}
\end{tikzpicture}%
\\
%
%
\begin{tikzpicture}

\begin{axis}[%
width=0.75\figwidth,
height=\figheight,
at={(0\figwidth,0\figheight)},
scale only axis,
point meta min=-2.0000,
point meta max=2.0000,
axis on top,
xmin=-3.0060,
xmax=3.0060,
xtick={-3,  0,  3},
ymin=-3.0060,
ymax=3.0060,
ytick={-3,  0,  3},
axis background/.style={fill=white},
mystyle
]
\addplot [forget plot] graphics [xmin=-3.0060,xmax=3.0060,ymin=-3.0060,ymax=3.0060] {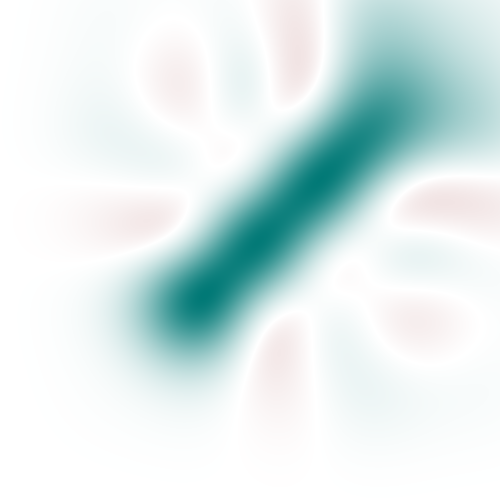};
\end{axis}
\end{tikzpicture}%
 &
%
%
\begin{tikzpicture}

\begin{axis}[%
width=0.75\figwidth,
height=\figheight,
at={(0\figwidth,0\figheight)},
scale only axis,
point meta min=-2.0000,
point meta max=2.0000,
axis on top,
xmin=-3.0060,
xmax=3.0060,
xtick={-3,  0,  3},
ymin=-3.0060,
ymax=3.0060,
ytick={\empty},
axis background/.style={fill=white},
mystyle
]
\addplot [forget plot] graphics [xmin=-3.0060,xmax=3.0060,ymin=-3.0060,ymax=3.0060] {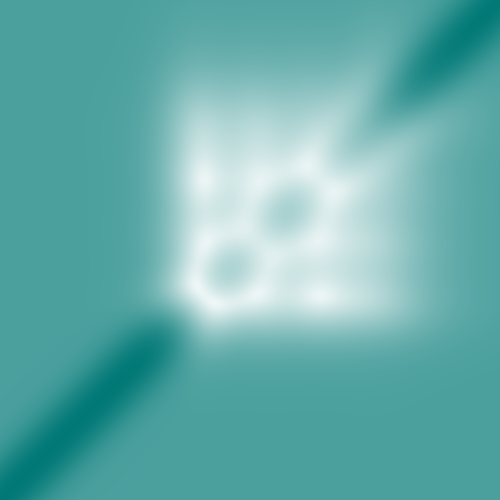};
\end{axis}
\end{tikzpicture}%
 &
%
%
\begin{tikzpicture}

\begin{axis}[%
width=0.75\figwidth,
height=\figheight,
at={(0\figwidth,0\figheight)},
scale only axis,
point meta min=-2.0000,
point meta max=2.0000,
axis on top,
xmin=-3.0060,
xmax=3.0060,
xtick={-3,  0,  3},
ymin=-3.0060,
ymax=3.0060,
ytick={\empty},
axis background/.style={fill=white},
mystyle
]
\addplot [forget plot] graphics [xmin=-3.0060,xmax=3.0060,ymin=-3.0060,ymax=3.0060] {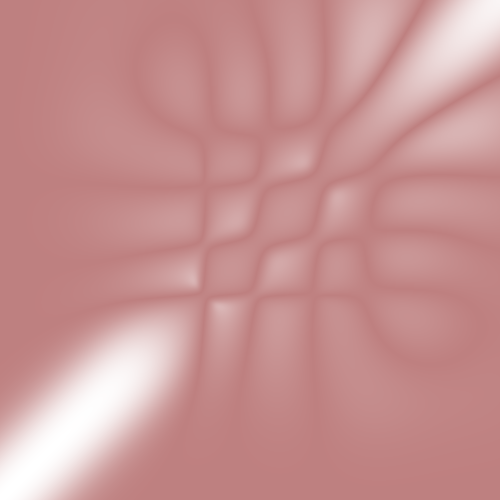};
\end{axis}
\end{tikzpicture}%
 &
%
%
\begin{tikzpicture}

\begin{axis}[%
width=0.75\figwidth,
height=\figheight,
at={(0\figwidth,0\figheight)},
scale only axis,
point meta min=-2.0000,
point meta max=2.0000,
axis on top,
xmin=-3.0060,
xmax=3.0060,
xtick={-3,  0,  3},
ymin=-3.0060,
ymax=3.0060,
ytick={\empty},
axis background/.style={fill=white},
mystyle
]
\addplot [forget plot] graphics [xmin=-3.0060,xmax=3.0060,ymin=-3.0060,ymax=3.0060] {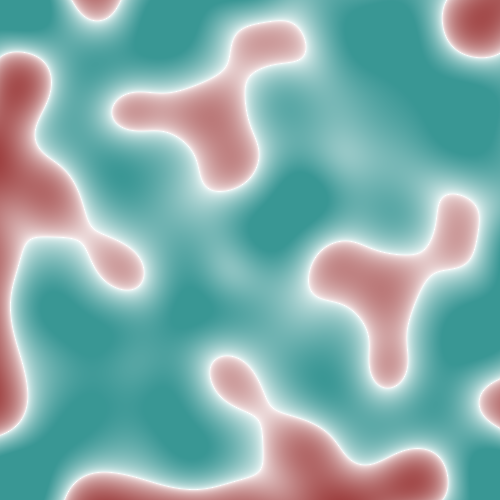};
\end{axis}
\end{tikzpicture}%
\\
%
%
\begin{tikzpicture}

\begin{axis}[%
width=0.75\figwidth,
height=\figheight,
at={(0\figwidth,0\figheight)},
scale only axis,
point meta min=-2.0000,
point meta max=2.0000,
axis on top,
xmin=-3.0060,
xmax=3.0060,
xtick={-3,  0,  3},
ymin=-3.0060,
ymax=3.0060,
ytick={-3,  0,  3},
axis background/.style={fill=white},
mystyle
]
\addplot [forget plot] graphics [xmin=-3.0060,xmax=3.0060,ymin=-3.0060,ymax=3.0060] {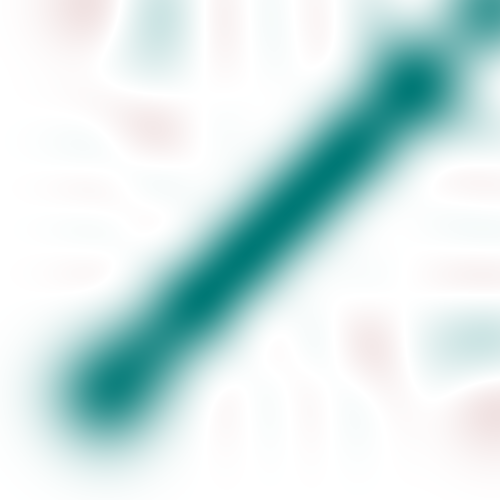};
\end{axis}
\end{tikzpicture}%
 &
%
%
\begin{tikzpicture}

\begin{axis}[%
width=0.75\figwidth,
height=\figheight,
at={(0\figwidth,0\figheight)},
scale only axis,
point meta min=-2.0000,
point meta max=2.0000,
axis on top,
xmin=-3.0060,
xmax=3.0060,
xtick={-3,  0,  3},
ymin=-3.0060,
ymax=3.0060,
ytick={\empty},
axis background/.style={fill=white},
mystyle
]
\addplot [forget plot] graphics [xmin=-3.0060,xmax=3.0060,ymin=-3.0060,ymax=3.0060] {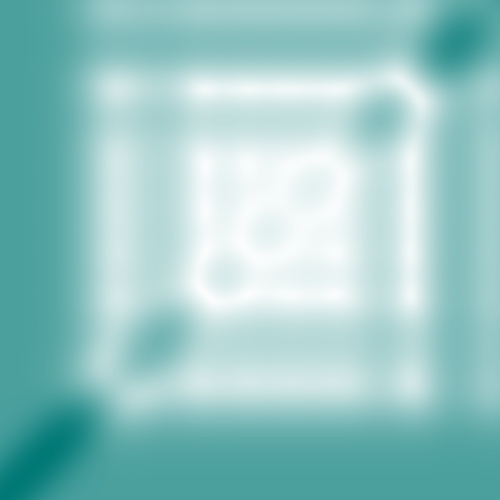};
\end{axis}
\end{tikzpicture}%
 &
%
%
\begin{tikzpicture}

\begin{axis}[%
width=0.75\figwidth,
height=\figheight,
at={(0\figwidth,0\figheight)},
scale only axis,
point meta min=-2.0000,
point meta max=2.0000,
axis on top,
xmin=-3.0060,
xmax=3.0060,
xtick={-3,  0,  3},
ymin=-3.0060,
ymax=3.0060,
ytick={\empty},
axis background/.style={fill=white},
mystyle
]
\addplot [forget plot] graphics [xmin=-3.0060,xmax=3.0060,ymin=-3.0060,ymax=3.0060] {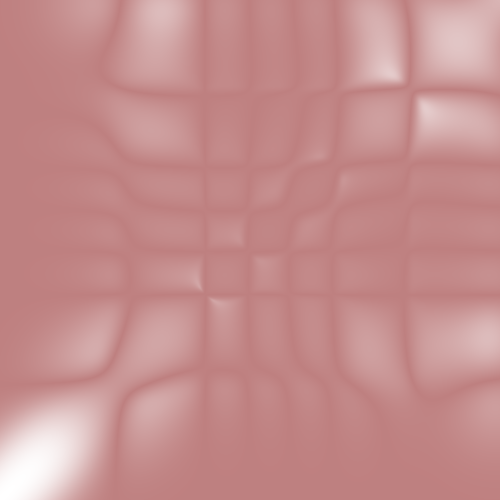};
\end{axis}
\end{tikzpicture}%
 &
%
%
\begin{tikzpicture}

\begin{axis}[%
width=0.75\figwidth,
height=\figheight,
at={(0\figwidth,0\figheight)},
scale only axis,
point meta min=-2.0000,
point meta max=2.0000,
axis on top,
xmin=-3.0060,
xmax=3.0060,
xtick={-3,  0,  3},
ymin=-3.0060,
ymax=3.0060,
ytick={\empty},
axis background/.style={fill=white},
mystyle
]
\addplot [forget plot] graphics [xmin=-3.0060,xmax=3.0060,ymin=-3.0060,ymax=3.0060] {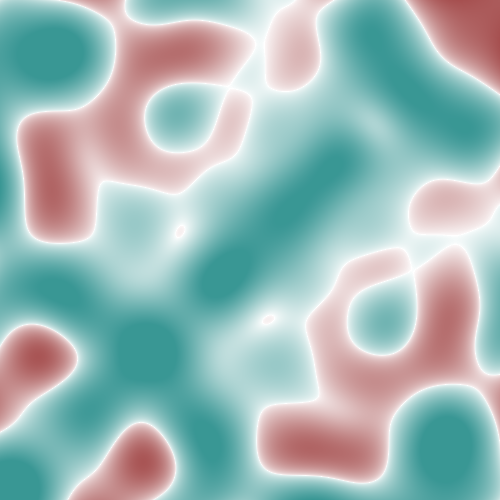};
\end{axis}
\end{tikzpicture}%
 \end{tabular}
 \hphantom{$-3$}
%
%
\begin{tikzpicture}

\begin{axis}[%
width=0.951\figwidth,
height=\figheight,
at={(0\figwidth,0\figheight)},
scale only axis,
point meta min=0.0000,
point meta max=1.0000,
xmin=-2.0000,
xmax=2.0000,
ymin=-2.0000,
ymax=2.0000,
hide axis,
mystyle,
colormap={mymap}{[1pt] rgb(0pt)=(0.605518,0.225595,0.225595); rgb(1011pt)=(1,1,1); rgb(2023pt)=(0.225403,0.590781,0.582028)},
colorbar horizontal,
colorbar style={at={(0.5,0.03)},anchor=south,xtick={0, 0.5, 1}, xticklabels={-2.0, 0.0, 2.0},xticklabel pos=upper,width=0.97*\pgfkeysvalueof{/pgfplots/parent axis width}}
]

\addplot[area legend,solid,draw=black,fill=black,forget plot]
table[row sep=crcr] {%
x	y\\
0.0000	0.0000\\
}--cycle;
\end{axis}
\end{tikzpicture}%
 }
\caption{progression of the posterior (Eq.~\ref{eq:kmcg_kernel}) for \kmcg{} on the toy example from Figure \ref{fig:toy_example} for $P=2, 4$ and $8$ conjugate gradients steps.
The columns show from left to right: mean, standard deviation, standardized error
(white refers to perfect calibration, green to overconfidence and red to underconfidence) and a sample.
}
\label{fig:kmcg_progression}
\end{figure}

\begin{restatable}[Projected Bayes Regressor]{proposition}{PBRprop}
\label{thm:pbr}
Consider the prior of \cref{eq:prior} with $\kmu\ce 0$ and $w\ce k$ and the likelihood defined in \cref{ex:rks}.
Let $\lambda_1$ to $\lambda_P$ be the largest eigenvalues of the kernel $k$ (w.r.t~to the Mercer expansion) and assume the inputs $\vec x_1, ..., \vec x_N$ are independent and identical draws from $\nu$.
Then the posterior kernel $\km$ leads to the \emph{Projected Bayes Regressor} \citep{Trecate.1999}.
\end{restatable}
The proof is part of \cref{app:pbr}.


\section{Conjugate Gradients for Kernel Machines}
\label{sec:intuition}
The previous section introduced a probabilistic estimation rule for the kernel $k$.
This section presents another data-collection approach using conjugate gradients that leads to a new approximation algorithm: \pname{kernel machine conjugate gradients} (\kmcg{}).

The interest to use conjugate gradients for kernel machines goes back to more than 25 years \citep{skilling1993numerical} and is still continuing \citep{davies2015effectiveGPimplementation, filippone2015ulisse}.
Albeit quadratic costs per step, CG has advantages over many of the approximation methods referenced in the introduction. 
CG has only one parameter, the desired precision, which is more natural than \eg{}the number of inducing inputs for inducing point methods \citep{quinonero2005unifying}.
This means, the computational budget of CG is not fixed in advance but varies as necessary for the problem at hand.

\subsection{Conjugate Gradients}
\label{sec:cg}
Conjugate gradients (Algorithm \ref{algo:cg}) is an iterative solver for linear equation systems $\vec A\vec x=\vec b$ where $\vec A\in\Re^{N\times N}$ is a real, symmetric and positive definite matrix \citep{hestenes1952methods}.
In theory, CG returns the exact solution $\vec x$ after $N$ steps.
In practice, CG is used as approximate solver and can provide good approximations to $\vec x$ in significantly less than $N$ steps.

The costs of running CG are dominated by a matrix-vector multiplication in each step which in general has complexity $\O(N^2)$.
The number of necessary steps depends on the eigenvalues $\lambda_1>...>\lambda_N$ of $\vec A$.
The following summary of the properties of CG is an excerpt from \citet[Chapter 5.1]{nocedal1999numerical}.
We use the notation $\|\vec x\|_{\vec \Lambda}^2\ce\vec x\Trans\vec \Lambda\vec x$ for any symmetric and positive definite matrix $\vec \Lambda$. 
The $\vec A$-error of CG decreases in each step with
\begin{align}
\|\vec x_{k+1}-\vec x\|_{\vec A}^2\leq \left(\frac{\lambda_{N-k}-\lambda_1}{\lambda_{N-k}+\lambda_1}\right)^2||\vec x_0-\vec x||_{\vec A}^2
\end{align}
and one can show that if $\vec A$ has at most $r$ distinct eigenvalues, then CG terminates after $r$ steps with the exact solution.
Thus conjugate gradients is particularly advantageous if the eigenvalues of $\vec A$ are clustered or decay rapidly.

\newlength{\maxwidth}
\newcommand{\algalign}[2]
{\makebox[\maxwidth][r]{$#1{}$}$\gets {}#2$}
\begin{algorithm}
\caption{Conjugate Gradients}
\label{algo:cg}
\begin{algorithmic}[1]
\settowidth{\maxwidth}{$\vec r_0$}
\Procedure{ConjugateGradients}{$\vec A$, $\vec b$, $\vec x_0$, $\epsilon$}
\State \algalign{\vec r_0}{\vec A \vec x_0 - \vec b} \Comment{The initial residual ...}
\State \algalign{\vec s_0}{-\vec r_0} \Comment{... is the first search direction.}
\State \algalign{i}{0} 
\While{$||\vec r_i||_2 > \epsilon$}
\settowidth{\maxwidth}{$\vec x_{i+1}$}
\State \algalign{\vec z_i}{\vec A \vec s_i} \Comment{the most expensive step: $\O(N^2)$ matrix-multiplication}
\State \algalign{\alpha_i}{\frac{\vec r_i\Trans\vec r_i}{\vec s_i\Trans \vec z_i}} \Comment{optimal linesearch along $\vec s_i$ for $\phi(\vec x)\ce\vec x\Trans \vec A\vec x- 2\vec x\Trans \vec b$}
\State \algalign{\vec x_{i+1}}{\vec x_{i}+\alpha_i \vec s_i} \Comment{update to the solution}
\State \algalign{\vec r_{i+1}}{\vec r_{i} + \alpha_i\vec z_i} \Comment{analogue update to the residual}
\State \algalign{\vec s_{i+1}}{-\vec r_{i+1}+\frac{\vec r_{i+1}\Trans\vec r_{i+1}}{\vec r_i\Trans \vec r_i}\vec s_i} \Comment{Gram-Schmidt applied to the new residual}
\State \algalign{i}{i+1} 
\EndWhile
\State \Return $\vec x_i$ 
\EndProcedure
\end{algorithmic}
\end{algorithm}

\subsection{Kernel-machine Conjugate Gradients}
\label{sec:kmcg}
Our approach is to run conjugate gradients for $P$ steps on a kernel matrix of size $M$ and to treat the matrix multiplications ($\vec z_i$ in Algorithm \ref{algo:cg}) as observations in the model presented in Section \ref{sec:model}. 
Formally the likelihood is defined similar to the \SoR{} likelihood (Example \ref{ex:dtc}) albeit scaled.

\begin{definition}[Conjugate-gradients likelihood]
Choose a subset $\vec X_M$ of size $M$ from $\vec X$ and denote as $\vec y_M\in \Re^M$ the vector that contains the corresponding entries of $\vec y$.
Run conjugate gradients (Algorithm \ref{algo:cg} on p.~\pageref{algo:cg}) with $\vec x_0\ce \vec 0$, $\vec A=k(\vec X_M, \vec X_M)$, $\vec b=\vec y_M$ and $\epsilon\ce0.01||\vec b||_2$.
In Equation \eqref{eq:observation_operator} set
\begin{align}
\label{eq:lk_kmcg}
p_{i}(\vec x)\ce&\sum_{j=1}^M s_j \delta(\vec x - \vec x_j)
\end{align}
where $s_j$ is the $j$-th entry of vector $\vec s_i$ in iteration $i$ of Algorithm \ref{algo:cg}. 
\end{definition}
\begin{remark}
\kmcg{} uses only the CG search directions $\vec s_1, ..., \vec s_P$ and \emph{not} the solution $\hat{\vec x}$. 
\end{remark}
\pn{
The relationship between the CG $\hat{\vec x}$ and the \kmcg{} $\hat{\vec x}$.
\begin{align}
\hat{\vec x}_{CG}&=\vec S(\vec S\Trans (\vec K + \noise \vec I)\vec S)^{-1}\vec S\Trans\vec y
\\\hat{\vec x}_{\kmcg{}}&=(\vec K\vec S(\vec S\Trans \vec K\vec S)\vec S\Trans \vec K+\noise \vec I)^{-1}\vec y
\\&=\frac{1}{\noise}\left(\vec y-\vec K\vec S(\vec S\Trans \vec K\vec K\vec S+\noise \vec S\Trans \vec K\vec S)^{-1}\vec S\Trans \vec K\vec y\right)
\end{align}
and the $\vec S$ differ in between equations.
}
\newcommand{\Gram}{\vec R\Trans \vec R+\noise \S\Trans\Km\S}
Using this likelihood, the resulting approximate kernel (Eq.~\eqref{eq:km}) and approximate Equations are (\cref{thm:kernel_post}, p.~\pageref{thm:kernel_post}):
\begin{align}
\label{eq:kmcg_kernel}
\hat{k}_M(\vec x_{*}, \vec x_{**})&=k(\vec x_*, \vec X_M)\vec S(\vec S\Trans \Km \vec S)^{-1}\vec S\Trans k(\vec X_M, \vec x_{**})
\\\hat{f}(\vec x_*) &=  k(\vec x_*, \vec X_M)\vec S(\Gram{})^{-1}\vec R\Trans \vec y \label{eq:kmcg_mean}
\\\label{eq:kmcg_var}
\hat{c}(\vec x_*, \vec x_{**}) &=  k(\vec x_*, \vec x_{**})-k(\vec x_*, \vec X_M)\S(\S\Trans\Km\S)^{-1}\S\Trans k(\vec X_M, \vec x_{**})
\\&\quad +\noise k(\vec x_*, \vec X_M)\S\left(\Gram{}\right)^{-1}\S\Trans k(\vec X_M, \vec x_{**}) \notag
\\\label{eq:kmcg_nlZ}
\ln \hat{Z}&=\frac{1}{2\sigma^2}(\vec y\Trans \vec y-\vec y\Trans \vec R(\Gram{})^{-1}\vec R\Trans \vec y)
\\&\quad +\frac{1}{2}\ln|\Gram{}|-\frac{1}{2}|\S\Trans \Km\S|\notag
\\&\quad+\frac{1}{2}(N-P)\ln \sigma^2 +\frac{1}{2}N\ln 2\pi \notag
\end{align}
where $\S\ce \begin{bmatrix}\vec s_1 & ... & \vec s_P\end{bmatrix}, \vec R\ce k(\vec X_M, \vec X)\S$ and $P$ is the number of CG-iterations.
We call this approximation \pname{kernel machine conjugate gradients} (\kmcg{}).
\begin{algorithm}
\caption{Kernel Machine Conjugate Gradients}
\label{algo:kmcg}
\algnewcommand{\LeftComment}[1]{\State \(\triangleright\) #1}
\begin{algorithmic}[1]
\Procedure{KMCG}{$k$, $\vec X$, $\vec y$, $\sigma^2$, $\epsilon$}
\LeftComment{We assume (w.l.o.g.) that the inducing inputs are a subset of $\vec X$, denoted by $\vec X_M$.}
\LeftComment{Let $\vec y_M$ the be  corresponding entries of $\vec y$.}
\State Conjugate Gradients$(k(\vec X_M, \vec X_M), \vec y_M, \epsilon)$ \Comment{ignore solution $\hat{\vec x}$}
\State $\vec S\gets [\vec s_1, ..., \vec s_P]$ \Comment{collect CG search directions}
\State $\vec Z\gets [\vec z_1, ..., \vec z_P]$ \Comment{$\vec Z = \vec K_M \vec S$}
\If{$M<N$}
\State $\vec R \gets k(\vec X, \vec X_M)\vec S$ 
\Else
\State $\vec R \gets \vec Z$ \Comment{When $\vec X_M=\vec X$ above matrix multiplication is not necessary.}
\EndIf
\State $\vec L_1 \gets \operatorname{chol}(\vec S\Trans \vec Z)$ \Comment{precompute required Choleskies}
\State $\vec L_2 \gets \operatorname{chol}(\sigma^2\vec S\Trans \vec Z+\vec R\Trans \vec R)$
\State evaluate Eqs.~\eqref{eq:kmcg_mean} to \eqref{eq:kmcg_nlZ}
\EndProcedure
\end{algorithmic}
\end{algorithm}
\subsection{Properties}
Figure \ref{fig:kmcg_progression} shows how the approximation to the kernel progresses for the toy example from Figure \ref{fig:toy_example}.
Computing the Cholesky of $\Gram$ costs $\O(NMP)$.
After that evaluating the mean prediction is possible in $\O(M)$ and the variance in $\O(MP)$.

In case $P=M$, \kmcg{} reduces to \SoR{} since all occurrences of $\S$ in Equation \eqref{eq:kmcg_kernel} cancel and what remains is the \SoR{} kernel (Equation \ref{eq:k_sor}).
If $\Km$ has a favorable distribution of eigenvalues such that conjugate gradients terminates in less than $M$ steps (\cf{}\cref{sec:cg}), \kmcg{} can be used to speed up \SoR{}.\footnote{The same applies to related methods such as \pname{DTC} \citep{quinonero2005unifying} and Titsias' method \citep{titsias09:_variat_learn_induc_variab_spars_gauss_proces}.}
In practice, this kind of advantage can only be expected to be beneficial when realized in low-level code.
The level of efficiency of existing low-level linear algebra routines makes it challenging to evaluate this area.

Recall that the computational complexity of CG for the solution of Eq.~\eqref{eq:prior} in $P$ iterations is $\O(N^2P)$, 
that of inducing point methods with $M$ inducing inputs is $\O(NM^2)$, and \kmcg{} running for $P$ iterations on $M$ inducing points has complexity $\O(NMP)$. 
While the main point of the present paper is to ``fix'' problems of CG in kernel machines, this structure hints at an interesting side-observation: 
Restricting the number of steps $P$ in advance can then allow to increase the number of inducing points $M$ beyond what would otherwise be feasible with standard inducing input methods. 
The subsequent evaluation section is dedicated to the case $M=N$, \ie{}using the whole dataset which places \kmcg{} in direct competition to plain conjugate gradients.

\subsubsection{Relationship to the Nadaraya-Watson estimator}
Taking only one step ($P=1$) implies $\vec S=\vec y_M$ and Equation \eqref{eq:kmcg_mean} takes the following form
$$\hat{f}(\vec x_*)=\alpha \sum_{m=1}^M k(\vec x_m, \vec x_*) y_m$$
where $\alpha=\frac{\vec y_M\Trans \vec K_M\vec y_M}{\noise \vec y_M\Trans \vec K_M\vec y_M + \vec y_M\Trans \vec K_M\vec K_M\vec y_M}$.
The equation bears resemblance to the Nadaraya-Watson estimator \cite[p.~301f]{bishop2006pattern}:
a sum over all training targets weighted by the similarity of the corresponding input to the test input.
However, the scaling-factor $\alpha$ is different.
Since conjugate gradients solves the linear system for the mean prediction, it is to be expected that this might incur a trade-off to the approximation of the variance.
See Section \ref{sec:exp_cg} for an empirical evaluation of the quality of the variance estimate.

\pn{
The mean of a Gaussian is the best least-squares approximation to ... in the span of the search directions.
But this still does not answer why we should choose the search direction.
Of CG we know that $\vec{\hat{x}}$ minimizes the polynomial...
This does not apply here, because we throw it away.
The analysis of above property is complicated.
The aim of this work is to explore empirically if there could be something interesting.
}

\subsubsection{Uncertainty}
\label{sec:uncertainty}
\pn{
We provide an approach to derive an uncertainty. 
This arises naturally from our setting and that is the reason why we mention it here even though it might not be usable.
Maybe the approach is helpful for someone else.
}
\pn{
Things I want to/ should write here.
\begin{itemize}
\item This section extends the analysis contained in Hennig2015 an Bartels2015
\item The uncertainty reduces to zero for known entries. (Well calibratedness!)
\item How would one actively select where to make observations?
\item Gamma prior
\item scale on the diagonal
\item bags of underconfidence where there is no data 
\begin{itemize}
  \item How does this deviate from what has been said in \cite{Hennig_LinearSolvers2015}?
  \item Can extend it by looking at posterior mean and not just prior.
\end{itemize}
\item Need to do a case distinction
\begin{itemize}
  \item SoR
  \item Spectral Methods?
\end{itemize}
\item What about non-stationary kernels? Could try periodic times polynomial.
\item 
\end{itemize}
}
In addition to the posterior mean $\km$, the Gaussian formulation of the approximation problem also provides a posterior variance $\ksigm$. 
It is a natural question to which degree this object can be interpreted as a notion of uncertainty or, more specifically, as an estimate of the square error $(k - \km)^2$. 
This section provides an analysis of this covariance for \kmcg{}, showing it to be an outer bound on the true error.
Figure \ref{fig:kmcg_progression} visualizes this for the toy dataset from Figure \ref{fig:toy_example}.

\begin{restatable}[relative error bound]{proposition}{VARprop}
The relative size of estimation error and error estimate is bounded from above by 2.
\begin{align}
\label{eq:sq_err_over_var}
\frac{(k(\vec x, \vec z)-\km(\vec x, \vec z))^2}{\ksigm(k(\vec x, \vec z), k(\vec x, \vec z))}\leq 2
\end{align}
\end{restatable}

\begin{proof}
To simplify notation, define $\vec k_x\Trans\ce k(\vec x, \vec X)$ and $\vec G \ce \vec S(\vec S\Trans \vec K\vec S)^{-1}\vec S\Trans$.
For \kmcg{} posterior mean and variance evalute to (Appendix \ref{app:dtc}):
\begin{align}
\km(\vec x, \vec z)&=\vec k_x\Trans\vec G\vec k_z,
\\\ksigm(k(\vec x, \vec z), k(\vec x, \vec z))
&=\nicefrac{1}{2}\left(k(\vec x, \vec x)k(\vec z, \vec z) + k(\vec x, \vec z)^2-\vec k_x\Trans\vec G \vec k_x \vec k_z\Trans\vec G \vec k_z-(\vec k_x\Trans\vec G \vec k_z)^2\right)\label{eq:scg_var}
\\&=\nicefrac{1}{2}\left(k(\vec x, \vec x)k(\vec z, \vec z) + k(\vec x, \vec z)^2-\km(\vec x, \vec x)\km(\vec z, \vec z)-\km(\vec x, \vec z)^2\right).
\end{align}
As a variance $\ksigm(k(\vec x, \vec x), k(\vec x, \vec x))$ is always larger than $0$ which implies $k(\vec x, \vec x)\geq \km(\vec x, \vec x)$ for all $\vec x$.
This allows to lower bound $\ksigm(k(\vec x, \vec z), k(\vec x, \vec z))$ by $\frac{1}{2} k(\vec x, \vec z)^2-\frac{1}{2}\km(\vec x, \vec z)^2$ leading to
\begin{align}
\frac{(k(\vec x, \vec z)-\km(\vec x, \vec z))^2}{\ksigm(k(\vec x, \vec z), k(\vec x, \vec z))}
\leq&2\frac{(k(\vec x, \vec z)-\km(\vec x, \vec z)^2}{k(\vec x, \vec z)^2-\km(\vec x, \vec z)^2}
\\&=2\frac{(k(\vec x, \vec z)-\km(\vec x, \vec z)^2}{(k(\vec x, \vec z)-\km(\vec x, \vec z))(k(\vec x, \vec z)+\km(\vec x, \vec z))}
\\&=2\frac{|k(\vec x, \vec z)-\km(\vec x, \vec z)|}{k(\vec x, \vec z)+\km(\vec x, \vec z)}
\\\leq & 2.
\end{align}
\pn{I would not be surprised if this bound was 1.}
\end{proof}

\subsection{Related Work} 
\label{sec:related_work}
\qtp{All of the papers cited here only sound like could be related but they are not.
I feel that pointing out why certain works are NOT related is equally important to a scientific work.
It is an aspect I love about your Entropy Search paper.}

In terms of using conjugate gradients for kernel machines there is related work by \citet{filippone2015ulisse}.
Their algorithm ULISSE is aimed at the estimation of the marginal likelihood $p(\vec\theta\mid \vec y)$ where $\vec \theta$ are hyper-parameters of the kernel $k$.
They use a randomized conjugate gradients to estimate gradients of the log-marginal likelihood (Eq.~\eqref{eq:llh}) which in combination with Stochastic Gradient Langevin Dynamics (SGLD) \citep{Welling2011sgld} allows to sample from $p(\vec\theta\mid \vec y)$.
Our work is complementary to ULISSE.
While running CG the matrix multiplications the inference perspective in Section \ref{sec:model} can be used to build a low-rank approximation of the kernel matrix which can serve as preconditioner for the next SGLD step.


\qtp{KISS-GP is not really related to our work. The purpose of this section is more to show that I had a look for related literature and this is what I found.}
Using the Kronecker product for efficient inference has been explored before for example in the KISS-GP framework \citep{wilson2015kissGP}.
The difference to this work is that \citet{wilson2015kissGP} factorize the kernel matrix $\vec K$ into a Kronecker-product where here it is the covariance matrix of the prior $\ksig(\vec K, \vec K)$ over the kernel that has Kronecker structure (cf.~Eq.~\ref{eq:prior}).
A synergy between their and our approach is hard to imagine.
However, the follow-up work by \cite{Pleiss2018LOVE} uses Lanczos iteration to build a low-rank approximation of a kernel matrix $\vec C$ for the variance prediction.
Presumably, one could use instead \kmcg{}.


\section{Empirical Comparison of Conjugate Gradients and Kernel Machine Conjugate Gradients}
\label{sec:exp_cg}

This section elaborates the conceptual differences between CG and \kmcg{} and then compares both algorithms with numerical experiments.
Consider Equation \eqref{eq:mean}, restated below for convenience. 
\begin{align}
\eqrepeat{eq:mean}
\end{align}
CG computes an approximation to $(\vec K+\noise I)^{-1}\vec y$ and uses the exact $\vec k_*$.
In contrast, \kmcg{} computes an approximation to $k$ and substitutes $\vec k_*$ as well.
That the systematic replacement of the kernel can be of importance has been noted before by \citet[p. 177]{RasmussenWilliams} when comparing \SoR{} and the Nystr\"om method \citep{williams2001nystroem}.
The \SoR{} method approximates $k$ with the kernel in Equation \eqref{eq:k_sor}.
In contrast Nystr\"om uses the exact $\vec k_*$ such that the predictive variance (Eq.~\ref{eq:var}) can become negative.
They further observed that for large $M$, Nystr\"om and \SoR{} have a similar performance, yet for small $M$ Nystr\"om performs poorly.
We made the same observations for CG and \kmcg{} in the following comparison on
common regression problems.

Classical conjugate gradients is used to solve the equations $(\vec K +\noise\vec I)\vec \alpha = \vec y$.
In contrast, since the goal of \kmcg{} is to learn an approximation to the kernel, the algorithm runs conjugate gradients on $\vec K \vec \alpha = \vec y$, \ie{}without noise term.
Both methods were evaluated in terms of the average relative error 
\begin{align}
\label{eq:relerr}
\relerr \ce \frac{1}{n_*}\sum_{k=1}^{n_*}\left|\frac{\bar{f}(\vec x_{*,k})-\hat{f}(\vec x_{*,k})}{\bar{f}(\vec x_{*,k})}\right|,
\end{align}
where $\vec x_{*,k}$ is a test input not part of the training set. 

The text book version of conjugate gradients in Algorithm \ref{algo:cg} is known to be numerically unstable\footnote{see additional results in Appendix \ref{app:experiments_instability}} \cite[p.~635]{golub1996matrix} and there exist different strategies to cope with this problem.
We refer the interested reader to \citet[p.~565f]{golub1996matrix} and the references therein.
To explore the potential of our method, we bypass this implementation issue using the slowest\footnote{Computing the exact solution is actually faster.} yet most stable solution: complete reorthogonalization\footnote{We experimented with selective reorthogonalization \citep{Simon1984Reorthogonalization} but found it in our experiments to be slower than full reorthogonalization.} \cite[p.~564]{golub1996matrix} and the explicit projection-method formulation \cite[p.~135 Eq.~(5.7)]{saad2003iterative} to compute $\vec \alpha$.
Therefore the following comparison will be conceptually, \ie{}over the number of conjugate gradient steps.
For completeness, Appendix \ref{app:experiments_realtime} contains results how \kmcg{} performs in wall-clock time.
Often the \baseline{} methods converge faster since block-matrix multiplication is faster than looped matrix-vector multiplication.

Baseline methods are the \pname{Fully Independent Training Conditional} (\FITC{}) approximation \citep{quinonero2005unifying}  and the \pname{Variational Free Energy} (\VFE{}) method \citep{Titsias.2009}, with inducing inputs randomly selected from the dataset as recommended by \cite{chalupka2013comparison}.
The \baseline{} runs were repeated \repetitions{} times and besides the average, each figure shows also the progressive minimum and maximum over all runs, to take into account for more elaborate inducing input selection schemes.

In all our experiments, we used two popular stationary kernel functions: automatic relevance determination (ARD) Squared Exponential and ARD Mat\'ern $\nicefrac{5}{2}$ \citep[p. 83f, p. 106]{RasmussenWilliams}, 
\begin{align}
\label{eq:sq_exp}
k_{SE}(d(\vec x, \vec z;\vec \Lambda))&= \theta_f \exp{\left(-\frac{1}{2}d^2\right)}
\\\label{eq:matern}
k_{52}(d(\vec x,\vec z; \vec \Lambda))&=\theta_f\left(1+\sqrt{5}d+\frac{5}{3}d^2\right)\exp{\left(-\sqrt{5}d\right)}
\end{align}
where $d=d(\vec x, \vec z; \vec \Lambda)=||\vec x-\vec z||_{{\vec \Lambda}}$ and $\vec\Lambda$ is a diagonal matrix.
All experiments were executed with Matlab R2019a on an Intel i7 CPU with 32 Gigabytes of RAM running Ubuntu 18.04. 

\subsection{Common Regression Datasets}
\label{sec:exp_common_benchmark}
\pn{
One of the reasons why I abandoned the hyper-parameter tuning experiments and instead reported these experiments was the nasty copying in Matlab.
However, I have the mex implementations still available and should report the timings of these algorithms as well.
}
The datasets chosen 
are small such that computation of the exact GP is still feasible.
Origin and purpose the datasets can be found in \cref{sec:datasets}.
Each dataset has been shuffled and split into two sets, using one for training and the other for testing. 
For each dataset, we optimized the kernel parameters running Carl Rasmussen's \texttt{minimize} function\footnote{This method is part of the GPML toolbox \citep{rasmussen2010gpml}, see \url{http://www.gaussianprocess.org/gpml/code/matlab/doc}.} for 100 optimization-steps,
where initially all kernel hyper-parameters are set to 1.

\pn{If I would do a different split for each seed I would have to train hyper-parameters again.
Doing only one split is fairly representative.
}

\cref{fig:cg_hybrid_error_se} shows how the average relative error develops for the described setup\footnote{Since the Mat\'ern kernel experiments look very similar, these results are in \cref{app:experiments_matern}}.
The number of inducing inputs $M$ was set to $M=\sqrt{NP}$ such that $\O$-notation costs are equivalent to \kmcg{}: Since \kmcg{} uses multiplications with $\vec K$ for observations, the costs per CG-step are $\O(N^2P)$.
The upper x-axis displays the number of conjugate gradients steps, the lower x-axis, the number of inducing inputs.
During early iterations the performance of CG is not as reliable as \kmcg{} and the latter also improves more consistently. 
In comparison to the \baselines{}, \kmcg{} often provides a worse approximation to start with but exhibits a faster convergence rate.

In contrast to plain conjugate gradients, \kmcg{} naturally provides estimates for variance (Eq.~\ref{eq:var}) and evidence (Eq.~\ref{eq:llh}). 
Define the average relative errors $\relerrVar{}$ and $\relerrNLZ{}$ analogously to Equation \eqref{eq:relerr}, respectively.
Figure \ref{fig:cg_hybrid_error_se_llh} and \ref{fig:cg_hybrid_error_se_var} show the average relative error of these estimates in comparison to the \baselines{}.
For all datasets one can observe that the approximation quality of \kmcg{} for the evidence (Eq.~\eqref{eq:llh}) is improving at first and then worsening.
\kmcg{} is better at approximating the quadratic form than the determinant.
Therefore, the approximation often `overshoots'.

The \baselines{} clearly outperform \kmcg{} in these experiments.
A possible explanation is that the \baselines{} provide a better overall-approximation to the kernel matrix:
After $P$ CG-steps, the \kmcg{} kernel is of rank $P$ whereas using $M$ inducing inputs, the \VFE{} kernel is of rank $M$ (so is the \FITC{} kernel, putting the diagonal correction aside).
Since $M=\sqrt{NP}$, the \baselines{} can afford more inducing inputs $M$ than \kmcg{} can afford CG-steps $P$.
Overall, when it comes to real-time, the \baselines{} are preferable over \kmcg{}.
The picture changes when matrix-multiplication is less expensive than $\O(N^2)$ which is investigated in the next section.

%
\begin{figure}\centering
\resultFigureTable{cg_hybrid_comparison}{covSEard}{relative_error}{\legendFull}
\caption{progression of the relative error $\relerr$ as a function of the number of iterations of CG and \kmcg{} for different datasets using the squared-exponential kernel (Eq.~\ref{eq:sq_exp}).
\commonCaption{}
}
\label{fig:cg_hybrid_error_se}
\end{figure}

\begin{figure}\centering
\resultFigureTable{cg_hybrid_comparison}{covSEard}{relative_error_of_var}{\legendWithoutCG}
\caption{progression of the relative error of the variance $\relerrVar$ as a function of the number of iterations of \kmcg{} and \baseline{} for different datasets using the squared-exponential kernel (Eq.~\ref{eq:sq_exp}).
\commonCaption{}
}
\label{fig:cg_hybrid_error_se_var}
\end{figure}

\begin{figure}\centering
\resultFigureTable{cg_hybrid_comparison}{covSEard}{relative_nlZ_error}{\legendWithoutCG}
\caption{
progression of the relative error of the evidence $\relerrNLZ$ as a function of the number of iterations of \baseline{} and \kmcg{} for different datasets using the squared-exponential kernel (Eq.~\ref{eq:sq_exp}).
\commonCaption{}
The small spikes in the plots where \kmcg{} appears to be close to the solution correspond to changes of the estimate from too small to too large.
}
\label{fig:cg_hybrid_error_se_llh}
\end{figure}

\subsection{Grid-structured Datasets}
\label{sec:exp_cg_grid}
\pn{Measuring the time does not make much sense here: DTC can make use of the fast MVM as well. (I should include it for the sake of completeness.)
However, DTC performs far worse. 
}
In the previous section 
the \baselines{} are the preferable estimators over \kmcg{}.  
This changes when matrix-multiplication costs less than $\O(N^2)$.
For example when the kernel is a product kernel (such as squared exponential) and the dataset has grid-structure, the costs for matrix-multiplication are almost linearly in the number of data-points \citep{wilson2014multidimensional}
such that the number of CG-steps \kmcg{} can take, matches the number of \baseline{} inducing inputs.

\subsubsection{Artificial Datasets}
\pn{
It is instructional to look at 10x10 first as the picture generalizes for the larger datasets.
Each algorithm can execute approximately 500 steps/inducing inputs before running out of memory.
Between 100 and 500 steps there is not much happening and we therefore show here the first 100 steps.
We show here the first 100 steps and in Appendix ... the same plots over time.
The picture is the same as for the 10x10 dataset.
\kmcg{} is slower yet better at all times.
}
\pn{
Options to deal with the bad results:
}
\pn{{
\begin{itemize}
\item report as is

MatlabHook crashes too early and FOM becomes too slow

Disadvantage: the reviewers will be probably not accept this.
In a scenario that is designed for CG it has to perform!

\item improve KMCG

How? Could try to ignore search directions from PCG that kill the Cholesky.
However, this is probably too expensive and will not work since PCG probably also just reports always the same directions as TextbookCG does.

\item precondition

Use KMCG as a way to improve a FITC/DTC solution.
It was one of the original motivations anyway.
However, this would be opening a whole can of worms.

\item partial reorthogonalization

Did not work so well.
Needs to perform an eigenvalue decomposition (of a triadiagonal matrix) in each step.

\item modify the problem

Could remove the noise.
The number of Eigenvalues (of the full) kernel matrix should then be very clustered.
The missing entries might still screw this up.

Without noise, the kernel matrix is Kronecker Toeplitz which allows to make it huge!
Could that be the salvation?

\item Things would work maybe with an adaptive Cholesky decomposition that ignores rows that make things numerically unstable.
\end{itemize}
Probable course of action:

increase dataset-size. Explain to the reviewers that original datasets where to small for timing experiments.
Hopefully on the larger datasets, the mean of KMCG converges faster than DTC. The rest does not matter.
}}
The datasets considered in the following are artificial multi-dimensional grids.\footnote{Computing the exact solution is feasible exploiting the Kronecker structure of the kernel matrix which we use to evaluate the quality of the approximation methods. However, we may imagine datapoints missing, s.t.~matrix-vector multiplication is fast but computing the exact solution is not.}
For the training set, along each axis $G$ points are equally spaced in $[-\nicefrac{G}{4}, \nicefrac{G}{4}]$ distorted by Gaussian noise $\N(0, 10^{-3})$.
One-hundred test inputs are uniformly distributed over the $[-\nicefrac{G}{4}, \nicefrac{G}{4}]$ cube.
Targets are drawn from a Gaussian process with squared exponential kernel (length scales and amplitude equal to 1).
The number of inducing inputs had to be capped at 500 due to memory limitations.

Figure \ref{fig:cg_hybrid_grid} shows how the approximation error to mean, variance and likelihood term evolves, zoomed in on the first 100 steps.
For reference, \cref{fig:cg_hybrid_grid} also shows a $10\times10$ dataset to give an idea how each method would evolve when investing more computational power would be feasible.
In the appendix, Figure \ref{fig:cg_hybrid_grid_time} shows the same comparison over time for the whole 500 steps, stopping \kmcg{} when it becomes slower than the baselines.
\pn{
Just describe what you have.
Report here that KMCG reaches the training time limit due to that strange Matlab thing (maybe it's MKL...).
Show the same time plots in the Appendix.
Argue that though KMCG is slower, it's still better.
Also, \baseline{} can not run longer due to memory constraints whereas KMCG can!
}

On these datasets \kmcg{} dominates the \baseline{} methods.
After already one hundred CG-steps \kmcg{} provides a useful approximation to the posterior mean whereas the \baselines{} hardly show any progress.
For the variance, the same computational effort is not enough. 
Though the \baselines{} find better solutions, all methods essentially fail to arrive at a satisfactory solution of a relative error below one.
The issue is that all methods overestimate the posterior variance by two orders of magnitude. 
The picture is similar for the evidence, albeit the approximations are closer to the truth and \kmcg{} performs slightly better on average.
\pn{
As Matlab does not support views but always triggers a copy when accessing subarrays, the full reorthogonalization becomes expensive.
However, even these few steps are sufficient for \kmcg{} to outperform the baselines.
It is puzzling that for the three dimensional dataset, the baseline methods need so much more time than on the other.
We tracked the problem down to the $\O(NM^2)$ matrix-multiplication $\vec \Phi \vec \Phi\Trans$ (cf.~Eqs.~\eqref{eq:mean_2} to \eqref{eq:llh_2}).
It seems, that the kernel matrix for this particular dataset has properties that trigger Matlab to treat the multiplication differently such that it takes longer to compute than for the other datasets.
It appears, that the smaller entries in matrices are, the longer it takes to multiply with them.
However, it all methods are affected the same way and the issue does not concern the validity of the results.
}
\pn{That the approximation to the evidence is better is explained partially due to the good mean approximation.
The $\vec y\Trans \vec K^{-1}\vec y$ term is well approximated.
I guess that \baseline{} gives the better approximation to the determinant.
}

\begin{figure}
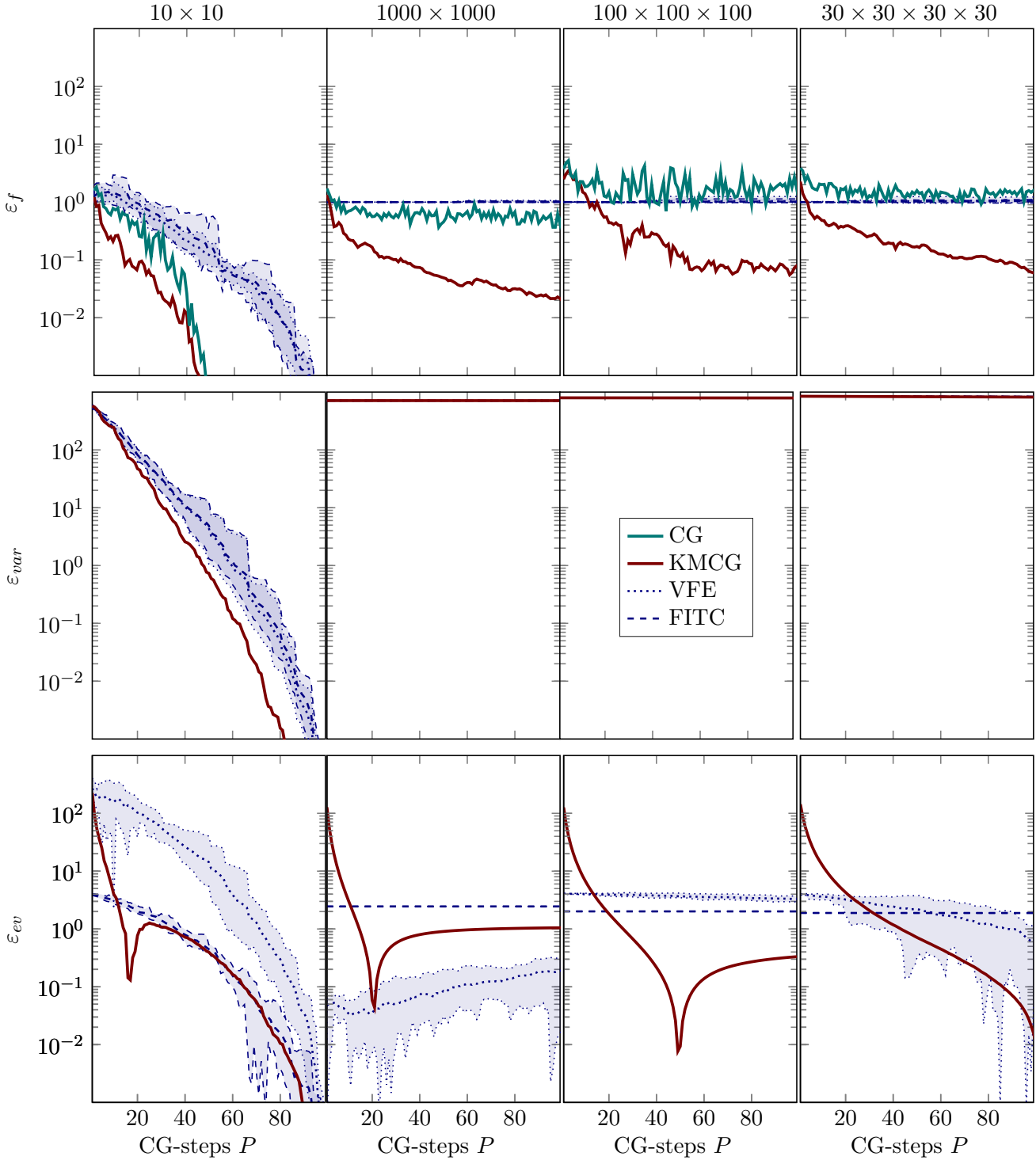

\centering
\gridResultFigureTable{cg_hybrid_grid}
\caption{comparison of \baseline{} and \kmcg{} on grid-structured datasets using the squared exponential kernel (Eq.~\ref{eq:sq_exp}).
The shaded area visualizes minimum and maximum over all \baseline{} runs.}
\label{fig:cg_hybrid_grid}
\end{figure}

\subsubsection{Natural Sound Modeling}
For a real-world example of a grid-structured dataset, we repeat the Natural Sound Modeling experiment considered by \cite{RichTurnerPhDThesis, wilson2015kissGP} and \cite{dong2017scalable}.
Given the intensity of a sound signal recorded over time, the objective is to recover the signal in missing regions.
All inputs (\ie{}including missing) are equidistant and hence the kernel matrix (over all inputs) is Toeplitz for stationary kernel.
The kernel matrix over the given inputs is not Toeplitz, which forbids to use this structure for exact inference.
Nevertheless matrix-vector-multiplication can be performed in linear time.
We use the squared-exponential kernel with the hyper-parameters used by \cite{dong2017scalable}.
Since the exact posterior is infeasible to compute, we report only the standardized mean squared-error:
\begin{align}
\smse\ce \frac{1}{\Var[\vec y]}\sum_{j=1}^{N_*}(\vec y_{*,j}-\hat{f}(\vec x_{*,j}))^2.
\end{align}
To conform with the original experiment, we added for each baseline method a run the inducing inputs where chosen to be on a regular grid.
The result of this run correspond to the minimum.
Figure \ref{fig:sound} confirms the observations from the previous section that \kmcg{} arrives at satisfactory solutions faster than \baseline{}, if matrix-vector multiplication is not an issue.

\begin{figure}
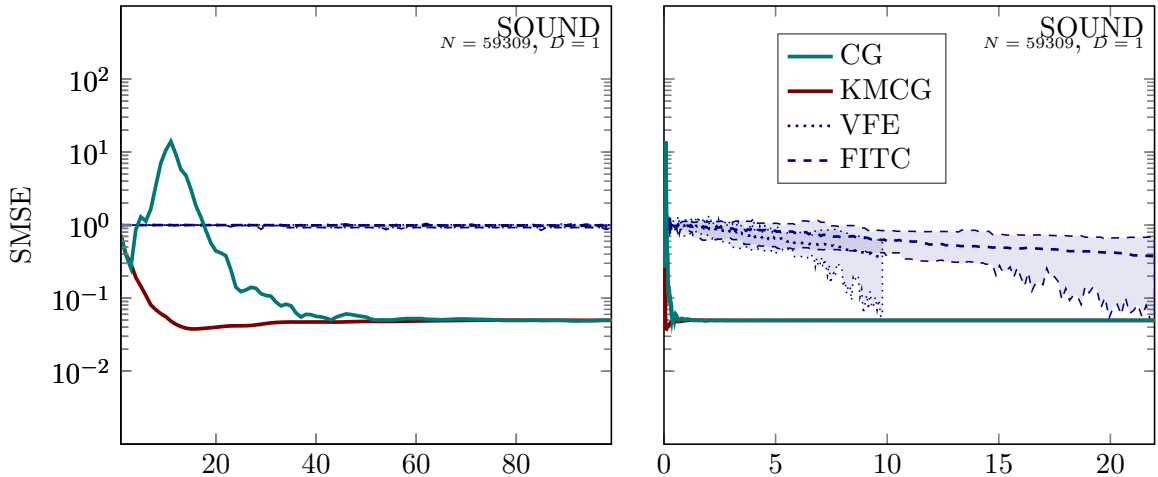

\setlength{\figheight}{0.27\textheight}
\setlength{\figwidth}{0.45\textwidth}
\hbox{ \hspace{-\tableoffset}
\begin{tabular}{cc}
\adjustbox{valign=t}{\input{tikz/cg_hybrid_comparison_SOUND_folds_2_folds_visible_1_covSEard__title_smse__seed_12345.tikz}}&
\adjustbox{valign=t}{\input{tikz/cg_hybrid_comparison_time_SOUND_folds_2_folds_visible_1_covSEard__title_smse__seed_12345.tikz}}
\raisebox{0.3\figheight}{
\makebox[0pt]{
\hspace{-1.25\figwidth}
\legendFull
}}
\end{tabular}
}
\caption{comparison of \kmcg{} and CG on the SOUND dataset.}
\label{fig:sound}
\end{figure}

\section{Conclusion}
\label{sec:conclusions}
We have presented a new approximate inference method for kernel machines that showed how linear solvers can be used in combination with low-rank kernel approximations.
The approach is based on a probabilistic numerics viewpoint: the kernel $k$ is treated as a latent quantity and a linear solver is used for collecting observations of $k$.
By design, the resulting approximate kernel is of low rank and is plugged into the nonparametric least-squares problem.
The approach is not restricted to least-squares problems but applicable in any scenario where the bottleneck is the inversion of a large kernel matrix, as \eg{}GP classification.

Our \emph{kernel machine conjugate gradients} (\kmcg{}), consistently outperforms plain conjugate gradients in numerical experiments. 
This does not change the fact that standard dense kernel least-squares problems are often more efficiently solved by inducing point methods. 
However, as demonstrated in \cref{sec:exp_cg_grid}, in the settings which allow fast multiplication with the kernel matrix, the new algorithm can improve upon the state of the art.

\pn{
framework can be used to 
\begin{itemize}
\item iteratively improve upon approximations
\item combine different approximation
\item yet nothing really worked out
\item iterative improvement of DTC with CG is no better than using more inducing inputs
\item combination of CG and DTC is no better than using more inducing inputs
\end{itemize}
}

\pn{
\begin{itemize}
\item How would I design experiments if I had more time?
\item Why would I need to run CG to convergence? I can stop before it becomes too slow and see if it is a good approximation. (Sometimes it probably is. But I guess I can not say that in general.)
\item What is a good stopping criterion?
\item The reason I do not really want to present the experiments is that they are out-dated. I think the noise should be incorporated. Still I could report the results where I used constant times norm and machine precision.
\end{itemize}
}


\acks{The authors thank Alexandra Gessner, Hans Kersting, Agustinus Kristiadi, Frederik Kunstner, Filip de Roos and Matthias Werner for helpful discussions and proof-reading.}



\appendix
%

%
%
%
%
%

\section{Properties of the Symmetric Kronecker Product}
\label{app:kronecker}
The Kronecker product and its symmetric version have been studied, among others, by \textcite{loan2000kronecker} and \textcite{magnus1980symmetric}.
The definitions used in this work slightly differ from the authors above and instead follow \textcite{Hennig_LinearSolvers2015}.
The Kronecker product for two arbitrary matrices $\vec A\in\Re^{N_1\times N_2}$, $\vec B\in \Re^{N_3\times N_4}$ is defined as 
$$[\vec A\kr\vec B]_{ij,kl}\ce\vec A_{ik}\vec B_{jl}$$
where $i\in \{1, ..., N_1\}$, $j\in \{1, ..., N_3\}$, $k\in\{1, ..., N_2\}$ and $l\in\{1, ..., N_4\}$, and $ij$ is not a product but a double-index. 
The following identities about Kronecker products and the vectorization operator can be found in \textcite{hennig13:_quasi_newton_method}, and are restated here for the convenience of the reader:
\begin{align}
\tag{K1} (\mat A \kr \mat B) \vecm{\mat C} = & \vecm{\mat A \mat C \mat B\Trans} \label{eq:K1}
\\\tag{K2} (\mat A \kr \mat B) (\mat C \kr \mat D) = & (\mat A\mat C) \kr (\mat B\mat D) \label{eq:K2}
\\\tag{K3} (\mat A \kr \mat B)^{-1} = & \mat A^{-1} \kr \mat B^{-1} \label{eq:K3}
\\\tag{K4} (\mat A \kr \mat B)^\top = & \mat A\Trans \kr \mat B\Trans \label{eq:K4}
\\\tag{K5} (\mat A + \mat B) \kr \mat C = & \mat A\kr \mat C + \mat B\kr \mat C \label{eq:K5}
\end{align}
where\footnote{The conditions can be more general but for ease of exposition, we assume all matrices are square and of equal size.} $\mat A, \mat B, \mat C, \mat D\in \Re^{N\times N}$, and $\mat A$ and $\mat B$ are assumed to be invertible.

An appealing property of Kronecker-structured matrices is their interaction with vectorized matrices.
For a square matrix $\mat A =\begin{bmatrix} \vec{a}_1 & \ldots & \vec{a}_N\end{bmatrix}\Trans\in \Re^{N\times N}$,
the \emph{vectorization operator} $\vecm{~} : \Re^{N\times N} \to \Re^{N^2}$ stacks the rows\footnote{Stacking the columns is equivalently possible and common. It is associated with a permutation in the definition of the Kronecker product, but the resulting inferences are equivalent.} 
of $\mat A$ into one vector:
\begin{equation}
\vecm{\mat A} = \begin{bmatrix} \vec a_1 \\ \vdots \\ \vec a_N \end{bmatrix},\quad\text{with}\quad \left[\vecm{\mat A}\right]_{(ij)} = [\mat A]_{ij} 
\end{equation}
and $\tomat{~}$ transforms an $N^2$ vector into an $N\times N$ matrix, s.t.~$\tomat{\vecm{\mat A}}=\mat A$.
A vector product of vectorized matrices corresponds to the trace of their product:
\begin{align}
\tag{V1} \label{eq:vec_trace}
\vecm{\vec A}\Trans \vecm{\vec B}= &\tr{\vec A \vec B\Trans}.
\end{align}
\begin{proof}
	\begin{align}
	\tr{\vec A\vec B\Trans}&=\sum_{i}[\vec A\vec B\Trans]_{ii}
	\\&=\sum_{i, j}\vec A_{ij}\vec B\Trans_{ji}
	\\&=\sum_{i, j}\vec A_{ij}\vec B_{ij}
	\\&=\vecm{\vec A}\Trans \vecm{\vec B}
	\end{align}
\end{proof}

The \emph{symmetric} Kronecker product for two square matrices $\mat A,\mat B\in \Re^{N\times N}$ of equal size is defined as
$$\mat A\sk\mat B\ce \mat \Gamma_N (\mat A\kr \mat B)\mat \Gamma_N$$
where $[\mat \Gamma_N]_{ij,kl}\ce\nicefrac{1}{2}\delta_{ik}\delta_{jl}+\nicefrac{1}{2}\delta_{il}\delta_{jk}$ satisfies $$\mat \Gamma\vecm{\mat C}=\nicefrac{1}{2}\vecm{\mat C}+\nicefrac{1}{2}\vecm{\mat C\Trans}$$ for all square-matrices $\mat C\in\Re^{N\times N}$.
Equivalently, one can write 
$$(\vec A \sk \vec B)_{ij,kl}=\frac{1}{4}\left(\vec A_{ik}\vec B_{jl}+\vec A_{il}\vec B_{jk}+\vec B_{ik}\vec A_{jl}+\vec B_{il}\vec A_{jk}\right).$$
The symmetric Kronecker product inherits some of the desirable properties of the Kronecker product.
Some of the following identities can, again, be found in \textcite{Hennig_LinearSolvers2015}, some are due to \textcite{loan2000kronecker} and \textcite{magnus1980symmetric} and some are novel.
The proof gives exact credit.
\begin{proposition}
	\label{prop:sym_kron}
	Let $\vec V, \vec W\in\Re^{N\times N}$ be square matrices and $\vec A\Trans, \vec B \in \Re^{N\times M}$ be rectangular.
	\begin{align}
	\label{eq:gamma_one}\tag{SK1}
	\vec W\sk \vec W&=\vec\Gamma_N(\vec W\kr\vec W) \\ 
	\label{eq:gamma_kron}\tag{SK2}
	\vec \Gamma_{M} (\vec A\kr \vec A)&=(\vec A\kr\vec A)\vec \Gamma_{N} \\ 
	\label{eq:sk_sym}\tag{SK3}
	\vec V \sk \vec W&=\vec W \sk \vec V
	\\\label{eq:sk_distributive}\tag{SK4}
	(\vec A \kr \vec A)(\vec W \sk \vec W)(\vec B \kr \vec B)&=(\vec A\vec W\vec B)\sk(\vec A\vec W\vec B)
	\\\label{eq:sk_factorization}\tag{SK5}
	\vec W\sk \vec W-\vec V\sk \vec V&=(\vec W+\vec V)\sk(\vec W-\vec V)
	\\\label{eq:sk_inverse}\tag{SK6}
	(\vec W \sk \vec W)^{-1} &= (\vec W^{-1}\sk \vec W^{-1}).
	\end{align}
	The interpretation of \cref{eq:sk_inverse} requires some care: symmetric Kronecker product matrices are rank deficient.
	\cref{eq:sk_inverse} is to be read in the sense that for symmetric $\mat Y\in \Re^{N\times N}$, \ie{}$\mat Y=\mat Y\Trans$, $\mat X\ce \tomat{(\inv{\mat W}\sk\inv{\mat W})\vecm{\mat Y}}$ satisfies $\vecm{\mat Y}=(\mat W\sk \mat W)\vecm{\mat X}$ and $\mat X$ is the unique symmetric solution.
\end{proposition}

\begin{proof}
	The proofs for \cref{eq:gamma_one,eq:gamma_kron} can be found in \cite{magnus1999matrix}[p.~46-50].
	In the notation of \cite{magnus1999matrix} $\vec \Gamma=N_n=D_nD_n^+$ and $K=2\vec \Gamma-2\vec I$.
	\cref{eq:gamma_one} is Theorem 13 (a).
	\cref{eq:gamma_kron} follows from Theorem 9 (a).
	
	To show $(\vec W\sk \vec V)=(\vec V \sk \vec W)$, let $\vec X\in \Re^{N\times N}$ be an arbitrary matrix.
	\begin{align}
	(\vec V \sk \vec W)\vecm{\vec X}&=\vec \Gamma (\vec V\kr\vec W)\vec \Gamma \vecm{\vec X}
	\\&=\frac{1}{2}\vec \Gamma (\vec V\kr \vec W)\vecm{\vec X+\vec X\Trans}
	\\&=\frac{1}{2}\vec \Gamma\vecm{\vec V(\vec X+\vec X\Trans)\vec W\Trans}
	\\&=\frac{1}{4}\vecm{\vec V(\vec X+ \vec X\Trans)\vec W\Trans + \vec W(\vec X+\vec X\Trans)\vec V\Trans}
	\\&=\frac{1}{2}\vec \Gamma \vecm{\vec W(\vec X+\vec X\Trans)\vec V\Trans}
	\\&=\frac{1}{2}\vec \Gamma (\vec W\kr \vec V)\vecm{\vec X+\vec X\Trans}
	\\&=\vec \Gamma (\vec W\kr\vec V)\vec \Gamma \vecm{\vec X}
	\\&=(\vec W \sk \vec V)\vecm{\vec X}
	\end{align}
	
	To show \cref{eq:sk_distributive}, use \eqref{eq:gamma_kron}.
	\begin{align}
	(\vec A \kr \vec A)(\vec W \sk \vec W)(\vec B \kr \vec B)&=(\vec A \kr \vec A)\vec{\Gamma}(\vec W \kr \vec W)\vec \Gamma(\vec B \kr \vec B)
	\\&=\vec \Gamma (\vec A \kr \vec A)(\vec W \kr \vec W)(\vec B \kr \vec B)\vec \Gamma
	\\&=\vec \Gamma (\vec A\vec W\vec B\kr \vec A\vec W\vec B)\vec \Gamma
	\\&=\vec A\vec W\vec B \sk \vec A\vec W\vec B
	\end{align}
	
	The proof of \cref{eq:sk_factorization} uses \eqref{eq:sk_sym}.
	\begin{align}
	(\vec A + \vec B)\sk (\vec A - \vec B)&=\vec \Gamma (\vec A + \vec B)\kr (\vec A - \vec B) \vec \Gamma
	\\&=\vec \Gamma(\vec A\kr \vec A - \vec A\kr \vec B+\vec B\kr\vec A - \vec B\kr \vec B)\vec \Gamma
	\\&=\vec A\sk \vec A - \vec A\sk \vec B+\vec B\sk\vec A - \vec B\sk \vec B
	\\&=\vec A\sk \vec A - \vec B\sk \vec A+\vec B\sk\vec A - \vec B\sk \vec B
	\\&=\vec A\sk \vec A - \vec B\sk \vec B
	\end{align}
	
	It remains to prove \cref{eq:sk_inverse}. 
	Assume $\mat Z$ satisfies $(\mat W\sk\mat W)\vecm{\mat Z}=\vecm{\mat Y}$ and $\mat Z = \mat Z\Trans$.
	Then, 
	\begin{align}
	\vecm{\mat Y}&=(\mat W\sk\mat W)\vecm{\mat Z}
	\\&=(\mat W \kr \mat W)\mat \Gamma_N \vecm{\mat Z}\eqcomment{using \cref{eq:gamma_one} and \cref{eq:gamma_kron}}
	\\&=(\mat W \kr \mat W) \vecm{\mat Z} \eqcomment{since $\mat Z=\mat Z\Trans$}
	\end{align}
	and hence, $\mat Z = \inv{(\mat W \kr \mat W)}\vecm{\mat Y}$.
	Using \cref{eq:K3} and again \cref{eq:gamma_one}, 
	$$\mat Z = {(\inv{\mat W} \sk \inv{\mat W})}\vecm{\mat Y}$$
	which is the definition of $\mat X$.
\end{proof}

\subsection{Sampling from a Gaussian with Symmetric Kronecker Covariance matrix}
\label{app:sampling}
To sample matrices from the \kmcg{} posterior (Eq.~\ref{eq:km}) the following proposition will be useful.
\begin{proposition}
Let $\vec W,\vec W_M\in \Re^{N\times N}$ be symmetric and positive semi-definite matrices s.t.~$\vec W-\vec W_M$ is symmetric positive-semidefinite as well.
Further let $\vecm{\vec Y}\sim \N(\vec 0, \vec W\sk \vec W-\vec W_M\sk \vec W_M)$, denote with $\vec L_+$ the Cholesky of $\vec W+\vec W_M$, with $\vec L_-$ the Cholesky of $\vec W-\vec W_M$ and let $\vecm{\vec X}\sim \N(\vec 0, \vec I_{N^2})$, then $\vec \Gamma (\vec L_1 \kr  \vec L_2)\vecm{\vec X}$ and $\vecm{\vec Y}$ have the same distribution.

Remark: This shows that $\vec Y$ is symmetric due to the $\vec \Gamma$-operator.
\end{proposition}
\begin{proof}
As $\vecm{\vec X}$ is standard normal, $\vec{\Gamma}(\vec L_+\kr\vec L_-)\vecm{\vec X}$ is distributed Gaussian with mean $\vec 0$ and covariance matrix $\vec{\Gamma}(\vec L_+\kr\vec L_-)(\vec{\Gamma}(\vec L_+\kr\vec L_-))\Trans$.
\begin{align}
\vec{\Gamma}(\vec L_+\kr\vec L_-)\left[\vec{\Gamma}(\vec L_+\kr\vec L_-)\right]\Trans&=\vec{\Gamma}(\vec L_+\kr\vec L_-)(\vec L_+\Trans\kr\vec L_-\Trans)\vec\Gamma
\\&=(\vec L_+\vec L_+\Trans)\sk(\vec L_-\vec L_-\Trans)
\\&=(\vec W+\vec W_M)\sk(\vec W-\vec W_M)
\end{align}
According to Equation \eqref{eq:sk_factorization}: $(\vec W+\vec W_M)\sk(\vec W-\vec W_M)=\vec W\sk\vec W-\vec W_M\sk\vec W_M$.
\end{proof}

\section{Inducing Input Methods}
\label{app:dtc}
This section contains the proof of Proposition \ref{thm:dtc}, introduced on page \pageref{thm:dtc}, and, for the readers convenience, restated below, along with the referenced equations.
\DTCprop*
The mentioned equations are 
\begin{align}
\eqrepeat{eq:prior}\text{,}\\
\eqrepeat{eq:w_requirement}\text{,}\\
\eqrepeat{eq:observation_operator}\text{,}\\
\text{and }\eqrepeat{eq:lk_dtc}.
\end{align}
\cref{thm:dtc} follows from the more general \cref{thm:kernel_post}, below.

\newcommand{\wXu}{\vec{W}_{\!\!\!M}}
\newcommand{\kXu}{\Km}
\begin{proposition}
\label{thm:kernel_post}
Consider the prior of \cref{eq:prior} (without the restriction $\kw=k$) and the likelihood defined in \cref{ex:dtc}.
The posterior over $k$ is $p(k\mid \vec Y = \T k)=\N(k; \km, \ksigm)$ with posterior mean
\begin{align}
\km(\vec a, \vec b)&=\kmu(\vec a, \vec b)+\kw(\vec a, \vec X_U)\S(\S\Trans \wXu\S)^{-1}\S\Trans \kXu \vec S(\S\Trans \wXu\S)^{-1}\S\Trans\kw(\vec X_U, \vec b) \label{eq:ind_post_mean}
\\&-\kw(\vec a, \vec X_U)\S(\S\Trans \wXu\S)^{-1}\S\Trans \kmu(\vec X_U, \vec X_U)\vec S(\S\Trans \wXu\S)^{-1}\S\Trans\kw(\vec X_U, \vec b)\notag
\end{align}
and posterior variance
\begin{align}
\ksigm(k(\vec a,& \vec b), k(\vec c, \vec d))=\frac{1}{2}\kw(\vec a, \vec c)\kw(\vec b,\vec d)+\frac{1}{2}\kw(\vec a, \vec d)\kw(\vec b, \vec c)\label{eq:ind_post_var}
\\&-\frac{1}{2}\kw(\vec a, \vec X_U)\S(\S\Trans \wXu \S)^{-1}\S\Trans \kw(\vec X_U, \vec c)\kw(\vec b, \vec X_U)\S(\S\Trans \wXu \S)^{-1}\S\Trans \kw(\vec X_U, \vec d)\notag 
\\&-\frac{1}{2}\kw(\vec a, \vec X_U)\S(\S\Trans \wXu \S)^{-1}\S\Trans \kw(\vec X_U, \vec d)\kw(\vec b, \vec X_U)\S(\S\Trans \wXu \S)^{-1}\S\Trans \kw(\vec X_U, \vec c)\notag 
\end{align}
where $\wXu=\kw(\vec X_U, \vec X_U)$.
\end{proposition}
\begin{proof}
The proof is tedious linear algebra.
If prior and likelihood are Gaussian, so is the posterior with mean and variance: 
\begin{align}
\km(\vec a, \vec b)&=\kmu(\vec a, \vec b) - (\T\ksig(k(\vec a, \vec b), \cdot))\Trans(\T(\T\kw)\Trans)^{-1}\vecm{\vec Y-\vec S\Trans\kmu(\vec X_U, \vec X_U)\vec S},
\\\ksigm(k(\vec a, \vec b), k(\vec c, \vec d))&=\ksig((\vec a, \vec b), (\vec c, \vec d))-(\T\ksig((\vec a, \vec b), (\cdot, \cdot)))\Trans(\T(\T\ksig)\Trans)^{-1}\T\ksig(\vec c, \vec d), (\cdot, \cdot))
\end{align}
With \cref{lem:t} and \cref{eq:sk_inverse}, we can write
\begin{align}
&(\T\ksig(k(\vec a, \vec b), \cdot))\Trans(\T(\T\kw)\Trans)^{-1}
\\&=\frac{1}{2}\vecmtrans{\vec S\Trans \kw(\vec X_U, \vec a)\kw(\vec b, \vec X_U)+ \kw(\vec X_U, \vec b)\kw(\vec a, \vec X_U))\vec S}\left((\vec S\Trans\wXu\vec S)^{-1}\sk (\vec S\Trans \wXu \vec S)^{-1}\right)
\\&=\frac{1}{2}\vecmtrans{(\vec S\Trans \wXu \vec S)^{-1}\vec S\Trans \kw(\vec X_U, \vec a)\kw(\vec b, \vec X_U)+ \kw(\vec X_U, \vec b)\kw(\vec a, \vec X_U))\vec S(\vec S\Trans \wXu \vec S)^{-1}}
\end{align}
and, using \cref{eq:vec_trace}, obtain for \cref{eq:ind_post_mean}:
\begin{align}
\km(\vec a, \vec b)
&=\kmu(\vec a, \vec b)
\\&+\frac{1}{2}\tr{(\vec S\Trans \wXu \vec S)^{-1}\vec S\Trans \kw(\vec X_U, \vec a)\kw(\vec b, \vec X_U)\vec S(\vec S\Trans \wXu \vec S)^{-1}(\Y-\kmu(\vec X_U, \vec X_U))} \notag
\\&+\frac{1}{2}\tr{(\vec S\Trans \wXu \vec S)^{-1}\vec S\Trans \kw(\vec X_U, \vec b)\kw(\vec a, \vec X_U)\vec S(\vec S\Trans \wXu \vec S)^{-1}(\Y-\kmu(\vec X_U, \vec X_U))} \notag
\\&=\kmu(\vec a, \vec b)+\kw(\vec a, \vec X_U)\S(\S\Trans \wXu\S)^{-1}\S\Trans\kXu\vec S(\S\Trans \wXu\S)^{-1}\S\Trans\kw(\vec X_U, \vec b)
\\&-\kw(\vec a, \vec X_U)\S(\S\Trans \wXu\S)^{-1}\S\Trans\kmu(\vec X_U, \vec X_U)\vec S(\S\Trans \wXu\S)^{-1}\S\Trans\kw(\vec X_U, \vec b)\notag
\end{align}
The derivation for Eq.~\eqref{eq:ind_post_var} follows analogously.
\end{proof}
\begin{lemma}
\label{lem:t}
Let $\T$ be defined by \cref{eq:lk_dtc}.
\begin{align}
\T \kw(k(\vec a, \vec b), \cdot)&=\frac{1}{2}\vecm{\vec S\Trans \left(\kw(\vec X_U, \vec a)\kw(\vec b, \vec X_U)+\kw(\vec X_U, \vec b)\kw(\vec a, \vec X_U)\right)\vec S} \label{eq:Tw}
\\\T (\T \kw(\cdot, \cdot))\Trans&=(\vec S\Trans \wXu\vec S)\sk (\vec S\Trans \wXu \vec S) \label{eq:TTw}
\end{align}
\end{lemma}
\begin{proof}
Denote with $\tomat{~}$ the complement of the vectorization operator, \ie{}$\tomat{\vecm{\vec A}}=\vec A$.
Define the matrix $\vec S\in\Re^{N\times M}$ as $\vec S_{ij}=s_{ij}$ and denote with $\vec S_l$ the $l$-th column of $\vec S$. 
Also recall that by Eq. \eqref{eq:w_requirement} $\ksig(k(\vec a, \vec b), k(\vec x, \vec z))=\frac{1}{2}(\kw(\vec a, \vec x)\kw(\vec b, \vec z)+\kw(\vec a, \vec z)\kw(\vec b, \vec x))$.
\begin{align}
&[\tomat{\T[\ksig(k(\vec a, \vec b), k(\cdot, \cdot))]}]_{ij}
\\&=\iint\! \ksig (k(\vec a, \vec b), k(\vec x, \vec z)) \left(\sum_{l=1}^M s_{il}\delta(\vec x - \vec u_l)\right)\left(\sum_{l=1}^M s_{jl}\delta(\vec z - \vec u_l)\right) \dx{\vec x}\dx{\vec z}
\\&=\iint\! \frac{1}{2}(\kw(\vec a, \vec x)\kw(\vec b, \vec z)+\kw(\vec a, \vec z)\kw(\vec b, \vec x)) \left(\sum_{l=1}^M s_{il}\delta(\vec x - \vec u_l)\right)\left(\sum_{l=1}^M s_{jl}\delta(\vec z - \vec u_l)\right) \dx{\vec x}\dx{\vec z}
\\&={\frac{1}{2}\sum_{m=1}^M\sum_{l=1}^M \vec S_{im}\vec S_{jl}(\kw(\vec a, \vec u_m)\kw(\vec b, \vec u_l)+\kw(\vec a, \vec u_l)\kw(\vec b, \vec u_m))}
\\&=\frac{1}{2}\left[\vec S\Trans \kw(\vec X_U, \vec a)\kw(\vec b, \vec X_U)\vec S+\vec S\Trans \kw(\vec X_U, \vec b)\kw(\vec a, \vec X_U)\vec S\right]_{ij}
\\&=\frac{1}{2}\left[\vec S\Trans (\kw(\vec X_U, \vec a)\kw(\vec b, \vec X_U)+\kw(\vec X_U, \vec b)\kw(\vec a, \vec X_U))\vec S\right]_{ij} 
\\& \text{which shows Eq.~}\eqref{eq:Tw} \notag
\\&=\left[\tomat{(\vec S\Trans \kr \vec S\Trans)\vec \Gamma \vecm{\kw(\vec X_U, \vec a)\kw(\vec b, \vec X_U)}}\right]_{ij}
\\&=\left[\tomat{(\vec S\Trans \kr \vec S\Trans)\vec \Gamma (\kw(\vec X_U, \vec a)\kr \kw(\vec X_U, \vec b))}\right]_{ij}
\end{align}
Repeating above derivations shows the second statement, \cref{eq:TTw}:
\begin{align}
\T (\T \ksig)\Trans&=(\vec S\Trans \kr \vec S\Trans)\vec \Gamma (\kw(\vec X_U, \vec X_U)\kr\kw(\vec X_U, \vec X_U))\vec \Gamma (\vec S\kr\vec S)
\\&=(\vec S \kr \vec S)\Trans(\kw(\vec X_U, \vec X_U)\sk\kw(\vec X_U, \vec X_U))(\vec S\kr\vec S)
\\&=(\S\Trans \wXu \S)\sk (\S\Trans \wXu \S) & \text{Equation \eqref{eq:sk_distributive}}
\end{align}
\end{proof}

\begin{proposition}
\label{thm:spd_mean}
If $\kmu=0$, $\S$ has rank $M$, and $k$ and $\kw$ are positive definite kernel functions then the posterior mean in \cref{eq:ind_post_mean} is symmetric and positive semi-definite.
\end{proposition}
\begin{proof}
%
%
With $\kmu=0$ the expression for $\km$ from \cref{thm:kernel_post} simplifies to 
\begin{align}
\km(\vec x, \vec z)=&\kw(\vec x, \vec X_U)\vec S(\S\Trans \wXu\S)^{-1}\S\Trans \kXu \S (\S\Trans \wXu\S)^{-1}\S\Trans \kw(\vec X_U, \vec z).
\end{align}
The function $\km$ is symmetric since $k$ is symmetric.
The bivariate function $\km$ is said to be positive (semi-)definite iff for all $n\in \mathbb{N}$ and for all $\vec Z \in \mathbb{X}$, $\km(\vec Z, \vec Z)$ is a positive (semi-)definite matrix.
Since $k(\vec X_U, \vec X_U)$ is a symmetric and positive definite (s.p.d.) matrix, so is $\vec S\Trans k(\vec X_U, \vec X_U)\vec S$ for arbitrary $\vec S$.
The same argument holds for $\S\Trans \wXu \S$.
Since $\S$ is rank $M$, $(\S\Trans \wXu \S)^{-1}$ exists and the inverse of an s.p.d.~matrix is s.p.d.~as well.
Therefore $\vec S(\S\Trans \wXu\S)^{-1}\S\Trans \kXu \S (\S\Trans \wXu\S)^{-1}\S$ is symmetric and positive semi-definite.
This completes the proof.
\end{proof}

\section{Projected Bayes Regressor}
\label{app:pbr}
This section contains the proof of Proposition \ref{thm:pbr} restated below.
\PBRprop*
\begin{proof}
Given \cref{lem:t_eigenfunction} below, all that remains is to substitute $\km$ in \cref{eq:mean} which evaluates to 
\begin{align}
\label{eq:trecate}
\phi(\vec x_*)\Trans(\vec \Phi\vec \Phi\Trans + \noise \vec\Lambda^{-1})\vec \Phi\Trans\vec y.
\end{align}
Comparing $b(\vec x)$ in Definition 1 in \cite{Trecate.1999} and \cref{eq:trecate} one observes that both are equivalent.
\end{proof}


\begin{lemma}
\label{lem:t_eigenfunction}
Let $\phi_i$ $i=1, ..., P$ be orthogonal Eigenfunctions of $k$ with respect to a density $\nu$ on $\mathbb{X}$, \ie{}
\begin{align}
\int k(\vec x, \vec z)\phi_i(\vec z)\nu(\vec z)\dx{\vec z}&=\lambda_i\phi_i(\vec x)
\\\int \phi_i(\vec z)\phi_j(\vec z)\nu(\vec z)\dx{\vec z}&=\delta_{ij}
\end{align}
where $\lambda_i\in \Re$ and $\delta_{ij}$ is the Kronecker delta.
Under the prior of \cref{eq:prior} with $\kmu\ce 0$ and $w\ce k$ and the likelihood defined in \cref{ex:rks} with $p_i(\vec x)=\phi_i(\vec x)\nu(\vec x)$, the approximate kernel (\cref{eq:km}) evaluates to
\begin{align}
\km(\vec x, \vec z)&=\sum_{j=1}^M\lambda_j \phi_j(\vec x)\phi_j(\vec z)=\vec \phi(\vec x)\Trans \vec \Lambda \vec \phi(\vec z)
\end{align}
where $[\vec \phi(\vec x)]_i=\phi_i(\vec x)$ and $\vec \Lambda_{ij}\ce \delta_{ij}\lambda{i}$.
\end{lemma}
\begin{proof}
With a zero prior-mean, the posterior over $k$ (\cref{eq:km}) simplifies to
\begin{align}
\km(\vec x, \vec z)&=(\T \ksig(k(\vec x, \vec z), \cdot)\Trans(\T(\T\ksig)\Trans)^{-1}\T k.
\end{align}
Differing from the proof of Proposition \ref{thm:kernel_post} the observation operator $\T$ (Eq.~\ref{eq:observation_operator}) is of the form:
\begin{align}
[\Tm{k}]_{ij}&=\iint\! k(\vec x, \vec z) \phi_i(\vec x) \phi_j(\vec z) \ \nu(\mathrm{d} \vec x) \nu(\mathrm{d} \vec z)
\\&=\lambda_i\int \phi_i(\vec z) \phi_j(\vec z)\ \nu(\mathrm{d}\vec z)
\\&=\lambda_i \delta_{ij}
\\&=\vec\Lambda_{ij}.
\end{align}
The observation operator $\T$ applied to the covariance function $\kw$ evaluates to:
\begin{align}
[\Tm{\ksig (k(\vec a, \vec b), k(\cdot, \cdot\cdot ))}]_{ij}&=\left[\Tm{\left[\frac{1}{2}k(\vec a, \cdot)k(\vec b, \cdot\cdot)+\frac{1}{2}k(\vec a, \cdot\cdot)k(\vec b, \cdot)\right]}\right]_{ij}
\\&=\frac{1}{2}\iint\! k(\vec a, \vec x)k(\vec b, \vec z)\phi_i(\vec x)\phi_j(\vec z)\nu(\mathrm{d} \vec x) \nu(\mathrm{d} \vec z)
\\&+\frac{1}{2}\iint\! k(\vec a, \vec x)k(\vec b, \vec z)\phi_j(\vec x)\phi_i(\vec z)\nu(\mathrm{d} \vec x) \nu(\mathrm{d} \vec z) \notag
\\\label{eq:app_rks_h1}
&=\frac{1}{2}\lambda_i\lambda_j(\phi_i(\vec a)\phi_j(\vec b)+\phi_i(\vec b)\phi_j(\vec a))
\\&=\frac{1}{2}[\vec\Lambda(\vec\phi(\vec a)\vec \phi(\vec b)\Trans+\vec\phi(\vec b)\vec\phi(\vec a)\Trans)\vec\Lambda]_{ij}.
\end{align}
Applying $\T$ again, leads to
\begin{align}
[\T(\T\ksig)\Trans]_{ij,gh}&=\iint \! [\T\ksig(k(\vec x, \vec z), k(\cdot, \cdot\cdot)]\Trans_{gh}\phi_i(\vec x)\phi_j(\vec z)\ \nu(\mathrm{d}\vec x)\nu(\mathrm{d}\vec z)
\\&\text{using Equation \eqref{eq:app_rks_h1}}
\\&=\frac{1}{2}\lambda_g\lambda_h\iint \! (\phi_g(\vec x)\phi_h(\vec z)+\phi_g(\vec z)\phi_h(\vec x))\phi_i(\vec x)\phi_j(\vec z) \ \nu(\mathrm{d}\vec x)\nu(\mathrm{d}(\vec z)
\\&=\frac{1}{2}\lambda_g\lambda_h\int\!(\delta_{ig}\phi_h(\vec z)+\delta_{ih}\phi_g(\vec z))\phi_j(\vec z)\ \nu(\mathrm{d}(\vec z)
\\&=\frac{1}{2}\lambda_g\lambda_h(\delta_{ig}\delta_{jh}+\delta_{ih}\delta_{jg})
\\&=[{\vec\Lambda}\sk{\vec\Lambda}]_{ij,gh}
\end{align}
where the last equation follows from the definition of the symmetric Kronecker product.
This implies for the posterior mean over the kernel:
\begin{align}
\km(\vec a, \vec b)&=(\T \ksig(k(\vec a, \vec b), \cdot)\Trans(\T(\T\ksig)\Trans)^{-1}\T k
\\&=\frac{1}{2}\vecm{{\vec\Lambda}(\vec\phi(\vec a)\vec \phi(\vec b)\Trans+\vec\phi(\vec b)\vec\phi(\vec a)\Trans){\vec\Lambda}}\Trans({\vec\Lambda}\sk{\vec\Lambda})^{-1}
\vecm{{\vec\Lambda}}
\\&=\frac{1}{2}\vecm{(\vec\phi(\vec a)\vec \phi(\vec b)\Trans+\vec\phi(\vec b)\vec\phi(\vec a)\Trans)}\Trans \vecm{\vec\Lambda}
\\&\text{applying Equation \eqref{eq:vec_trace}:}
\\&=\frac{1}{2}\tr{{\vec\Lambda}(\vec\phi(\vec a)\vec \phi(\vec b)\Trans+\vec\phi(\vec b)\vec\phi(\vec a)\Trans)}
\\&=\vec\phi(\vec a)\Trans{\vec\Lambda}\vec\phi(\vec b).
\end{align}
\end{proof}

\section{Benchmark Datasets}
\label{sec:datasets}
\cref{tbl:dataset_credits} describes purpose and origin of standard benchmark datasets used for Gaussian process regression.
More information on PRECIPITATION can be found at \url{http://www.image.ucar.edu/Data/US.monthly.met/}.
It appears that the datasets AILERONS, ELEVATORS and POLETELECOMM are no longer available under the link \url{https://www.dcc.fc.up.pt/~ltorgo/Regression/DataSets.html}.
However, all files are part of this submission.
\begin{table}
	\centering
	\begin{tabular}{lp{.15\textwidth}p{.25\textwidth}p{.35\textwidth}}
		\hline
		name & reference & url & description \\
		\hline
		ABALONE & \textcite{Nash1994abalone, Waugh1995abalone, Dua2019uci} & \url{https://archive.ics.uci.edu/ml/datasets/Abalone} &  age prediction of abalone from physical measurements \\
		AILERONS & \textcite{Camachol1998AileronsElevators}  & n/a & control action prediction on the ailerons of an F16 aircraft \\
		ELEVATORS & \textcite{Camachol1998AileronsElevators} & n/a & control action prediction on the elevators of an F16 aircraft\\
		MPG & \textcite{Quinlan1993mpg, Dua2019uci} & \url{https://archive.ics.uci.edu/ml/datasets/auto+mpg} & fuel consumption prediction in miles per gallon for different attributes of cars\\
		POLETELECOMM & \textcite{Weiss1995Poletelecomm} & n/a & commercial telecommunication application--no further information \\
		PRECIPITATION & \textcite{Vanhatalo2008precipitation} & \url{github.com/gpstuff-dev/gpstuff/blob/master/gp/demo_regression_ppcs.m} & US annual precipitation prediction for the year 1995 \\
		PUMADYN & \textcite{Snelson2006pseudo} &  \url{ftp://ftp.cs.toronto.edu/pub/neuron/delve/data/tarfiles/pumadyn-family/pumadyn-32nm.tar.gz} & acceleration prediction one of the arm links given angles,
		positions and velocities of other links of a \pname{Puma560} robot\\
		SOUND & \textcite{RichTurnerPhDThesis, wilson2015kissGP} & \url{https://github.com/kd383/GPML_SLD/blob/master/demo/sound/audio_data.mat} & sound intensity prediction of a signal recorded over time for missing regions\\
		TOY & introduced in this work & n/a & targets are a draw from a zero-mean Gaussian process with squared exponential kernel (\cref{eq:sq_exp} with $\vec \Lambda=0.25$ and $\theta_f=2$), inputs stem in equal parts from a Gaussian mixture ($\N(0, 1)+\N(1, 0.1)+\N(-0.5, 0.05)$) and the uniform distribution over $[0,1]$\\
		\hline
	\end{tabular}
	\caption{descriptions and sources for all datasets considered in this work.}
	\label{tbl:dataset_credits}
\end{table}

\section{Additional Experiments and Results}
This section consists of figures showing the results of \cref{sec:exp_common_benchmark} for the Mat\'ern kernel, real-time experiments and experiments with the textbook version of conjugate gradients.
\subsection{Real-time Results}
\label{app:experiments_realtime}
This section shows the same results as in \cref{sec:exp_common_benchmark} but over training-time instead of CG-steps.
All figures have been trimmed to the slowest \baseline{} method.
\cref{fig:cg_hybrid_error_se_time} shows how the relative error $\relerr$ develops over time for the squared exponential kernel and \cref{fig:cg_hybrid_grid_time} shows the same for experiments over grid-structured datasets from \cref{sec:exp_cg_grid}. 
For the x-axis values we took the median of all measurements and fitted a quadratic function to these.
\begin{figure}
\resultFigureTable[\\time in $s$&time in $s$&time in $s$]{cg_hybrid_comparison_time}{covSEard}{relative_error}{\legendFull}
\caption{progression of the relative error $\relerr$ over training time for different datasets using the squared-exponential kernel (Eq.~\ref{eq:sq_exp}).
\commonCaption{}
}
\label{fig:cg_hybrid_error_se_time}
\end{figure}

\begin{figure}
\thispagestyle{empty} 
\centering
\gridResultFigureTable[time in $s$]{cg_hybrid_grid_time}
\caption{progression of the relative error $\relerr$ over training time for different datasets using the squared-exponential kernel (Eq.~\ref{eq:sq_exp}).
\commonCaption{}
It may seem surprising that the runs on the $100\times100\times100$ dataset take more than twice as long.
By chance, the dataset contains more extreme values in the kernel matrix, \ie{}smaller than $1e^{-50}$.
Multiplication with these elements takes more time.
}
\label{fig:cg_hybrid_grid_time}
\end{figure}

\subsection{Mat\'ern Kernel Results}
\label{app:experiments_matern}
The figures in this section show the results for the Mat\'ern $\nicefrac{5}{2}$ kernel (\cref{eq:matern}) for the experiment setup described \cref{sec:exp_common_benchmark}.
\cref{fig:cg_hybrid_error_matern} shows the results for the relative error $\relerr$, \cref{fig:cg_hybrid_error_matern_var} and \cref{fig:cg_hybrid_error_matern_llh} the results for $\relerrVar$ and $\relerrNLZ$, respectively.
\cref{fig:cg_hybrid_error_matern_time} displays the relative error over time. 

\begin{figure}\centering
\resultFigureTable{cg_hybrid_comparison}{covMaternard_5}{relative_error}{\legendFull}
\caption{
progression of the relative error $\relerr$ as a function of the number of iterations of \baseline{} and \kmcg{} for different datasets using the Mat\'ern kernel (Eq.~\ref{eq:matern}).
\commonCaption{}
}
\label{fig:cg_hybrid_error_matern}
\end{figure}

\begin{figure}\centering
\resultFigureTable{cg_hybrid_comparison}{covMaternard_5}{relative_error_of_var}{\legendWithoutCG}
\caption{progression of the relative error of the variance $\relerrVar$ as a function of the number of iterations of \baseline{} and \kmcg{} for different datasets using the Mat\'ern kernel (Eq.~\ref{eq:matern}).
\commonCaption{}
}
\label{fig:cg_hybrid_error_matern_var}
\end{figure}

\begin{figure}\centering
\resultFigureTable{cg_hybrid_comparison}{covMaternard_5}{relative_nlZ_error}{\legendWithoutCG}
\caption{
progression of the relative error of the evidence $\relerrNLZ$ as a function of the number of iterations of \baseline{} and \kmcg{} for different datasets using the Mat\'ern kernel (Eq.~\ref{eq:matern}).
\commonCaption{}
}
\label{fig:cg_hybrid_error_matern_llh}
\end{figure}

\begin{figure}
\resultFigureTable[\\time in $s$&time in $s$&time in $s$]{cg_hybrid_comparison_time}{covMaternard_5}{relative_error}{\legendFull}
\caption{progression of the relative error $\relerr$ over training time for different datasets using the Mat\'ern kernel (Eq.~\ref{eq:matern}).
\commonCaption{}
}
\label{fig:cg_hybrid_error_matern_time}
\end{figure}

\subsection{Instability of Textbook Conjugate Gradients}
\label{app:experiments_instability}
The experiments in \cref{sec:exp_cg}, where carried out by running conjugate gradients with full reorthogonalization.
\cref{fig:cg_unstable} demonstrates that for the problems under consideration, the textbook version of conjugate gradients is not sufficiently numerically stable.
With vanilla conjugate gradients in the background, \kmcg{} can run only for a couple of steps before the necessary Cholesky decompositions fail to be computable.
Furthermore, conjugate gradients itself does not converge.
\begin{figure}
\renewcommand{\cgImpl}[1]{(#1)}
\renewcommand{\FOM}{FOM}
\resultFigureTable{cg_unstable}{covSEard}{relative_error}{\legendAllCG}
\caption{progression of the relative error $\relerr$ over 100 CG-steps for different datasets using the squared exponential kernel (\cref{eq:sq_exp}), comparing CG and FOM.
\commonCaption{}
}
\label{fig:cg_unstable}
\end{figure}

\bibliography{bibfile}

\begin{thebibliography}{53}
\providecommand{\natexlab}[1]{#1}
\providecommand{\url}[1]{\texttt{#1}}
\expandafter\ifx\csname urlstyle\endcsname\relax
  \providecommand{\doi}[1]{doi: #1}\else
  \providecommand{\doi}{doi: \begingroup \urlstyle{rm}\Url}\fi

\bibitem[Benoit(1924)]{Cholesky}
Benoit.
\newblock {{N}ote s{\^u}re une m{\'e}thode de r{\'e}solution des {\'e}quations
  normales provenant de l'application de la m{\'e}thode des moindres carr{\'e}s
  a un syst{\`e}me d'{\'e}quations lin{\'e}aires en nombre inf{\'e}rieure a
  celui des inconnues. {A}pplication de la m{\'e}thode a la r{\'e}solution d'un
  syst{\`e}me d{\'e}fini d'{\'e}quations lin{\'e}aires. (Proc{\'e}d{\'e} du
  {C}ommandant {C}holesky)}.
\newblock \emph{Bulletin Geodesique}, 7\penalty0 (1):\penalty0 67--77, 1924.

\bibitem[Bishop(2006)]{bishop2006pattern}
Christopher~M. Bishop.
\newblock \emph{{Pattern Recognition and Machine Learning}}.
\newblock Springer, 2006.

\bibitem[Camachol(1998)]{Camachol1998AileronsElevators}
Rui Camachol.
\newblock Inducing models of human control skills.
\newblock In Claire N{\'e}dellec and C{\'e}line Rouveirol, editors,
  \emph{Machine Learning: ECML-98}, pages 107--118, 1998.

\bibitem[Chalupka et~al.(2013)Chalupka, {Williams,~C.~K.~I.}, and
  Murray]{chalupka2013comparison}
Krzysztof Chalupka, {Williams,~C.~K.~I.}, and Iain Murray.
\newblock A framework for evaluating approximation methods for {G}aussian
  process regression.
\newblock \emph{Journal of Machine Learning Research}, 14\penalty0
  (1):\penalty0 333--350, 2013.

\bibitem[Csat{\'o} and Opper(2002)]{Csato.2002}
Lehel Csat{\'o} and Manfred Opper.
\newblock Sparse on-line {G}aussian processes.
\newblock \emph{Neural Computation}, 14\penalty0 (3):\penalty0 641--668, 2002.

\bibitem[Davies(2015)]{davies2015effectiveGPimplementation}
Alexander Davies.
\newblock \emph{{Effective implementation of Gaussian process regression for
  machine learning}}.
\newblock PhD thesis, University of Cambridge, 2015.

\bibitem[Dong et~al.(2017)Dong, Eriksson, Nickisch, Bindel, and
  Wilson]{dong2017scalable}
Kun Dong, David Eriksson, Hannes Nickisch, David Bindel, and Andrew~G Wilson.
\newblock Scalable log determinants for gaussian process kernel learning.
\newblock In \emph{Advances in Neural Information Processing Systems}, pages
  6330--6340, 2017.

\bibitem[Dua and Graff(2019)]{Dua2019uci}
Dheeru Dua and Casey Graff.
\newblock {UCI} machine learning repository, 2019.
\newblock URL \url{http://archive.ics.uci.edu/ml}.

\bibitem[{F}ilippone and {E}ngler(2015)]{filippone2015ulisse}
{M}aurizio {F}ilippone and {R}aphael {E}ngler.
\newblock {E}nabling scalable stochastic gradient-based inference for
  {G}aussian processes by employing the {U}nbiased {LI}near {S}ystem {S}olv{E}r
  ({ULISSE}).
\newblock In \emph{Proceedings of the 32nd International Conference on Machine
  Learning}, pages 1015--1024, Lille, France, 2015.

\bibitem[Golub and {Van Loan}(1996)]{golub1996matrix}
Gene~H. Golub and Charles~F. {Van Loan}.
\newblock \emph{{Matrix computations}}.
\newblock Johns Hopkins Univ Pr, 1996.

\bibitem[Hennig and Kiefel(2013)]{hennig13:_quasi_newton_method}
P.~Hennig and M.~Kiefel.
\newblock {Quasi-{N}ewton Methods -- a new direction}.
\newblock \emph{Journal of Machine Learning Research}, 14:\penalty0 834--865,
  March 2013.

\bibitem[Hennig(2015)]{Hennig_LinearSolvers2015}
Philipp Hennig.
\newblock Probabilistic interpretation of linear solvers.
\newblock \emph{SIAM Journal on Optimization}, 25\penalty0 (1):\penalty0
  210--233, 2015.

\bibitem[Hennig et~al.(2015)Hennig, Osborne, and Girolami]{HenOsbGirRSPA2015}
Philipp Hennig, Michael~A. Osborne, and Mark Girolami.
\newblock Probabilistic numerics and uncertainty in computations.
\newblock \emph{Proceedings of the Royal Society of London A: Mathematical,
  Physical and Engineering Sciences}, 2015.

\bibitem[Hensman et~al.(2018)Hensman, Durrande, and
  Solin]{hensman18variationalfourier}
James Hensman, Nicolas Durrande, and Arno Solin.
\newblock Variational fourier features for gaussian processes.
\newblock \emph{Journal of Machine Learning Research}, 18\penalty0
  (151):\penalty0 1--52, 2018.

\bibitem[Hestenes and Stiefel(1952)]{hestenes1952methods}
Magnus~R. Hestenes and Eduard Stiefel.
\newblock {Methods of conjugate gradients for solving linear systems}.
\newblock \emph{Journal of Research of the National Bureau of Standards},
  49\penalty0 (6):\penalty0 409--436, 1952.

\bibitem[Hoerl and Kennard(1970)]{hoerl1970ridge}
Arthur~E. Hoerl and Robert~W. Kennard.
\newblock Ridge regression: Biased estimation for nonorthogonal problems.
\newblock \emph{Technometrics}, 12\penalty0 (1):\penalty0 55--67, 1970.

\bibitem[Kimeldorf and Wahba(1970)]{kimeldorf1970correspondence}
George~S. Kimeldorf and Grace Wahba.
\newblock A correspondence between {B}ayesian estimation on stochastic
  processes and smoothing by splines.
\newblock \emph{The Annals of Mathematical Statistics}, pages 495--502, 1970.

\bibitem[L{\'a}zaro-Gredilla et~al.(2010)L{\'a}zaro-Gredilla,
  Qui{\~n}onero-Candela, Rasmussen, and Figueiras-Vidal]{LazaroGredilla.2010}
Miguel L{\'a}zaro-Gredilla, Joaquin Qui{\~n}onero-Candela, Carl~E. Rasmussen,
  and An{\'i}bal~R. Figueiras-Vidal.
\newblock Sparse spectrum {G}aussian process regression.
\newblock \emph{Journal of Machine Learning Research}, 11:\penalty0 1865--1881,
  2010.

\bibitem[Le et~al.(2013)Le, Sarlos, and Smola]{Le.2013}
Quoc Le, Tamas Sarlos, and Alexander Smola.
\newblock Fastfood - computing {H}ilbert space expansions in loglinear time.
\newblock In \emph{Proceedings of the 28th International Conference on Machine
  Learning}, pages 244--252, 2013.

\bibitem[Loan(2000)]{loan2000kronecker}
Charles F.~Van Loan.
\newblock The ubiquitous kronecker product.
\newblock \emph{Journal of Computational and Applied Mathematics}, 123\penalty0
  (1):\penalty0 85 -- 100, 2000.
\newblock Numerical Analysis 2000. Vol. III: Linear Algebra.

\bibitem[Magnus and Neudecker(1980)]{magnus1980symmetric}
Jan~R. Magnus and H.~Neudecker.
\newblock The elimination matrix: Some lemmas and applications.
\newblock \emph{SIAM Journal on Algebraic Discrete Methods}, 1\penalty0
  (4):\penalty0 422--449, 1980.
\newblock \doi{10.1137/0601049}.

\bibitem[Magnus and Neudecker(1999)]{magnus1999matrix}
Jan~R. Magnus and Heinz Neudecker.
\newblock \emph{Matrix Differential Calculus with Applications in Statistics
  and Econometrics}.
\newblock John Wiley, second edition, 1999.

\bibitem[Matheron(1973)]{matheron1973intrinsic}
Georges Matheron.
\newblock The intrinsic random functions and their applications.
\newblock \emph{Advances in applied probability}, pages 439--468, 1973.

\bibitem[Nash et~al.(1994)Nash, Sellers, Talbot, Cawthorn, and
  Ford]{Nash1994abalone}
Warwick~J. Nash, Tracy~L. Sellers, Simon~R. Talbot, Andrew~J. Cawthorn, and
  Wes~B. Ford.
\newblock The population biology of abalone (haliotis species) in tasmania. 1,
  blacklip abalone (h.~rubra) from the north coast and the islands of bass
  strait.
\newblock Technical Report~48, Sea Fisheries Division, Marine Research
  Laboratories - Taroona, Department of Primary Industry and Fisheries,
  Tasmania, 1994.
\newblock URL \url{https://trove.nla.gov.au/work/11326142}.

\bibitem[Nocedal and Wright(1999)]{nocedal1999numerical}
Jorge Nocedal and Stephen~J. Wright.
\newblock \emph{{Numerical Optimization}}.
\newblock Springer Verlag, 1999.

\bibitem[Pleiss et~al.(2018)Pleiss, Gardner, Weinberger, and
  Wilson]{Pleiss2018LOVE}
Geoff Pleiss, Jacob Gardner, Kilian Weinberger, and Andrew~Gordon Wilson.
\newblock Constant-time predictive distributions for {G}aussian processes.
\newblock In \emph{Proceedings of the 35th International Conference on Machine
  Learning}, pages 4114--4123, 2018.

\bibitem[Quinlan(1993)]{Quinlan1993mpg}
J.~Ross Quinlan.
\newblock Combining instance-based and model-based learning.
\newblock In \emph{Proceedings of the Tenth International Conference on
  International Conference on Machine Learning}, ICML'93, pages 236--243, 1993.

\bibitem[Qui{\~n}onero-Candela and Rasmussen(2005)]{quinonero2005unifying}
Joaquin Qui{\~n}onero-Candela and Carl~E. Rasmussen.
\newblock {A unifying view of sparse approximate {G}aussian process
  regression}.
\newblock \emph{Journal of Machine Learning Research}, 6:\penalty0 1939--1959,
  2005.

\bibitem[Rahimi and Recht(2008)]{Rahimi.2008}
Ali Rahimi and Benjamin Recht.
\newblock Random features for large-scale kernel machines.
\newblock In J.~C. Platt, D.~Koller, Y.~Singer, and S.~T. Roweis, editors,
  \emph{Advances in Neural Information Processing Systems 20}, pages
  1177--1184. 2008.

\bibitem[Rahimi and Recht(2009)]{Rahimi.2009}
Ali Rahimi and Benjamin Recht.
\newblock Weighted sums of random kitchen sinks: Replacing minimization with
  randomization in learning.
\newblock In \emph{Advances in Neural Information Processing Systems 23}, pages
  1313--1320. 2009.

\bibitem[Rasmussen and Williams(2006)]{RasmussenWilliams}
Carl~E. Rasmussen and Christopher~K.I. Williams.
\newblock \emph{{Gaussian Processes for Machine Learning}}.
\newblock MIT, 2006.

\bibitem[Rasmussen and Nickisch(2010)]{rasmussen2010gpml}
Carl~Edward Rasmussen and Hannes Nickisch.
\newblock {Gaussian Processes for Machine Learning (GPML) Toolbox}.
\newblock \emph{Journal of Machine Learning Research}, 11:\penalty0 3011--3015,
  2010.

\bibitem[Saad(2003)]{saad2003iterative}
Yousef Saad.
\newblock \emph{Iterative Methods for Sparse Linear Systems}.
\newblock Society for Industrial and Applied Mathematics, second edition, 2003.

\bibitem[Simon(1984)]{Simon1984Reorthogonalization}
Horst~D. Simon.
\newblock Analysis of the symmetric lanczos algorithm with reorthogonalization
  methods.
\newblock \emph{Linear Algebra and its Applications}, 61:\penalty0 101 -- 131,
  1984.

\bibitem[Skilling(1993)]{skilling1993numerical}
John Skilling.
\newblock Bayesian numerical analysis.
\newblock In Jr~W.~T.~Grandy and P.~W. Milonni, editors, \emph{Physics and
  Probability: Essays in Honor of Edwin T. Jaynes}, pages 207--222. 1993.

\bibitem[Snelson and Ghahramani(2006)]{Snelson2006pseudo}
Edward Snelson and Zoubin Ghahramani.
\newblock Sparse gaussian processes using pseudo-inputs.
\newblock In Y.~Weiss, B.~Sch\"{o}lkopf, and J.~C. Platt, editors,
  \emph{Advances in Neural Information Processing Systems 18}, pages
  1257--1264. 2006.

\bibitem[Snelson and Ghahramani(2007)]{Snelson.2007}
Edward Snelson and Zoubin Ghahramani.
\newblock Local and global sparse {G}aussian process approximations.
\newblock In \emph{Proceedings of the Eleventh International Conference on
  Artificial Intelligence and Statistics}, pages 524--531, 2007.

\bibitem[Solin and S{\"a}rkk{\"a}(2014)]{Solin.2014}
Arno Solin and Simo S{\"a}rkk{\"a}.
\newblock {H}ilbert space methods for reduced-rank {G}aussian process
  regression, 2014.
\newblock URL \url{https://arxiv.org/abs/1401.5508v1}.

\bibitem[Titsias(2009{\natexlab{a}})]{Titsias.2009}
Michalis~K. Titsias.
\newblock Variational learning of inducing variables in sparse {G}aussian
  processes.
\newblock In \emph{Proceedings of the Twelth International Conference on
  Artificial Intelligence and Statistics}, pages 567--574, 2009{\natexlab{a}}.

\bibitem[Titsias(2009{\natexlab{b}})]{titsias09:_variat_learn_induc_variab_spars_gauss_proces}
Michalis~K. Titsias.
\newblock {Variational Learning of Inducing Variables in Sparse {G}aussian
  Processes}.
\newblock In \emph{Proceedings of the Twelth International Conference on
  Artificial Intelligence and Statistics}, 2009{\natexlab{b}}.

\bibitem[Trecate et~al.(1999)Trecate, Williams, and Opper]{Trecate.1999}
Giancarlo~F. Trecate, Christopher~K.~I. Williams, and Manfred Opper.
\newblock Finite-dimensional approximation of {G}aussian processes.
\newblock In \emph{Advances in Neural Information Processing Systems 2}, pages
  218--224, 1999.

\bibitem[Turner(2010)]{RichTurnerPhDThesis}
Richard~E. Turner.
\newblock \emph{{Statistical Models for Natural Sounds}}.
\newblock PhD thesis, University College London, 2010.

\bibitem[Vanhatalo and Vehtari(2008)]{Vanhatalo2008precipitation}
Jarno Vanhatalo and Aki Vehtari.
\newblock Modelling local and global phenomena with sparse gaussian processes.
\newblock In David McAllester and Petri Myllym\"aki, editors, \emph{UAI 2008,
  Twenty-Fourth Conference on Uncertainty in Artificial Intelligence, Helsinki,
  Finland, July 9-12, 2008}, 2008.

\bibitem[Wahba(1990)]{wahba1990spline}
Grace Wahba.
\newblock \emph{{Spline models for observational data}}.
\newblock Number~59 in {CBMS-NSF Regional Conferences series in applied
  mathematics}. SIAM, 1990.

\bibitem[Walder et~al.(2008)Walder, Kim, and Sch{\"o}lkopf]{Walder.2008}
Christian Walder, Kwang~In Kim, and Bernhard Sch{\"o}lkopf.
\newblock Sparse multiscale {G}aussian process regression.
\newblock In \emph{Proceedings of the 25th International Conference on Machine
  Learning}, pages 1112--1119, 2008.

\bibitem[Waugh(1995)]{Waugh1995abalone}
Sam Waugh.
\newblock \emph{{Extending and benchmarking Cascade-Correlation}}.
\newblock PhD thesis, University of Tasmania, 1995.

\bibitem[Weiss and Indurkhya(1995)]{Weiss1995Poletelecomm}
Sholom~M. Weiss and Nitin Indurkhya.
\newblock Rule-based machine learning methods for functional prediction.
\newblock \emph{Journal of Artificial Intelligence Research}, 3\penalty0
  (1):\penalty0 383--403, 1995.

\bibitem[Welling and Teh(2011)]{Welling2011sgld}
Max Welling and Yee~Whye Teh.
\newblock Bayesian learning via stochastic gradient langevin dynamics.
\newblock In \emph{Proceedings of the 28th International Conference on
  International Conference on Machine Learning}, pages 681--688, 2011.

\bibitem[Williams and Seeger(2001)]{williams2001nystroem}
Christopher Williams and Matthias Seeger.
\newblock Using the {N}ystr\"om method to speed up kernel machines.
\newblock In \emph{Advances in Neural Information Processing Systems 13}, 2001.

\bibitem[Wilson and Nickisch(2015)]{wilson2015kissGP}
Andrew Wilson and Hannes Nickisch.
\newblock Kernel interpolation for scalable structured gaussian processes
  (kiss-gp).
\newblock In Francis Bach and David Blei, editors, \emph{Proceedings of the
  32nd International Conference on Machine Learning}, volume~37 of
  \emph{Proceedings of Machine Learning Research}, pages 1775--1784, Lille,
  France, 07--09 Jul 2015. PMLR.
\newblock URL \url{http://proceedings.mlr.press/v37/wilson15.html}.

\bibitem[Wilson et~al.(2014)Wilson, Gilboa, Nehorai, and
  Cunningham]{wilson2014multidimensional}
Andrew~G. Wilson, Elad Gilboa, Arye Nehorai, and John~P. Cunningham.
\newblock Fast kernel learning for multidimensional pattern extrapolation.
\newblock In \emph{Advances in Neural Information Processing Systems 27}, pages
  3626--3634. 2014.

\bibitem[Yan and Qi(2010)]{Yan.2010}
Feng Yan and Yuan Qi.
\newblock Sparse {G}aussian process regression via l1 penalization.
\newblock In \emph{Proceedings of the 27th International Conference on Machine
  Learning}, pages 1183--1190, 2010.

\bibitem[Zhu et~al.(1998)Zhu, Williams, Rohwer, and Morciniec]{Zhu.1998}
Huaiyu Zhu, Christopher K.~I. Williams, Richard~J. Rohwer, and Michal
  Morciniec.
\newblock {G}aussian regression and optimal finite dimensional linear models.
\newblock In \emph{Neural Networks and Machine Learning}. 1998.

\end{thebibliography}
\end{document}